\documentclass[11pt]{article}
\usepackage[margin=1.1in]{geometry}

\usepackage{amsthm}
\usepackage{natbib}
\usepackage[utf8]{inputenc} %
\usepackage[T1]{fontenc}    %
\usepackage{hyperref}       %
\usepackage{url}            %
\usepackage{booktabs}       %
\usepackage{amsmath,amssymb}
\usepackage{amsfonts}       %
\usepackage{graphicx}
\usepackage{algorithm}
\usepackage[noend]{algpseudocode}
\usepackage{nicefrac}       %
\usepackage{microtype}      %
\usepackage{mathtools}
\usepackage{xcolor}         %
\usepackage{macros}
\usepackage[capitalize]{cleveref}
\usepackage{enumerate}      %
\usepackage{mdframed}       %
\usepackage{subcaption}     %
\usepackage{pgfplots}
\usepgfplotslibrary{groupplots}
\pgfplotsset{compat=1.18}
\usepackage{pgfplotstable}

\usepackage{authblk}
\usepackage[toc,page]{appendix}
\usepackage{minitoc}

\usepackage{nameref}
\makeatletter
\let\orgdescriptionlabel\descriptionlabel
\renewcommand*{\descriptionlabel}[1]{%
	\let\orglabel\label
	\let\label\@gobble
	\phantomsection
	\edef\@currentlabel{#1}%
	\let\label\orglabel
	\orgdescriptionlabel{#1}%
}
\makeatother

\makeatletter
\newcommand{\leqnomode}{\tagsleft@true\let\veqno\@@leqno}
\makeatother

\newtheorem{theorem}{Theorem}[section]
\newtheorem{lemma}[theorem]{Lemma}
\newtheorem{fact}{Fact}
\newtheorem{definition}[theorem]{Definition}
\newtheorem{proposition}[theorem]{Proposition}
\newtheorem{corollary}{Corollary}[section]

\theoremstyle{definition}
\newtheorem{remark}[theorem]{Remark}
\newtheorem{assumption}{Assumption}
\newtheorem{claim}{Claim}

\AddToHook{env/definition/begin}{\crefalias{theorem}{definition}} %
\AddToHook{env/lemma/begin}{\crefalias{theorem}{lemma}} %
\AddToHook{env/proposition/begin}{\crefalias{theorem}{proposition}} %
\AddToHook{env/remark/begin}{\crefalias{theorem}{remark}} %

\newcommand{\loss}{\mathcal{L}}
\newcommand{\R}{\mathbb{R}}
\newcommand{\E}{\mathbb{E}}
\newcommand{\KL}[2]{D_{\mathrm{KL}}\!\left(#1\,\|\,#2\right)}

\newcommand{\const}{C_{U}}
\newcommand{\cfloor}{C_{\mathrm{floor}}}
\newcommand{\coll}{\mathrm{coll}}

\newcommand{\gdpolyak}{\texttt{GDPolyak}}

\newcommand{\disp}{\mathsf{D}}

\newcommand{\cavg}{\mathsf{Avg}}

\title{Local linear convergence of gradient methods for overparameterized Gaussian mixtures}

\author[1]{Jingxing Wang}
\author[1,2]{Vasileios Charisopoulos}
\author[1,3]{Maryam Fazel}
\affil[1]{Electrical \& Computer Engineering\\University of Washington, Seattle, WA}
\affil[2]{National Institute for Theory and Mathematics in Biology, Chicago, IL}
\affil[3]{Amazon, Inc.}
\affil[ ]{\texttt{\{jxwang1,vchariso,mfazel\}@uw.edu}}

\allowdisplaybreaks
\pgfplotsset{
	/pgfplots/errorplot/.style={
			width=0.99\linewidth,
			xmin=0,
			xmax=200,
			xlabel={Iteration},
			tick label style={scale=0.75, nodes={transform shape}},
			label style={scale=0.75, nodes={transform shape}},
			xtick pos={bottom},
			ytick pos={left},
			unbounded coords=jump,
			grid=major,
			legend style={
					nodes={scale=0.6, transform shape},
					legend cell align=left,
				},
			ylabel shift=-0.25em,
		}
}

\pgfplotsset{
	/pgfplots/gdptick/.style={
			green!50!black,
			only marks,
			mark=square*,
			mark options={fill=white},
			mark size=1.5pt,
		},
	/pgfplots/gdpplot/.style = {
			green!50!black,
			very thick,
			mark=none,
		},
	/pgfplots/grademplot/.style = {
			blue,
			very thick,
			densely dashed,
			mark=none,
		}
}

\doparttoc %

\begin{document}

\maketitle

\begin{abstract}
	We study the problem of learning Gaussian mixture models under %
	overparameterization.
	Prior work has shown that while overparameterization is essential for avoiding spurious local optima and enables global recovery of the ground-truth model using the gradient-EM (expectation-maximization) algorithm,
	it can dramatically slow down the local rate of convergence.
	Under certain assumptions on the mixture weights, we show that
	a standard divergence measure minimized by statistical
	learning procedures possesses a manifold of slow growth on which the well-known Polyak stepsize
	reduces the loss geometrically, and design a
	gradient-based method that converges to minimizers at a locally
	linear rate. Additionally, we show that our method converges to nearly optimal solutions --- up to a natural misspecification threshold ---
	for mixtures with arbitrary weights.
	At a high level, the method alternates between several ``short'' gradient descent steps that approach the
	manifold and ``long'' Polyak steps that contract the distance to minimizers. Our results suggest that slow
	convergence is not an intrinsic challenge of overparameterization, but can be overcome by exploiting
	the favorable structure of the loss landscape.

\end{abstract}

\faketableofcontents
\part{}

\vspace*{-2em}

\section{Introduction}

Gaussian mixture models (GMMs) are canonical latent variable models with a long history in statistics and machine learning, dating back to the work of Pearson in the late $\text{19}^{\text{th}}$ century. This framework posits that the unknown target distribution $p^{\star}$ is a mixture of $m$ Gaussian distributions (for simplicity, in this paper we consider isotropic Gaussians with identity covariance):
\begin{equation}
	p^{\star} = \sum_{i = 1}^{m} \pi_{i}^{\star} \phi(\cdot \mid \mu_{i}^\star) , \;\; %
	{\pi^{\star}} \geq 0, \; \sum_{i = 1}^{m} \pi^{\star}_{i} = 1, \;\; \text{where}
	\;\;
	\phi(x \mid \mu) := \frac{1}{(2 \pi)^{d/2}} e^{-\frac{\norm{x - \mu}^2}{2}}.
	\label{eq:gaussian-mixture-intro}
\end{equation}
Given samples from $p^{\star}$, the weights $\pi^{\star} \in \R^m$ and parameters $\mu_{i}^{\star}\in \R^d$
are often estimated via an iterative procedure known as the Expectation-Maximization (EM) algorithm~\citep{dempster1977maximum} and its gradient-based variants,
which aim to minimize the Kullback-Leibler (KL) divergence between the mixture and the unknown $p^{\star}$. Classical analyses on both population and sample-based EM show that EM enjoys local linear convergence under suitable regularity conditions~\citep{wu1983convergence,RW84}, such as well-separated components and a correctly specified number of components.

However, recent work has revealed a fundamental limitation of this classical perspective: if we
learn the $m$-component distribution by optimizing over an $n$-component model (with variables $\pi_i$ and $\mu_i$),
\begin{equation}
	p_{\mu, \pi} = \sum_{i = 1}^{n} \pi_{i} \phi(\cdot \mid \mu_{i}), \;\;
	\pi \geq 0, \; \sum_{i} \pi_i = 1,
	\label{eq:student-mixture-intro}
\end{equation}
then when the model is exactly parameterized ($n=m$), the KL divergence landscape can contain spurious local optima, and gradient-based methods fail to recover the ground-truth mixture; specifically, negative results for $m>2$ are shown in \citep{jin2016local}. In contrast, a growing line of work shows that overparameterization---fitting a model with more components than the ground truth; i.e., $n>m$---is essential for global convergence \citep{xu2024toward,zhou2025global}. In this regime, under a separation condition, gradient-EM dynamics provably recover the ground-truth GMM from a random initialization: each ground-truth component is captured by a cluster of fitted components, while redundant components are automatically pruned \citep{zhou2025global}.

This raises a natural question: if overparameterization is necessary for global recovery, what are its optimization consequences? A key challenge is that overparameterization fundamentally alters the local geometry of the loss. Near a solution, multiple fitted components may represent a single ground-truth component, leading to a singular Fisher Information Matrix (FIM) and flat directions in the loss landscape, which
translates to slow convergence rates near solutions~\citep{dwivedi2020singularity}. As a result, standard gradient-based methods exhibit sublinear convergence rates in the local phase \citep{xu2024toward,zhou2025global}, even after the global structure of the mixture has been correctly identified. A key question is whether this significant slowdown is unavoidable. We show that this is not the case, by designing and analyzing geometry-aware gradient-based methods that achieve the statistical recovery benefits \emph{and} favorable (linear) convergence rates simultaneously.

\subsection{Our contributions} \label{sec:subsec:our-contribution}
The works~\citep{xu2024toward,zhou2025global} outline two phases of convergence for the population gradient EM algorithm: a
\emph{global} phase, which drives the loss below a prescribed threshold, and a \emph{local} phase where all non-redundant
student means converge to the nearest teacher mean. In particular, both works show that the KL divergence between
$p^{\star}$ and $p_{\mu, \pi}$, which we denote by $\dkl{p^{\star}}{p_{\mu, \pi}}$, satisfies the celebrated {\L}ojasiewicz inequality~\citep{Lojasiewicz1963} near
solutions:
\begin{align}
	\norm{\grad_{\mu} \dkl{p^{\star}}{p_{\mu, \pi}}} \gtrsim (\dkl{p^{\star}}{p_{\mu, \pi}})^{\theta}, \;\; \text{for some $\theta \in (0, 1)$.}
	\label{eq:lojasiewicz-inequality}
\end{align}
For $\theta = \nicefrac{1}{2}$, the condition~\eqref{eq:lojasiewicz-inequality} is known as the \emph{Polyak-{\L}ojasiewicz} inequality~\citep{Polyak1963},
leading to local linear convergence of gradient methods~\citep{KNS16}. The exponent
$\theta$ is nonstandard and different from $\frac{1}{2}$ for gradient EM. Indeed, the singularity of the FIM
suggests that standard gradient methods should not achieve linear convergence rates, even in the case $n > m = 1$ where the student weights $\pi_{i}$ can be assigned arbitrarily.
In this work, we ask:

\begin{center}
	\begin{minipage}{0.8\textwidth}
		\centering
		\itshape Is it possible to design a (nearly) linearly convergent first-order method for learning overparameterized GMMs?
	\end{minipage}
\end{center}

We answer this question affirmatively by designing a two-stage gradient-based algorithm that leverages the landscape of the KL loss (as a function of the student means $\mu_{i}$) near minimizers. We show that once the global phase succeeds, the local slowdown due to overparameterization is not intrinsic and can be removed.
Our algorithm converges at a nearly \emph{linear}
rate to a collection of student means $\set{\widetilde{\mu}_i}_{i=1}^n$ that are optimal up to a natural misspecification threshold:
\begin{equation}
	\dkl{p^{\star}}{p_{\widetilde{\mu}, \pi}} \lesssim \Delta_{\pi}^2, \quad \text{where} \quad
	\Delta_{\pi} := \max_{\ell = 1, \dots, m} \set*{
		\frac{\abs[\big]{(\sum_{j: \text{$\mu_j$ is near $\mu_{\ell}^{\star}$}} \pi_{j}) - \pi_{\ell}^{\star}}}{\pi_{\ell}^{\star}}
	}.
	\label{eq:mismatch-high-level}
\end{equation}
Thus the parameters $\widetilde{\pi}_i$ are nearly optimal up to a threshold that depends on the worst-case
mismatch between the weight of the $\ell^{\text{th}}$ teacher component, $\pi_{\star}^{\ell}$, and the aggregate weight of
the cluster of students concentrating around $\mu_{\ell}^{\star}$. In the stylized setting of~\cite{xu2024toward} where
the teacher density is Gaussian, we have $\Delta_{\pi} = 0$. Consequently, our method converges
at a nearly linear rate to the unique solution $\widetilde{\mu}_1 = \dots = \widetilde{\mu}_n = 0$ --- an exponential
improvement over prior work.
More generally, our algorithm alternates between updates to student weights and student means using the methodology
in~\citep{zhou2025global}, gradually driving $\Delta_{\pi}$ towards $0$.
\Cref{fig:intro-acceleration} demonstrates the typical behavior of our algorithm on an instance with $\Delta_{\pi} = 0$,
comparing it with the gradient EM method.
As the plots illustrate, the latter settles into a sublinear rate of convergence; in contrast, our algorithm reduces both the KL loss $\dkl{p^{\star}}{p_{\mu, \pi}}$ and a natural measure of parameter
distance at a nearly linear rate.
In summary, our contributions include:
\begin{enumerate}[(i)]
	\item A geometric characterization of the slowdown (ravine structure) and decomposition of error;
	\item Local acceleration via a geometry-aware first-order method;
	\item Robustness to weight mismatch via perturbation analysis.
\end{enumerate}
Taken together, our results provide a unified picture: overparameterization is globally beneficial (enabling recovery), but locally singular (inducing flat directions). We show that these two effects can be reconciled by exploiting the induced geometric structure, yielding fast convergence in the local phase without sacrificing the benefits of overparameterization.

\pgfplotsset{
	/pgfplots/introplot/.style={
			width=0.95\linewidth,
			grid=major,
			xmin=0,
			xmax=200,
			xlabel={Iteration},
			tick label style={scale=0.75, nodes={transform shape}},
			label style={scale=0.75, nodes={transform shape}},
			xtick pos={bottom},
			ytick pos={left},
			unbounded coords=jump,
		}
}

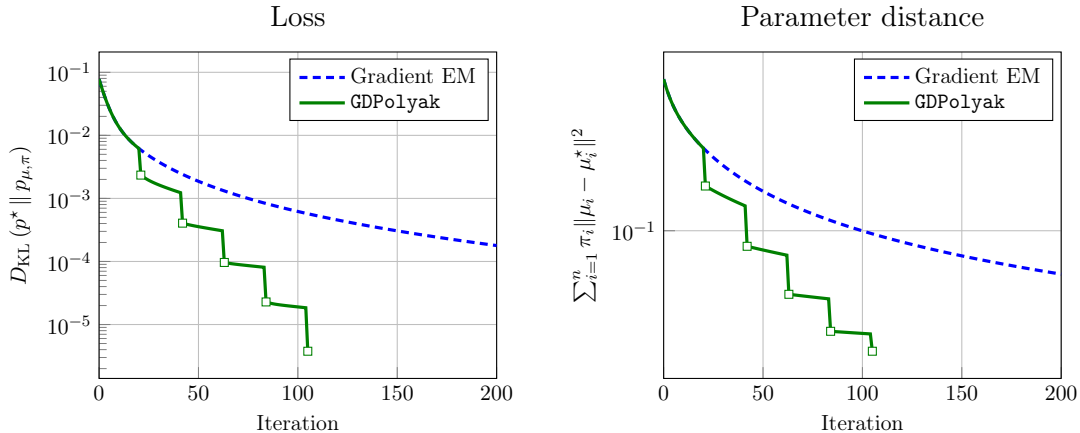
\begin{figure}[t]
	\centering
	\begin{subfigure}[b]{0.45\linewidth}
		\centering
		\begin{tikzpicture}
			\begin{axis}[
					introplot,
					title={Loss},
					ymode=log,
					ylabel={$\dkl{p^{\star}}{p_{\mu, \pi}}$},
					legend style={
							nodes={scale=0.75, transform shape},
							legend cell align=left,
						},
				]
				\addplot+[blue, very thick, densely dashed, mark=none] table[
						x=iteration,
						y=mean_gd_kl,
						col sep=comma,
					] {figures/exp-no-mismatch.csv};
				\addlegendentry{Gradient EM}

				\addplot[green!50!black, very thick, mark=none] table[
						x=iteration,
						y=mean_gd_polyak_kl,
						col sep=comma,
					] {figures/exp-no-mismatch.csv};
				\addlegendentry{\gdpolyak}

				\addplot[green!50!black, only marks, mark=square*, mark options={fill=white}, mark size=1.5pt] table[
						x=iteration,
						y=mean_gd_polyak_kl,
						col sep=comma,
						restrict expr to domain={\thisrow{polyak_marker}}{1:1},
					] {figures/exp-no-mismatch.csv};

			\end{axis}
		\end{tikzpicture}
	\end{subfigure}~
	\begin{subfigure}[b]{0.45\linewidth}
		\centering
		\begin{tikzpicture}
			\begin{axis}[
					introplot,
					title={Parameter distance},
					ymode=log,
					ylabel={$
								\sum_{i = 1}^{n} \pi_{i} \norm{\mu_{i} - \mu_{i}^{\star}}^2
							$},
					legend cell align=left,
					legend style={
							nodes={scale=0.75, transform shape},
						},
				]
				\addplot[grademplot] table[
						x=iteration,
						y=mean_gd_params,
						col sep=comma,
					] {figures/exp-no-mismatch.csv};
				\addlegendentry{Gradient EM}

				\addplot[gdpplot] table[
						x=iteration,
						y=mean_gd_polyak_params,
						col sep=comma,
					] {figures/exp-no-mismatch.csv};
				\addlegendentry{\(\gdpolyak\)}

				\addplot[gdptick] table[
						x=iteration,
						y=mean_gd_polyak_params,
						col sep=comma,
						restrict expr to domain={\thisrow{polyak_marker}}{1:1},
					] {figures/exp-no-mismatch.csv};
			\end{axis}
		\end{tikzpicture}
	\end{subfigure}
	\caption{
		Local acceleration for loss and parameter distance. The square markers indicate iterates obtained with Polyak steps, which are interleaved with several steps of gradient descent. In this instance, we have dimension $d = 5$, $m = 3$ teacher components, and $n = 20$ student components.
	}
	\label{fig:intro-acceleration}
\end{figure}
\paragraph{Notation and basic constructions.}
\label{sec:subsec:notation}
We keep only the notation needed for the main text here; additional notation and standing assumptions are listed in Appendix~\ref{sec:notation-appendix}. Throughout, \(\norm{\cdot}\) denotes the Euclidean norm, \([k]:=\set{1,\ldots,k}\), and \(\Delta^{n-1}\) denotes the probability simplex. We write \(a\lesssim b\) to hide constants depending only on fixed model parameters and write \(a \asymp b\) when \(a \lesssim b\) and \(b \lesssim a\) simultaneously.
We work under the standard nondegeneracy, boundedness, and separation assumptions used in~\citet{zhou2025global}; for completeness, we list these in~\cref{assm:teacher-mixture} in the Appendix.

Given a set $\mathcal{X} \subset \Rbb^d$, we write $\dist_{\mathcal{X}}$ and $\proj_{\mathcal{X}}$ for its distance function and projection map:
\begin{equation}
	\dist_{\mathcal{X}}(x) = \inf_{y \in \mathcal{X}} \norm{x - y}, \qquad
	\proj_{\mathcal{X}}(x) = \argmin_{y \in \mathcal{X}} \norm{x - y}.
	\label{eq:distance-and-projection}
\end{equation}

Given a mean vector $\mu \in \R^{nd}$ partitioned into $\pmx{\mu_1^{\T} & \dots & \mu_{n}^{\T}}^{\T}$ with $\mu_{i} \in \R^d$, and corresponding
mixture weights $\pi = \pmx{\pi_1 & \dots & \pi_{n}}^{\T} \in \Delta^{n-1}$, we denote the mixture density by $p_{\mu, \pi}$, as defined in equation \eqref{eq:student-mixture-intro}.
Writing $\theta = (\mu, \pi)$ for the complete parameterization of the mixture, we let
\begin{equation}
	\psi_{i}(x; \theta) := \frac{\pi_{i} \phi(x \mid \mu_{i})}{\sum_{j = 1}^n \pi_{j} \phi(x \mid \mu_{j})}, \quad
	\set{\psi_{i}(x; \theta)}_{i=1}^{n} \in \Delta^{n-1}
	\label{eq:responsibilities}
\end{equation}
denote the so-called \emph{responsibility} of the $i^{\text{th}}$ mixture component.
We also write $\loss(\theta)$ for the population KL loss between $p^{\star}$ and $p_{\mu, \pi}$, with $\theta = (\mu, \pi)$:
\begin{equation}
	\loss(\theta) := \dkl{p^{\star}}{p_{\mu, \pi}} := \E_{X \sim p^{\star}}\Big[\log\Big(\frac{p^{\star}(X)}{p_{\mu,\pi}(X)}\Big)\Big].
	\label{eq:kl-loss}
\end{equation}

We also record a basic fact about $C^{2}$ manifolds that we use throughout our proofs.
\begin{fact}[Local expansion] \label{fact:local-C2-expansion}
	Let $\mathcal{M}$ be a $C^{2}$ manifold and fix $\bar{x} \in \mathcal{M}$. We have that
	\begin{equation}
		\label{eq:local-C2-expansion}
		x - \bar{x} \in \mathcal{T}_{\mathcal{M}}(\bar{x}) + O(\norm{x - \bar{x}}^2), \;\;
		\text{for all $x \in \mathcal{M}$ near $\bar{x}$.}
	\end{equation}
\end{fact}

\begin{figure}[t]
	\centering
	\begin{subfigure}[b]{0.48\linewidth}
		\centering
		\includegraphics[width=0.98\linewidth]{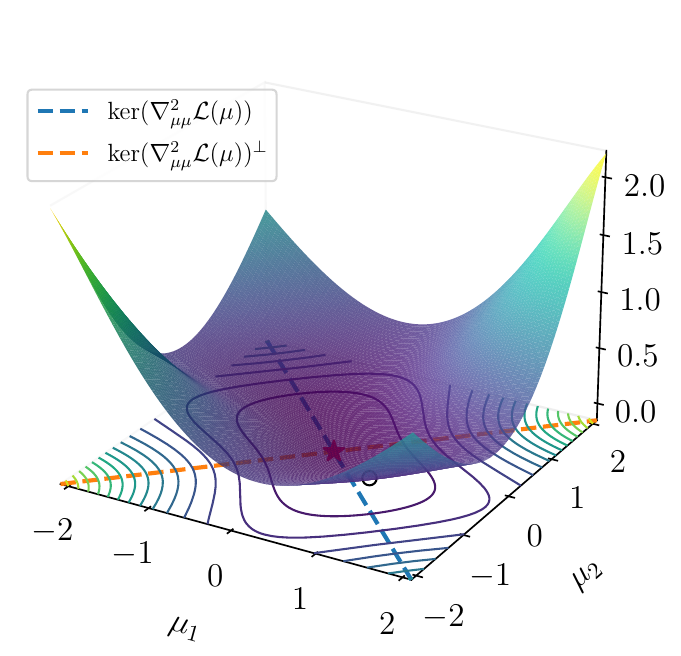}
		\subcaption{KL loss landscape}
		\label{fig:subfig:kl-loss-landscape}
	\end{subfigure}
	\hfill
	\begin{subfigure}[b]{0.48\linewidth}
		\centering
		\includegraphics[width=0.98\linewidth]{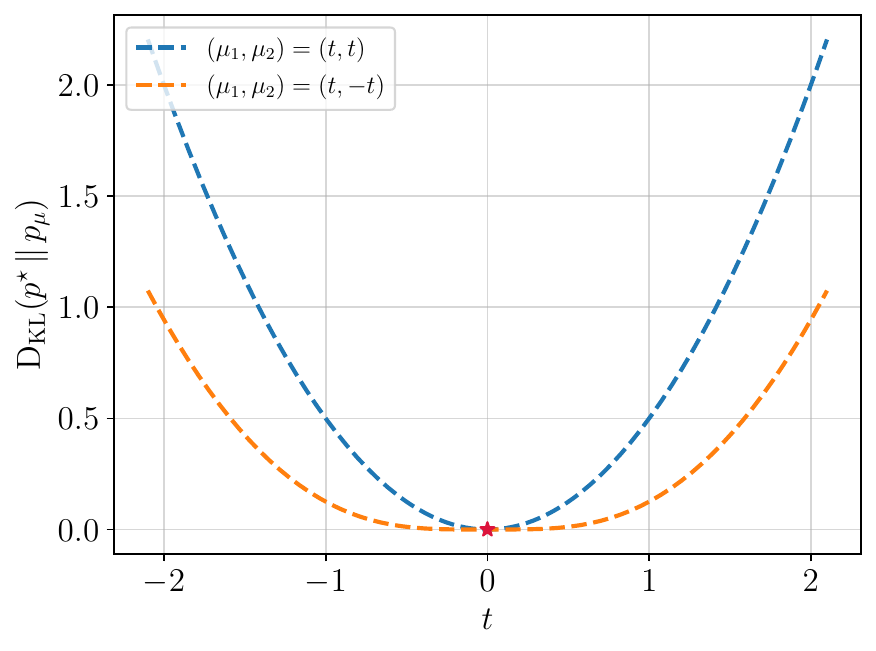}
		\subcaption{Loss along tangent and normal directions}
		\label{fig:subfig:loss-tangent-normal}
	\end{subfigure}
	\caption{
		Loss landscape for the example in~\eqref{eq:simple-example-2-student-intro}. \Cref{fig:subfig:kl-loss-landscape}: the KL landscape is steep across the tangent space and flat along it.
		\Cref{fig:subfig:loss-tangent-normal}: the loss grows quadratically along the normal direction and quartically along the tangent direction.
		Our algorithm alternates gradient descent steps, which lead towards the ravine, with Polyak steps near the ravine to reduce the loss.
	}
	\label{fig:ravine-toy}
\end{figure}

\section{Related work}

\label{sec:related-work}
\paragraph{EM and gradient EM for Gaussian mixtures.}
EM and gradient EM for Gaussian mixtures have been studied extensively, from classical convergence theory to modern population and finite-sample analyses~\citep{dempster1977maximum,wu1983convergence,balakrishnan2014statistical}.
Global convergence guarantees have been proved for two-component mixtures, while general multi-component mixtures 
allow only local analyses under separation assumptions~\citep{xu2016global,daskalakis2017ten,yan2017convergence,zhao2018statistical,kwon2020sample,segol2021improved}.
In contrast, exact-parameterized mixtures with three or more components can have bad local optima~\citep{jin2016local}, motivating the recent study of over-parameterized gradient EM.
The works by~\cite{xu2024toward} and~\cite{zhou2025global} are most closely related to our work. The first paper, addressing the simplest setting where $p^{\star} \sim \mathcal{N}(0, I_{d})$,
shows that the gradient EM algorithm converges (at a sublinear rate) from \emph{any} initialization.
The second paper~\citep{zhou2025global} extends this result to well-separated, multi-component Gaussian mixtures under mild overparameterization.
Our work shows how to significantly accelerate the ``local'' phase of convergence in the setting of~\cite{zhou2025global}, {by identifying and exploiting favorable geometry.}

\paragraph{Overparameterization, slowdown, and pruning.}

Overparameterization has emerged as a central feature of modern learning systems, with both statistical and algorithmic consequences. On the one hand, it enables interpolation without overfitting in classical settings \citep{bartlett2020benign} and underlies the double descent phenomenon \citep{belkin2019double, belkin2020two}. On the other hand, it can substantially alter optimization dynamics, even leading to provably slower convergence of gradient-based methods \citep{xu2023overparameterization}, where the rate slows down from exponential to polynomial $O(T^{-3})$.

Overparameterization creates singular or weakly identifiable directions, leading to %
slow learning dynamics~\cite{dwivedi2020singularity,dwivedi2021sharp} for the EM algorithm.
Related work on overspecified mixtures studies identifiability, redundant components, and vanishing weights~\citep{honguyen2015identifiability,rousseau2011overfitted}.
In the gradient-EM setting,~\citep{xu2024toward,zhou2025global} show that while overparameterization is essential for global convergence, it also leads to slow optimization dynamics near solutions.
We focus on the local difficulty: even after identification or pruning, a teacher may still be represented by multiple active students, producing flat directions that standard gradient descent traverses slowly.

\paragraph{Loss landscapes and manifold identification.}
The algorithm proposed in this paper is grounded in a long line of work on exploiting favorable
structure in optimization problems, including Wright's work on identifiable surfaces~\citep{Wright1993}, the partial smoothness framework of~\cite{Lewis2002}, and the closely related $\mathcal{V}\mathcal{U}$-framework of~\cite{Lemarechal2000}.
The main message from these works is that general optimization problems admit distinguished geometric structures which behave ``favorably'';
moreover, simple gradient-based methods tend to accelerate once they identify these structures~\citep{DDJ24,DDJ25,Mifflin2002,Hare2004,Drusvyatskiy2013,Lewis2016}.
The structure identified in~\cite{davis2024gradient}, which underpins our main results, differs from these results in the sense that the manifold
actually \emph{slows down} gradient-based methods. Our analysis
complements recent work studying local minima in general Gaussian mixtures~\citep{CSXZ24}, which does not show how to leverage the local loss landscape to achieve
acceleration.

\section{Main results}

\label{sec:main_result}
We present an overview of our main results, which can be roughly divided into two parts:

\paragraph{Ravine geometry.} First, we study the loss landscape near minimizers under the assumption that student clusters concentrating
around teacher components have the same aggregate weights as the corresponding teachers. We show that the KL loss
admits an algorithmically exploitable decomposition relative to a
manifold of slow growth, named the \emph{ravine} in recent work by~\cite{davis2024gradient}, and propose an algorithm that contracts
the distance to the optimal means at a linear rate.

\paragraph{Convergence on arbitrary mixtures.}
Second, we use a careful trajectory analysis to show that our algorithm behaves nearly identically (up to a weight-mismatch-dependent perturbation) under small mismatches between the aggregate weights of student clusters and their nearest teacher means---a setting that corresponds to the ``local identifiability'' phase analyzed in~\cite{zhou2025global}.

\subsection{Overparameterization, slowdown and acceleration}
Before we present our results in full generality, we consider a simple warm-up example that illustrates how overparameterization
slows down gradient methods, while simultaneously revealing the key mechanism behind our acceleration method, in a geometrically
transparent way. Let
\begin{equation}
	p^{\star} = \mathcal{N}(0, I_d), \;\; \text{and} \;\;
	p_{\mu} = \frac{1}{2} \mathcal{N}(\mu_1, I_d) + \frac{1}{2} \mathcal{N}(\mu_2, I_d).
	\label{eq:simple-example-2-student-intro}
\end{equation}
Clearly, the optimal student model has $\mu_{1} = \mu_{2} = 0$. A routine calculation shows that
\begin{equation}
	\grad_{\mu\mu}^2 \dkl{p^{\star}}{p_{\mu}}|_{\mu = (0, 0)} = \frac{1}{4} \begin{bmatrix}
		1 & 1 \\ 1 & 1
	\end{bmatrix},
\end{equation}
matching our intuition that the Hessian at the optimal solution is rank-deficient; its kernel is spanned by the vector $(1, -1)$. Indeed, the
KL loss grows quadratically along the ``normal'' direction $(1, 1)$, while it
grows much slower (at a \emph{quartic} rate) along the ``tangent'' direction $(1, -1)$, slowing down gradient descent steps; see Figure~\ref{fig:ravine-toy}
for an illustration.

In order to effectively deal with ``slow'' directions for gradient descent, we turn to the model
problem $f(x, y) := x^{2} + y^{4}$, mimicking the KL loss near $(\mu_1, \mu_2) = 0$.
When $x = 0$, a single step of gradient descent using the well-known Polyak stepsize~\citep{Polyak69} leads to
\[
	(x_+, y_+) = (x, y)
	- \frac{f(x, y) - f^{\star}}{\norm{\grad f(x, y)}^2} \cdot \grad f(x, y) = \big(0, {3y}/{4}\big).
\]
In other words, the gradient method equipped with the Polyak stepsize contracts the distance to $(x^{\star}, y^{\star}) = (0, 0)$
geometrically. Indeed, the work of~\cite{davis2024gradient} suggests that several optimization problems whose objectives grow quartically
away from minimizers possess a manifold of slow growth --- called the \emph{ravine} --- that contains the set of minimizers and satisfies the following properties: (i) the objective function behaves like a pure quartic, $x \mapsto \norm{x}^4$ along the ravine; and (ii) gradient descent with
constant stepsize approaches the ravine at a \emph{geometric} rate.

Building on the above, \cite{davis2024gradient} design an algorithm, dubbed \gdpolyak, that repeatedly interleaves
several gradient descent steps (approaching the ravine) with a single Polyak step (reducing the objective
function geometrically); one epoch of the \(\gdpolyak\) method for an arbitrary loss $f: \mathbb{R}^d \to \R$ (with $f_{\textrm{lb}} \leq \min f$) initialized at $x$ implements the following steps:
\begin{equation}\leqnomode
	\tag{$\gdpolyak(f, f_{\textrm{lb}}, x, \eta, K)$}
	\boxed{
		\begin{aligned}
			y^{(0)}   & := x - \frac{f(x) - f_{\textrm{lb}}}{\norm{\grad f(x)}^2} \grad f(x);           \\
			y^{(i+1)} & := y^{(i)} - \eta \cdot \grad f(y^{(i)}), \quad \text{for $i = 0, \dots, K-1$.}
		\end{aligned}
	}
	\label{eq:gdpolyak-epoch}
\end{equation}

\subsection{From single-teacher to arbitrary GMMs}
\label{sec:main-no-mismatch}
The results of~\cite{davis2024gradient} suggest that a ravine exists
for \emph{any} loss with isolated solutions that grows locally quartically away from its minimizers,
and it tangent to the nullspace of the Hessian at optimal solutions.
In the single-teacher setting of~\cite{xu2024toward}, the solution is unique and the KL loss~\eqref{eq:kl-loss} satisfies quartic growth: indeed, we show
that (see Prop.~\ref{prop:single-teacher-quartic-general} for a precise statement):
\begin{equation}
	\label{eq:single-teacher-quartic-growth-intro}
	\loss(\mu) \gtrsim \Bigl(
	\pi_{\min} \sum_{i = 1}^{n} \norm{\mu_{i} - \mu^{\star}}^2
	\Bigr)^2, \quad
	\text{where} \;\; \pi_{\min} := \min_{i = 1, \dots, n} \pi_{i}.
\end{equation}
To understand the KL loss geometry for arbitrary mixtures, we first pass to an idealized setting where every teacher component is represented by a cluster of students whose aggregate weight matches the corresponding teacher weight.
Such configurations arise during the local convergence phase of gradient EM~\citep{zhou2025global} near minimizers ${\theta}^{\star} = (\widetilde{\mu}, \pi)$ that induce partitions of the following form (cf.~\cref{definition:exact-grouping}):
\begin{equation}
	[n]= \bigsqcup_{\ell=1}^m S_\ell,
	\quad
	\widetilde{\mu}_i = \mu_{\ell}^\star, \;\;
	\text{for $i \in S_\ell$},
	\;\; \text{with} \;\;
	\sum_{i \in S_\ell}\pi_i=\pi_\ell^\star.
	\label{eq:partition}
\end{equation}
Our main result in this section characterizes the Hessian (for fixed $\pi$) at optimal solutions and shows that the loss indeed admits a ravine.
Inspecting the Hessian nullspace, we find that the ``slow growth'' directions are precisely those along which student cluster means remain unchanged.

\begin{theorem}[Ravine geometry of clustered loss (informal); see~\cref{thm:exact-grouped-ravine-geometry}]
	\label{theorem:clustered-loss-geometry-informal}
	Fix a minimizer $\theta^{\star}$ inducing a partition~\eqref{eq:partition}.
	The Hessian of the KL loss at $\theta^{\star}$ satisfies
	\begin{equation}
		\ker(\grad^2_{\mu\mu} \loss(\theta^{\star})) = \set{
			u \in \R^{nd} \mid \sum\limits_{i \in S_{\ell}} \pi_{i} u_{i} = 0, \;
			\text{for all $\ell \in [m]$}
		}. \label{eq:hessian-at-reference-optimum}
	\end{equation}
	Moreover, define the aggregate cluster weights and corresponding cluster means as
	\begin{equation}
		\widehat{\pi}_{\ell} := \sum_{i \in S_{\ell}} \pi_{i}, \quad
		\bar{\mu}_{\ell} := \frac{1}{\widehat{\pi}_{\ell}} \sum_{i \in S_{\ell}} \pi_{i} \mu_{i}.
		\label{eq:cluster-weights-and-means}
	\end{equation}
	Then, near the optimal solution $\theta^{\star}$, the KL loss satisfies quartic growth via the decomposition
	\begin{equation}
		\sqrt{\loss(\mu, \pi)} \gtrsim
		\sum_{\ell=1}^m
		\sum_{i\in S_\ell}
		\pi_i \norm{\mu_i-\mu_\ell^\star}^2 =
		\underbrace{
			\sum_{\ell=1}^m\widehat\pi_\ell
			\norm{
				\bar{\mu}_{\ell} - \mu_{\ell}^{\star}
			}^2
		}_{\text{cluster bias}}
		+
		\underbrace{
			\sum_{\ell=1}^m\sum_{i\in S_\ell}
			\pi_{i} \norm{
				\mu_{i} - \bar{\mu}_{\ell}
			}^2
		}_{\text{intra-cluster dispersion}}.
		\label{eq:bias-dispersion-informal}
	\end{equation}
\end{theorem}

\begin{remark}[Local convergence]
	The above theorem implies that the KL loss has fourth-order growth near the (isolated) minimizer $\theta^{\star}$. Therefore, the \(\gdpolyak\) method of~\cite{davis2024gradient} applied to $\loss(\mu)$ converges locally nearly-linearly to $\theta^{\star}$.
\end{remark}

\subsection{Trajectory analysis under weight mismatch}
The condition~\eqref{eq:partition} may not hold in practice. Aggregate weights of clusters could only approximately match the corresponding teacher weights;
similarly, student mixtures may include redundant components.
Therefore, our forthcoming analysis focuses on the ``local'' phase of gradient EM~\citep{zhou2025global}, where weight mismatches are sufficiently small and redundant student components can be pruned via simple thresholding. We
develop a perturbative analysis that relates the trajectory of $\gdpolyak$ for
the actual objective $\loss(\mu, \pi)$ to its trajectory on a ``reference''
objective $\loss(\mu, \bar{\pi})$, suitably reweighted so that~\eqref{eq:partition} takes
hold for its minimizers.

To state our results, we need some notation. We define the cluster weight mismatch:
\[
	\Delta_\pi :=
	\max_{\ell\in[m]}
	\frac{\abs[\big]{\sum_{i\in S_\ell}\pi_i-\pi_\ell^\star}}{\pi_\ell^\star}.
\]
Moreover, we define the reference weights $\bar{\pi}_i := \pi_{i} \pi_{\ell}^{\star} / \widehat{\pi}_{\ell}$ for each $i \in S_{\ell}$.
Clearly, $\Delta_{\bar{\pi}} = 0$.

\begin{lemma}[Perturbed loss (informal); see~\cref{proposition:loss-perturbation}]
	\label{thm:mismatched-frozen-weights}
	Suppose that $\min_{i} \pi_{i} > 0$ and that $\loss(\mu, \pi)$ is sufficiently small. Then, for every $\mu$ on a compact neighborhood $U$ of $\mu^{\star}$, it holds that
	\[
		\abs{\loss(\mu,\pi) - \loss(\mu,\bar\pi)},
		\;\;
		\norm{\grad_{\mu} \loss(\mu, \pi) - \grad_{\mu} \loss(\mu, \bar{\pi})},
		\;\;
		\opnorm{
			\grad_{\mu\mu}^2 \loss(\mu,\pi) -
			\grad_{\mu\mu}^2 \loss(\mu,\bar\pi)
		}
		\lesssim \Delta_{\pi}.
	\]
	Moreover, the function value at the minimizer of the perturbed loss satisfies
	\begin{equation}
		\min_{\mu \in U} \loss(\mu,\pi) - {\loss}^\star \lesssim \Delta_\pi^2.
		\label{eq:perturbed-loss-at-optimum}
	\end{equation}
\end{lemma}

The reference weights \(\bar\pi\) preserve the relative weights within each cluster but adjust the cluster aggregate weights. Consequently, \(\loss(\mu,\bar\pi)\) satisfies the conditions
of~\cref{theorem:clustered-loss-geometry-informal},
whence it possesses a ravine. In what follows, we let $\mathcal{M}$
denote the ravine of the reference objective.

\begin{theorem}[Contraction of Polyak step (informal); see~\cref{thm:tangential-annulus-contraction-or-core}]
	\label{thm:mismatched-polyak-contraction}
	Suppose that the assumptions of~\cref{thm:mismatched-frozen-weights} %
	hold and
	that $\mu$ satisfies the following condition:
	\begin{equation}
		\max\set*{
			\dist_{\mathcal{M}}(\mu), \Delta_{\pi}
		} \lesssim \left( \sum_{\ell = 1}^{m}
		\sum_{i \in S_{\ell}} \norm{\mu_{i} - \bar{\mu}_{\ell}}^2
		\right)^{3/2}.
		\label{eq:in-annulus-informal}
	\end{equation}
	Then, the point $\mu^{+} := \mu - \frac{\loss(\mu, \pi)}{\norm{\grad \loss(\mu, \pi)}^2} \grad \loss(\mu, \pi)$ satisfies one of the following conditions:
	\begin{enumerate}[(i)]
		\item $\norm{P_{\mathcal{M}}(\mu^{+}) - \mu^{\star}} \lesssim \Delta_{\pi}^{1/3}$; or
		\item $\norm{P_{\mathcal{M}}(\mu^{+}) - \mu^{\star}} \leq (1 - \gamma) \norm{P_{\mathcal{M}}(\mu) - \mu^{\star}}, \; \text{for some $\gamma \in (0, 1)$.}$
	\end{enumerate}
\end{theorem}

\paragraph{Proof sketch.}
Our proof separates the neighborhood of the minimizer into three regions: (i) a tube $\mathcal{T}$ around the manifold $\mathcal{M}$, where $\dist_{\mathcal{M}}(\mu)$ is small relative to the distance to minimizers; (ii) a ``core'' region $\mathcal{C}$, where $\Delta_{\pi}$ dominates the dispersion term from~\cref{theorem:clustered-loss-geometry-informal}; and (iii) an annulus $\mathcal{A} := \mathcal{T} \setminus \mathcal{C}$, wherein Polyak steps make algorithmic progress. We then show (see~\cref{fig:manifold}):
\begin{enumerate}
	\item Starting from $\mathcal{A}$, the projection of a single Polyak step onto the manifold either reduces the distance to $\mu^{\star}$ geometrically or already falls into $\mathcal{C}$ (\cref{thm:tangential-annulus-contraction-or-core}).
	\item The Polyak iterate can escape
	      $\mathcal{A}$; in that case, a few steps of gradient descent with sufficiently small stepsize restore proximity to the manifold (\cref{lem:short-gd-ravine-recovery}).
	\item Finally, we show that the gradient descent trajectory cannot undo the progress
	      towards $\mu^{\star}$ achieved by the Polyak step (\cref{cor:distance-contraction-after-gdpolyak-epochs}).
\end{enumerate}
To prove the above results for the general loss $\loss(\mu, \pi)$, we show that they hold for the ``reference'' loss $\loss(\mu, \bar{\pi})$
and establish a stronger version of the perturbation estimates in~\cref{lemma:relative-value-annulus}, using~\cref{proposition:loss-perturbation} as a boostrap. This allows us to directly relate the progress achieved by a single Polyak step on the two different objective functions (\cref{lemma:polyak-step-comparison}).

\begin{figure*}
	\centering
	\includegraphics[width=0.65\linewidth]{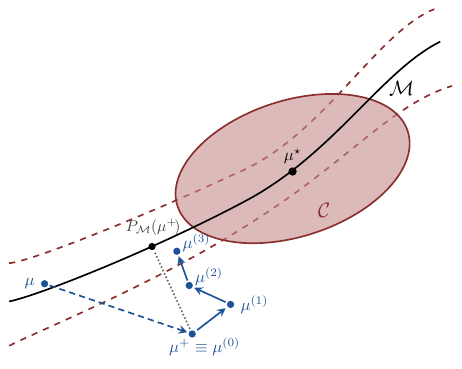}
	\caption{Illustration of algorithmic behavior. The Polyak step, $\mu^{+}$, makes significant progress towards $\mu^{\star}$ along the ravine $\mathcal{M}$ but can
		escape the tube surrounding it. A few ``short steps'' of gradient descent rapidly
		restore proximity to $\mathcal{M}$, wherein a new Polyak step can be attempted.}
	\label{fig:manifold}
\end{figure*}

\section{Implementation and numerical study}
\label{sec:experiments}

In this section, we formally describe our two phase method (\cref{alg:switching-gdpolyak}) and conduct a numerical study to
validate our theoretical predictions and demonstrate the local acceleration mechanism.
While our theory covers the population gradient EM setting, our experiments use a large \emph{fixed} batch
of $N = 10^{7}$ samples to approximate all quantities,
treating the empirical dynamics as a finite-sample perturbation of population-level behavior~\citep{balakrishnan2014statistical}.
Each experiment took less than five minutes on a single NVIDIA L40S GPU node with 48GB memory.

\begin{algorithm}[h]
    \caption{Phase-switching \(\gdpolyak\)}
    \begin{algorithmic}[1]
        \State \textbf{Input}: threshold $0 < \varepsilon < \varepsilon_{\text{id}}$, stepsize $\eta > 0$, epochs $T_{\mathrm{II}}$, gradient EM steps $K$.
        \State \textbf{Initialize}: $\mu_{i}^{(0)} \sim p^{\star}$,
        $\pi^{(0)} = \tfrac{1}{n} \cdot {\bm{1}_n}$.
        \For{$t = 0, 1, \dots$ until $\loss(\mu^{(t)}, \pi^{(t)}) \leq \varepsilon_{\text{id}}$}
            \State $\pi^{(t+1)} := \argmin_{\pi \in \Delta^{n-1}} \loss(\mu^{(t)}, \pi)$
            \State $\mu^{(t+1)} := \mu^{(t)} - \eta \cdot \grad_{\mu} \loss(\mu^{(t)}, \pi^{(t+1)})$
            \Comment{Gradient EM}
        \EndFor
        \For{$\tau = 0, \dots, T_{\mathrm{II}}-1$}
            \Comment{See~\eqref{eq:gdpolyak-epoch}}
            \State $\mu^{(t+\tau+1)} := \gdpolyak(\mu \mapsto \loss(\mu, \pi^{(t)}), 0, \mu^{(t+\tau)}, \eta, K)$
        \EndFor
    \end{algorithmic}
    \label{alg:switching-gdpolyak}
\end{algorithm}
\textbf{Weight updates.} Following~\cite{zhou2025global}, we update weights by approximately solving the convex subproblem \(\argmin_{\pi \in \Delta^{n-1}} \loss(\mu^{(t)}, \pi)\).
We employ the standard EM weight update, which can be viewed as the fixed-point iteration on the KKT system, as the minimization oracle:
\begin{equation}
    \pi_{i}^{+} \gets
    \frac{1}{N} \sum_{k = 1}^{N} \frac{\pi_{i} \phi(X_{k} \mid \mu_{i}^{(t)})}{
        \sum_{j = 1}^{n} \pi_{j} \phi(X_{k} \mid \mu_{j}^{(t)})
    }.
    \label{eq:weight-update-em}
\end{equation}
We found that $10$ steps of~\eqref{eq:weight-update-em} were sufficient for stable performance in our experiments.

\paragraph{Experiment: acceleration with exact clustering.}
In our first experiment, we demonstrate the slowdown of gradient EM and the acceleration achieved by our method on an instance with $m = 3$ teachers with $\pi^{\star} = (0.35, 0.35, 0.3)$, $n = 20$ students with $\pi_i = \tfrac{1}{20}$ for all $i$, and ambient dimension $d = 5$; we initialize the student means artificially close to $\mu^{\star}$ so that~\eqref{eq:partition} is satisfied.
We compare the fixed-weight gradient EM method with \(\gdpolyak\), using one Polyak step every $20$ gradient EM steps, and terminate both methods upon reaching a target KL loss of $\varepsilon = 10^{-6}$.
The results are illustrated in~\cref{fig:no-mismatch-bias-dispersion}; we find that the intra-cluster dispersion term drives the slowdown of gradient EM, but is reduced sharply by the interleaved Polyak steps, as prescribed by our convergence analysis.

\pgfplotsset{
  /pgfplots/emsubplot/.style={
    width=\linewidth,
    height=0.82\linewidth,
    xmin=0,
    xmax=200,
    xlabel={Iteration},
    grid=major,
    major grid style={draw=gray!35, line width=0.25pt},
    tick label style={scale=0.72, nodes={transform shape}},
    label style={scale=0.82, nodes={transform shape}},
    title style={scale=0.92, nodes={transform shape}},
    xtick pos={bottom},
    ytick pos={left},
    unbounded coords=jump,
  }
}

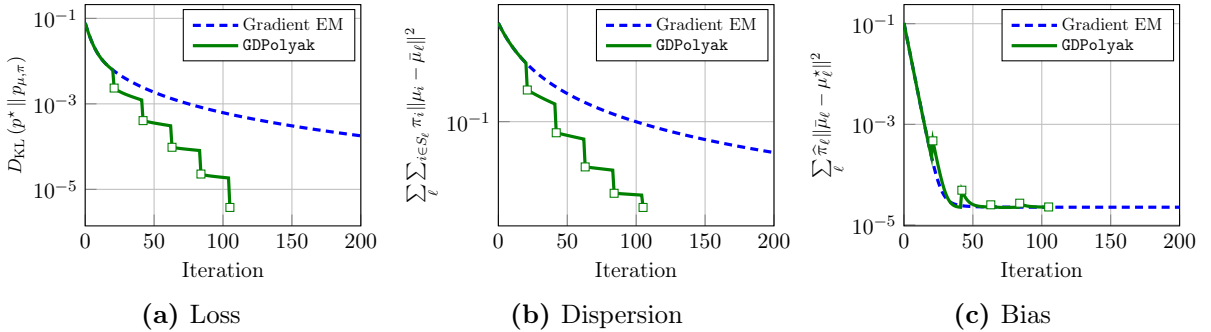
\begin{figure}[htbp]
\centering
\begin{subfigure}[b]{0.33\textwidth}
\centering
\begin{tikzpicture}
\begin{axis}[
    errorplot,
    ymode=log,
    ylabel={
        $\dkl{p^{\star}}{p_{\mu, \pi}}$
    },
    ylabel style={
        scale=0.85,
        nodes={transform shape},
    }
]
\addplot[grademplot] table[
    x=iteration,
    y=mean_gd_kl,
    col sep=comma,
] {figures/exp-no-mismatch.csv};
\addlegendentry{Gradient EM}

\addplot[gdpplot] table[
    x=iteration,
    y=mean_gd_polyak_kl,
    col sep=comma,
] {figures/exp-no-mismatch.csv};
\addlegendentry{\(\gdpolyak\)}

\addplot[gdptick] table[
    x=iteration,
    y=mean_gd_polyak_kl,
    col sep=comma,
    restrict expr to domain={\thisrow{polyak_marker}}{1:1},
] {figures/exp-no-mismatch.csv};
\end{axis}
\end{tikzpicture}
\caption{Loss}
\label{fig:subfig:loss-nomismatch}
\end{subfigure}~%
\begin{subfigure}[b]{0.33\textwidth}
\centering
\begin{tikzpicture}
\begin{axis}[
    errorplot,
    ymode=log,
    ylabel={
        $\sum\limits_{\ell} \sum_{i \in S_{\ell}}
        \pi_{i} \norm{\mu_{i} - \bar{\mu}_{\ell}}^2$
    },
    ylabel style={
        scale=0.85,
        nodes={transform shape},
    }
]
\addplot[grademplot] table[
    x=iteration,
    y=mean_gd_dispersion,
    col sep=comma,
] {figures/exp-no-mismatch.csv};
\addlegendentry{Gradient EM}

\addplot[gdpplot] table[
    x=iteration,
    y=mean_gd_polyak_dispersion,
    col sep=comma,
] {figures/exp-no-mismatch.csv};
\addlegendentry{\(\gdpolyak\)}

\addplot[gdptick] table[
    x=iteration,
    y=mean_gd_polyak_dispersion,
    col sep=comma,
    restrict expr to domain={\thisrow{polyak_marker}}{1:1},
] {figures/exp-no-mismatch.csv};
\end{axis}
\end{tikzpicture}
\caption{Dispersion}
\label{fig:subfig:dispersion-nomismatch}
\end{subfigure}~%
\begin{subfigure}[b]{0.33\textwidth}
\centering
\begin{tikzpicture}
\begin{axis}[
    errorplot,
    ymode=log,
    ylabel={
        $\sum\limits_{\ell} \widehat{\pi}_{\ell} \norm{\bar{\mu}_{\ell} - \mu^{\star}_{\ell}}^2$
    },
    ylabel style={
        scale=0.85,
        nodes={transform shape},
    },
]
\addplot[grademplot] table[
    x=iteration,
    y=mean_gd_bias,
    col sep=comma,
] {figures/exp-no-mismatch.csv};
\addlegendentry{Gradient EM}

\addplot[gdpplot] table[
    x=iteration,
    y=mean_gd_polyak_bias,
    col sep=comma,
] {figures/exp-no-mismatch.csv};
\addlegendentry{\(\gdpolyak\)}

\addplot[gdptick] table[
    x=iteration,
    y=mean_gd_polyak_bias,
    col sep=comma,
    restrict expr to domain={\thisrow{polyak_marker}}{1:1},
] {figures/exp-no-mismatch.csv};
\end{axis}
\end{tikzpicture}
\caption{Bias}
\label{fig:subfig:bias-nomismatch}
\end{subfigure}
\caption{
Experiment on an instance with $\Delta_{\pi} = 0$. Gradient EM struggles to reduce the intra-cluster dispersion, which contributes the majority of the loss near minimizers. Polyak steps (indicated by square marks) sharply reduce the dispersion, but slightly increase the cluster bias due to leaving the vicinity of the ravine; interleaved gradient descent steps restore proximity to the manifold.
}
\label{fig:no-mismatch-bias-dispersion}
\end{figure}

\paragraph{Experiment: acceleration with arbitrary mixtures.}
We turn to an experiment under a more realistic setting where we update both means and mixture weights.
We maintain the previous setup but now sample ${\mu_i}^{(0)}$ from $p^{\star}$,
and compare gradient EM with~\cref{alg:switching-gdpolyak} using $\varepsilon_{\text{id}} = 10^{-3}$; the results are shown in \cref{fig:mismatch-freeze}, where the ``freeze'' threshold is indicated by a vertical line. Again, we find that Polyak steps are essential for reducing
the dispersion and achieving local linear convergence.

\newcommand{\freezeline}{%
  \draw[red!75!black, densely dashed, line width=0.85pt]
    ({axis cs:10,0} |- {rel axis cs:0,0}) --  ({axis cs:10,0} |- {rel axis cs:0,1.0});
}
\newcommand{\freezelabel}{%
  \node[red!75!black, scale=0.62, nodes={transform shape}, anchor=north west]
    at ({axis cs:10,0} |- {rel axis cs:0,0.98}) {freeze};
}

\begin{figure}[htbp]
\centering
\begin{subfigure}[b]{0.33\textwidth}
\centering
\begin{tikzpicture}
\begin{axis}[
    errorplot,
    ymode=log,
    ylabel={
        $\dkl{p^{\star}}{p_{\mu, \pi}}$
    },
    ylabel style={
        scale=0.85,
        nodes={transform shape},
    }
]
\addplot[grademplot] table[
    x=iteration,
    y=gradient_em_kl,
    col sep=comma,
] {figures/exp-freeze_10.csv};
\addlegendentry{Gradient EM}

\addplot[gdpplot] table[
    x=iteration,
    y=freeze_gd_polyak_kl,
    col sep=comma,
] {figures/exp-freeze_10.csv};
\addlegendentry{Alg.~\ref{alg:switching-gdpolyak}}

\addplot[gdptick] table[
    x=iteration,
    y=freeze_gd_polyak_kl,
    col sep=comma,
    restrict expr to domain={\thisrow{polyak_marker}}{0.5:1.5},
] {figures/exp-freeze_10.csv};

\freezeline
\freezelabel
\end{axis}
\end{tikzpicture}
\end{subfigure}~%
\begin{subfigure}[b]{0.33\textwidth}
\centering
\begin{tikzpicture}
\begin{axis}[
    errorplot,
    ymode=log,
    ylabel={
        $\sum\limits_{\ell} \sum_{i \in S_{\ell}}
        \pi_{i} \norm{\mu_{i} - \bar{\mu}_{\ell}}^2$
    },
    ylabel style={
        scale=0.85,
        nodes={transform shape},
    },
]
\addplot[grademplot] table[
    x=iteration,
    y=gradient_em_dispersion,
    col sep=comma,
] {figures/exp-freeze_10.csv};
\addlegendentry{Gradient EM}

\addplot[gdpplot] table[
    x=iteration,
    y=freeze_gd_polyak_dispersion,
    col sep=comma,
] {figures/exp-freeze_10.csv};
\addlegendentry{Alg.~\ref{alg:switching-gdpolyak}}

\addplot[gdptick] table[
    x=iteration,
    y=freeze_gd_polyak_dispersion,
    col sep=comma,
    restrict expr to domain={\thisrow{polyak_marker}}{0.5:1.5},
] {figures/exp-freeze_10.csv};

\freezeline
\freezelabel
\end{axis}
\end{tikzpicture}
\end{subfigure}~%
\begin{subfigure}[b]{0.33\textwidth}
\centering
\begin{tikzpicture}
\begin{axis}[
    errorplot,
    ymode=log,
    ylabel={
        $\sum\limits_{\ell} \widehat{\pi}_{\ell} \norm{\bar{\mu}_{\ell} - \mu^{\star}_{\ell}}^2$
    },
    ylabel style={
        scale=0.85,
        nodes={transform shape},
    },
]
\addplot[grademplot] table[
    x=iteration,
    y=gradient_em_bias,
    col sep=comma,
] {figures/exp-freeze_10.csv};
\addlegendentry{Gradient EM}

\addplot[gdpplot] table[
    x=iteration,
    y=freeze_gd_polyak_bias,
    col sep=comma,
] {figures/exp-freeze_10.csv};
\addlegendentry{Alg.~\ref{alg:switching-gdpolyak}}

\addplot[gdptick] table[
    x=iteration,
    y=freeze_gd_polyak_bias,
    col sep=comma,
    restrict expr to domain={\thisrow{polyak_marker}}{0.5:1.5},
] {figures/exp-freeze_10.csv};

\freezeline
\freezelabel
\end{axis}
\end{tikzpicture}
\end{subfigure}
\caption{
Experiment with general mixtures. Both methods are initialized identically; as in~\cref{fig:no-mismatch-bias-dispersion}, gradient EM struggles to decrease the dispersion term, in contrast with \(\gdpolyak\). The vertical dashed line marks the iteration upon which~\cref{alg:switching-gdpolyak} enters its second phase.
}
\label{fig:mismatch-freeze}
\end{figure}
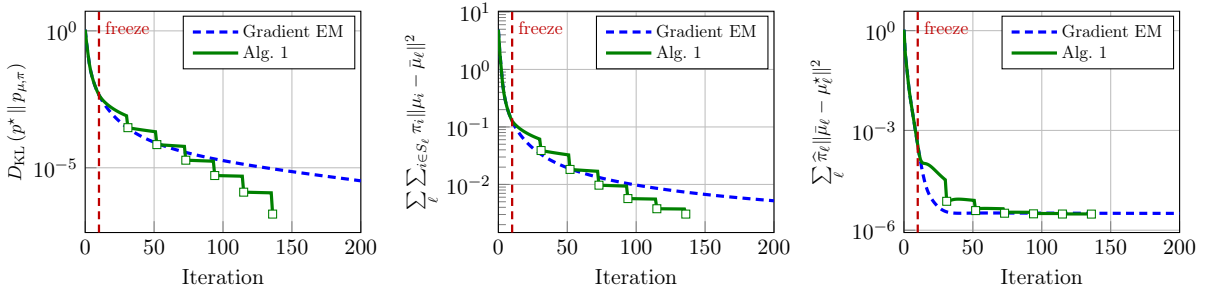

\paragraph{Experiment: effect of cluster weight mismatch.}
We probe the effect of the mismatch $\Delta_{\pi}$ on the accuracy of the proposed method, as well as the acceleration mechanism itself.
We introduce a controlled mismatch $\Delta_{\pi} \in \set{10^{-1}, 10^{-2}, \dots, 10^{-4}}$ by perturbing the teacher weights from the setup of~\cref{fig:no-mismatch-bias-dispersion}, freeze student weights, and optimize the means using \(\gdpolyak\).
Plotting the results in~\cref{fig:mismatch-effect}, we observe that smaller $\Delta_{\pi}$ lead to similar geometric contraction factors for \(\gdpolyak\),
while larger mismatches can interfere with acceleration.\
All configurations induce a visible loss ``barrier'' proportional to the prediction $\Delta_{\pi}^2$ (cf.~\cref{thm:mismatched-frozen-weights}).

\definecolor{misLarge}{HTML}{D62728}
\definecolor{misMed}{HTML}{F28E2B}
\definecolor{misSmallA}{HTML}{4E79A7}
\definecolor{misSmallB}{HTML}{59A14F}

\begin{figure}[htbp]
	\centering
    \begin{subfigure}[b]{0.45\textwidth}
	\begin{tikzpicture}
    \begin{axis}[
        width=0.99\textwidth,
        height=0.77\textwidth,
        tick label style={
            scale=0.85,
            nodes={transform shape},
        },
        label style = {
            scale=0.75,
            nodes={transform shape},
        },
        ymode=log,
        ylabel={$\dkl{p^{\star}}{p_{\mu, \pi}}$},
        xlabel={Iteration},
        xmin=0,
        xmax=200,
        ymax=1,
        ymin=2e-7,
        grid=major,
        legend style={
            legend pos=south east,
            legend cell align=left,
            nodes={scale=0.55, transform shape},
        },
    ]
			\addplot[
				misLarge,
				very thick,
				mark=triangle*,
                mark size=1.5pt,
                mark repeat=10,
			] table[
					x=iteration,
					y=mismatch_0_kl,
					col sep=comma,
				] {figures/mismatch_GDP.csv};
			\addlegendentry{\(\Delta_{\pi} = 10^{-1}\)}

			\addplot[
				misMed,
				very thick,
                mark=diamond,
                mark size=1.5pt,
                mark repeat=10,
			] table[
					x=iteration,
					y=mismatch_1_kl,
					col sep=comma,
				] {figures/mismatch_GDP.csv};
			\addlegendentry{\(\Delta_{\pi} = 10^{-2}\)}

			\addplot[
				misSmallB,
				very thick,
				mark=o,
				mark size=1.5pt,
				mark repeat=10,
			] table[
					x=iteration,
					y=mismatch_2_kl,
					col sep=comma,
				] {figures/mismatch_GDP.csv};
			\addlegendentry{\(\Delta_{\pi} = 10^{-3}\)}

			\addplot[
				misSmallA,
				very thick,
                mark=square,
                mark size=1.5pt,
                mark repeat=20,
			] table[
					x=iteration,
					y=mismatch_3_kl,
					col sep=comma,
				] {figures/mismatch_GDP.csv};
			\addlegendentry{\(\Delta_{\pi} = 10^{-4}\)}

			\addplot[
				misLarge!45,
				densely dashed,
                very thick,
				mark=none,
			] coordinates {(0,1e-2) (200,1e-2)};
            \addplot[
				misMed!45,
				densely dashed,
                very thick,
				mark=none,
			] coordinates {(0,1e-4) (200,1e-4)};
			\addplot[
				misSmallB!45,
				densely dashed,
                very thick,
				mark=none,
			] coordinates {(0,1e-6) (200,1e-6)};  
        \end{axis}
	\end{tikzpicture}
    \caption{Loss for varying mismatch $\Delta_{\pi}$}
    \label{fig:subfig:loss-by-mismatch}
    \end{subfigure}~%
    \begin{subfigure}[b]{0.45\textwidth}
	\begin{tikzpicture}
    \begin{axis}[
        width=0.99\textwidth,
        height=0.77\textwidth,
        tick label style={
            scale=0.85,
            nodes={transform shape},
        },
        label style = {
            scale=0.75,
            nodes={transform shape},
        },
        ymode=log,
        ylabel={Mismatch $\Delta_{\pi}$},
        xlabel={Iteration},
        xmin=0,
        xmax=20,
        ymax=1,
        ymin=2e-5,
        grid=major,
        legend style={
            legend cell align=left,
            nodes={scale=0.65, transform shape},
        },
    ]
			\addplot[
				grademplot
			] table[
					x=iteration,
					y=max_group_mismatch,
					col sep=comma,
				] {figures/mismatch_small.csv};
			\addlegendentry{Gradient EM}

        \end{axis}
	\end{tikzpicture}
    \caption{Mismatch reduction during gradient EM}
    \label{fig:subfig:gradient-em-mismatch-reduction}
    \end{subfigure}
	\caption{
		Effect of weight mismatch $\Delta_{\pi}$ on acceleration. (\textbf{\ref{fig:subfig:loss-by-mismatch}}): small $\Delta_{\pi}$ maintain the contraction rate and final loss essentially unaffected, while larger $\Delta_{\pi}$ stall at higher loss floors. (\textbf{\ref{fig:subfig:gradient-em-mismatch-reduction}}): gradient EM rapidly reduces $\Delta_{\pi}$, which stabilizes within less than 15 iterations. 
	}
	\label{fig:mismatch-effect}
\end{figure}
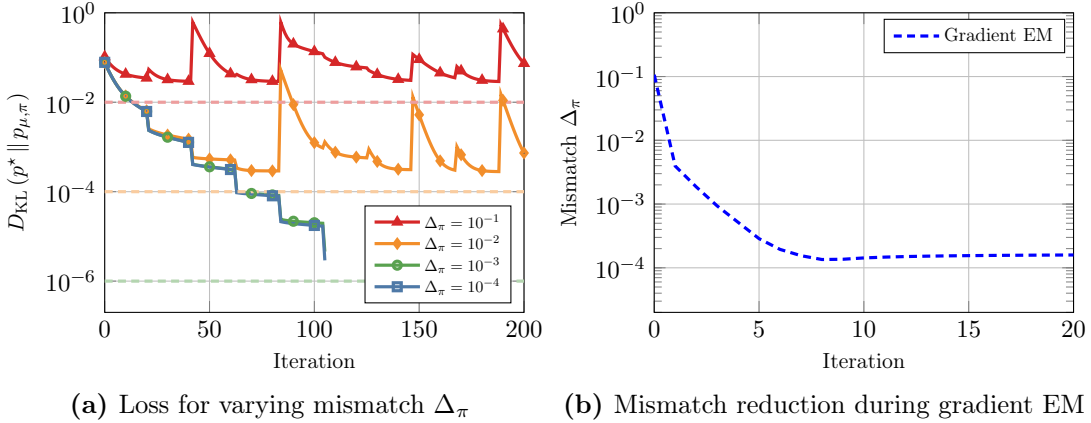

\paragraph{Experiment: model remains overparameterized.}
Our last experiments examines whether gradient EM can prune the student model down to a ``minimal'' parameterization with $m$ active student components. We sample a mixture with $m = 3$ teachers and fit a student mixture with $n = 20$ components using 50 iterations of gradient EM over several trials in dimension \(d \in \set{2,5,10}\);
we count the number of active student components, defined as satisfying \(\pi_{i} > \tfrac{1}{100}\). \Cref{fig:active-components} shows the number of active components is typically larger than \(m=3\), especially in low dimensions.

\begin{figure}[htbp]
	\centering
	\begin{tikzpicture}
		\begin{groupplot}[
				group style={
						group size=3 by 1,
						horizontal sep=0.5cm,
					},
				width=0.265\textwidth,
				height=0.215\textwidth,
				scale only axis,
				ybar,
				/pgf/bar width=6.5pt,
				xmin=2.5,
				xmax=13.5,
				ymin=0,
				ymax=13.5,
				xlabel={Active components},
				xtick={3,5,7,9,11,13},
				xticklabels={3,5,7,9,11,13+},
				ytick={0,5,10},
				tick label style={scale=0.75, nodes={transform shape}},
				label style={scale=0.75, nodes={transform shape}},
				title style={scale=0.75, nodes={transform shape}},
				axis line style={black!70},
				ymajorgrids=true,
				xmajorgrids=true,
				major grid style={draw=gray!25, line width=0.25pt},
				unbounded coords=jump,
				xtick pos=bottom,
			]

			\nextgroupplot[
				ylabel={Trials},
			]
			\addplot+[
				fill=blue!45,
				draw=black,
				line width=0.35pt,
				mark=none,
			] table[
					x expr={\thisrow{active_count} >= 13 ? 13 : \thisrow{active_count}},
					y=trial_count,
					col sep=comma,
					restrict expr to domain={\thisrow{dimension}}{2:2},
				] {figures/histogram.csv};

			\draw[black!75, densely dashed, line width=0.8pt]
			(axis cs:3,0) -- (axis cs:3,13.5);
			\node[anchor=north west] at (axis cs:3.25,13.1) {\scalebox{0.75}{$>m$}};
			\node[draw, fill=white, anchor=north east] at (rel axis cs:0.99,0.99) {\scalebox{0.75}{$d = 2$}};

			\nextgroupplot[
			]
			\addplot+[
				fill=blue!45,
				draw=black,
				line width=0.35pt,
				mark=none,
			] table[
					x expr={\thisrow{active_count} >= 13 ? 13 : \thisrow{active_count}},
					y=trial_count,
					col sep=comma,
					restrict expr to domain={\thisrow{dimension}}{5:5},
				] {figures/histogram.csv};

			\draw[black!75, densely dashed, line width=0.8pt]
			(axis cs:3,0) -- (axis cs:3,13.5);
            \node[draw, fill=white, anchor=north east] at (rel axis cs:0.99,0.99) {\scalebox{0.75}{$d = 5$}};

			\nextgroupplot[
			]
			\addplot+[
				fill=blue!45,
				draw=black,
				line width=0.35pt,
				mark=none,
			] table[
					x expr={\thisrow{active_count} >= 13 ? 13 : \thisrow{active_count}},
					y=trial_count,
					col sep=comma,
					restrict expr to domain={\thisrow{dimension}}{10:10},
				] {figures/histogram.csv};

			\draw[black!75, densely dashed, line width=0.8pt]
			(axis cs:3,0) -- (axis cs:3,13.5);
            \node[draw, semithick, fill=white, anchor=north east] at (rel axis cs:0.99,0.99) {\scalebox{0.75}{$d = 10$}};
		\end{groupplot}
	\end{tikzpicture}

	\caption{
		Active components after \(50\) iterations of gradient EM for a student with \(n=20\) components and a teacher with \(m = 3\) components. The histogram suggests that pruning can remove redundant student components, but does not completely eliminate overparameterization.
	}
	\label{fig:active-components}
\end{figure}
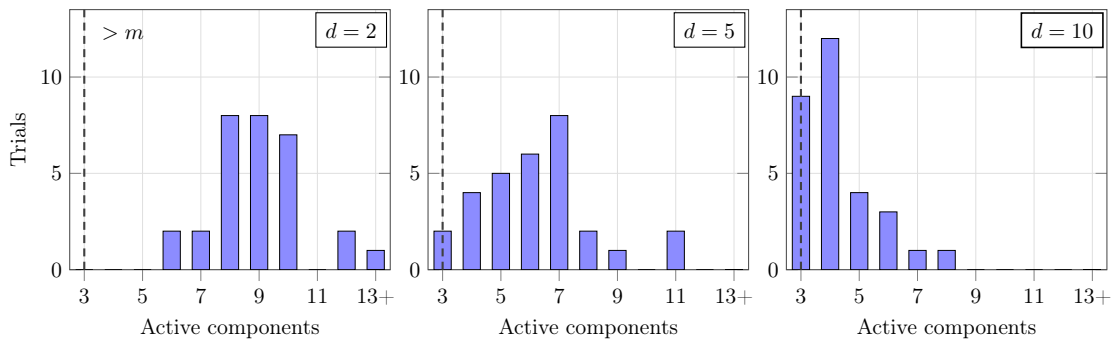

\subsection*{Discussion and future work}
We introduced a locally accelerated method for learning overparameterized GMMs, leveraging the ravine
geometry of the loss landscape near minimizers. Our work leaves open  several exciting research directions, such as: (i) extending our guarantees to the finite-sample setting;
(ii) obtaining precise quantitative estimates of the neighborhood of local acceleration; and (iii) analyzing a practical variant of~\cref{alg:switching-gdpolyak} that does not artificially separate the local stage by freezing the student weights.

\paragraph{Acknowledgements.}
The authors are grateful to Liwei Jiang, Mo Zhou and Weihang Xu for useful discussions. The work of VC was supported in part by grants from the NSF (DMS-2235451) and Simons Foundation (MP-TMPS-00005320) to the NSF-Simons National Institute for Theory and Mathematics in Biology (NITMB). The work of MF was
supported in part by awards NSF TRIPODS II
2023166, NSF CCF 2212261, NSF CCF 2312775, and by the Moorthy Family professorship at UW.

\bibliographystyle{plainnat}
\bibliography{main}

\clearpage

\appendix

\part{Appendix} %
\parttoc %

\section{Background}

\subsection{Notation and standing assumptions}
\label{sec:notation-appendix}

In this section we record additional notation used in the appendix.

We write $\ip{x, y} = \tr(x^{\T} y)$ for the Euclidean inner product and $\norm{x} = \sqrt{\ip{x, x}}$ for the induced norm.
We write $\mathbb{S}^{d-1}$ for the unit sphere in $d$ dimensions and $\Delta^{n-1}$ for the probability simplex in $n$ dimensions. We let $\opnorm{A} :=
	\sup_{x \in \mathbb{S}^{n-1}} \norm{Ax}$ denote the $\ell_2 \to \ell_2$ operator norm of a matrix $A \in \R^{m \times n}$, and write $A \otimes B$
for the Kronecker product between matrices $A$ and $B$. Finally, we write $[k]$ for the set $\set{1, \dots, k}$ and use
the notation $\mathcal{S} = \bigsqcup_{j} \mathcal{S}_j$ to indicate a union of disjoint sets: $\mathcal{S} = \bigcup_{j} \mathcal{S}_j$
with $\mathcal{S}_{j} \cap \mathcal{S}_{k} = \emptyset$.

Following~\cite{zhou2025global}, we make the following assumptions about $p^{\star}$.
\begin{assumption}[Teacher mixture] \label{assm:teacher-mixture}
	The ground-truth density $p^{\star} = p_{\mu^{\star}, \pi^{\star}}$
	is parameterized by $\mu^{\star} = ({\mu^{\star}_1}^{\T}, \dots, {\mu_{m}^{\star}}^{\T})^{\T} \in \R^{md}$ and
	$\pi^{\star} = (\pi_1^{\star}, \dots, \pi_{m}^{\star})^{\T} \in \Delta^{m-1}$ satisfying:
	\begin{description}
		\item[(A1)\label{assm:item:nondegeneracy}] (\textbf{Nondegeneracy}). Let $M^{\star} := \sum_{i = 1}^m \pi_{i}^{\star} \mu_{i}^{\star} {\mu_{i}^{\star}}^{\T}$.
		      There exist $\lambda_{\max}, \lambda_{\min} \geq 0$ such that

		      \[
			      0 < \lambda_{\min} =
			      \lambda_{m}(M^{\star})
			      \leq
			      \lambda_{1}(M^{\star})
			      = \lambda_{\max},
		      \]

		      where $\lambda_{1}(M^{\star}) \geq \dots \geq \lambda_{m}(M^{\star})$ are the nonzero eigenvalues of $M^{\star}$.
		\item[(A2)\label{assm:item:boundedness}] (\textbf{Boundedness}). The components of $\mu^{\star}$ satisfy
		      \[
			      D_{\max} \geq \norm{\mu_{i}^{\star}} \geq D_{\min} \geq
			      4 \cdot \frac{\lambda_{\max}}{\lambda_{\min}} \sqrt{dn}, \;\; \text{for all $i \in [m]$.}
		      \]
		\item[(A3)\label{assm:item:separation}] (\textbf{Separation}). Let $\Delta := \min_{i \neq j} \norm{\mu_{i}^{\star} - \mu_{j}^{\star}}$ and
		      $\pi_{\min}^{\star} = \min_{j} \pi_{j}^{\star}$; we have
		      \[
			      \Delta \geq C \cdot \max\set*{
				      \sqrt{D_{\max}} (dn)^{1/4}, \sqrt{\frac{d}{\pi_{\min}^{\star}}},
				      \sqrt{\log\left(
					      \frac{D_{\max} \cdot d n m}{\lambda_{\min} \pi_{\min}^{\star}}
					      \right)}
			      },
		      \]
		      for a sufficiently large constant $C > 0$.
	\end{description}
\end{assumption}

Note that Item~\ref{assm:item:nondegeneracy} implies that $\pi_{\min}^{\star} > 0$.

\subsection{Auxiliary results}
\label{sec:subsec:auxiliary-results}

In this section, we collect technical results about the KL loss used throughout our analysis.

\begin{lemma}[{Gradient of KL loss;~\citep[Lemma A.2]{zhou2025global}}] \label{lemma:mu-gradient}
	We have that

	\begin{equation}
		\grad_{\mu_i}\loss(\theta) =
		\mathbb{E}_{X \sim p^{\star}}\left[
			\psi_{i}(X; \theta)(\mu_{i} - X)
			\right] =
		\sum_{j = 1}^{m} \pi_{j}^{\star}\,\mathbb{E}_{X \sim \mathcal{N}(\mu_{j}^{\star}, \id_{d})}[\psi_{i}(X; \theta)(\mu_{i} - X)].
		\label{eq:mu-gradient}
	\end{equation}

	where $\loss(\theta)$ and $\psi_{i}$ are defined in~\eqref{eq:kl-loss} and~\eqref{eq:responsibilities}, respectively.
\end{lemma}

\begin{theorem}[Identifiability; {\citet[Theorem~B.1]{zhou2025global}}]
	\label{thm:mo-identifiability-input}
	Suppose~\cref{assm:teacher-mixture} is in force. There exist constants \(\varepsilon_0,C>0\) such that the following holds. Let \(\theta=(\mu,\pi)\), and suppose there is a partition
	\[
		[n]=\bigsqcup_{\ell=1}^m S_\ell
	\]
	such that every student in \(S_\ell\) lies in the local neighborhood of the \(\ell\)-th teacher mean. If
	\[
		\loss(\theta)\le \varepsilon \le \varepsilon_0,
	\]
	then for every \(\ell\in[m]\):
	\begin{subequations}
		\begin{align}
			\sum_{i\in S_\ell}\pi_i\|\mu_i-\mu_\ell^\star\|_2^2                      & \le C\sqrt{\loss(\theta)}, \label{eq:id-mean-quartic-growth}      \\
			\abs[\Big]{\sum_{i \in S_{\ell}} \pi_{i} -\pi_\ell^\star}                & \le C\sqrt{\loss(\theta)}, \label{eq:id-group-weight-quad-growth} \\
			\opnorm[\Big]{\sum_{i\in S_\ell}\pi_i\mu_i-\pi_\ell^\star\mu_\ell^\star} & \le C\sqrt{\loss(\theta)}. \label{eq:id-avg-mean-quad-growth}
		\end{align}
	\end{subequations}
	Moreover, letting $\delta := C \varepsilon^{1/4}$, we have the following inequality:
	\[
		\sum_{i\in S_\ell^{\mathrm{close}}(\delta)} \pi_i \ge \frac{1}{2}\pi_\ell^\star,
		\quad \text{for} \quad
		S_\ell^{\mathrm{close}}(\delta)
		:=
		\{i\in S_\ell:\|\mu_i-\mu_\ell^\star\|_2\le \delta\}.
	\]
\end{theorem}

\begin{corollary}[{\citet[Corollary~4.6]{davis2024gradient}}]
	\label{cor:ddj-morse-ravine-input}
	Let \(f\) be \(C^5\) near a minimizer \(\bar x\), and let \(\mathcal{S}^{\star}\) be the local solution set of \(f\) near \(\bar x\). Suppose that the following hold:
	\begin{itemize}
		\item the Hessian \(\nabla^2 f\) has constant rank on \(\mathcal{S}^{\star}\) near \(\bar x\);
		\item there exists \(D_{\mathsf{lb}}>0\) such that, for all $x$ near $\bar{x}$, it holds that

		      \begin{equation}
			      D_{\mathsf{lb}}\, \dist(x,\mathcal{S}^{\star})^4 \le f(x) - \min f.
			      \label{eq:ddj-4th-order-growth}
		      \end{equation}
	\end{itemize}
	Then \(f\) admits a local \(C^5\) Morse ravine near \(\bar x\), and satisfies~\cite[Assumption A]{davis2024gradient} there.
\end{corollary}

\begin{theorem}[{\normalfont\citep[Theorem~5.1]{davis2024gradient}}]
	\label{thm:ddj-local-convergence-input}
	Let \(f\) be \(C^2\) near a minimizer \(\bar x\),
	and suppose that Assumption~A of \citet{davis2024gradient} holds at \(\bar x\) with respect to a ravine \(\mathcal{M}\), with $C^{2}$ projection \(P_{\mathcal{M}}\) near \(\bar x\). Then there exist constants \(\delta_0,\eta_0,c,C>0\) such that for every initial point \(x_0\in B_{\delta_0}(\bar x)\), every stepsize \(\eta\in(0,\eta_0)\), and every pair of iteration counts \(K,I\in\mathbb N\), the output \(x_{\mathrm{out}}\) of the GD-Polyak method of \citet{davis2024gradient} satisfies
	\[
		f(x_{\mathrm{out}})-f^\star \le C e^{-c\eta \cdot \min\{K,I\}}.
	\]
	Moreover, the total number of gradient and function evaluations is at most \(I(K+1)\).
\end{theorem}

\begin{lemma}[Local Lipschitzness of logarithm]
	\label{lemma:log-lipschitz}
	Suppose that $x, y \geq \zeta > 0$. Then we have that
	\begin{align*}
		\abs{\log(x) - \log(y)} \leq \frac{\abs{x - y}}{\zeta}.
	\end{align*}
\end{lemma}
\begin{proof}
	Let $h(t)=\log t$. Fix $x \geq y$ without loss of generality. By the mean-value theorem, there is $t \in (y, x)$ such that

	\[
		h(x) - h(y) = h'(t) \cdot (x - y) = \frac{x - y}{t} \leq \frac{x - y}{\zeta}
	\]

	Repeating the above with the role of $x$ and $y$ reversed completes the proof.
\end{proof}

\begin{lemma}[Spectral norm of block matrix] \label{lemma:block-matrix-spectral-norm}
	Suppose that $A \in \mathbf{S}^{dn \times dn}$ is partitioned as
	\[
		A = \begin{bmatrix}
			A_{11}      & A_{12}      & \dots & A_{1n} \\
			A_{12}^{\T} & A_{22}      & \dots & A_{2n} \\
			\vdots      &             &       & \vdots \\
			A_{1n}^{\T} & A_{2n}^{\T} & \dots & A_{nn}
		\end{bmatrix}, \;\;
		\text{where} \;\; A_{ij} \in \mathbb{R}^{d \times d}.
	\]
	Then it holds that $\opnorm{A} \leq n \cdot \max_{i, j} \opnorm{A_{ij}}$.
\end{lemma}
\begin{proof}
	The spectral norm of $A$ is given by $\opnorm{A} = \sup_{u: \norm{u} = 1} \abs{\ip{u, Au}}$. For any such $u$,
	\begin{align*}
		\abs{\ip{u, Au}} & = \abs[\Big]{\sum_{i = 1}^{n} \ip{u_{i}, \sum_{j = 1}^{n} A_{ij} u_{j}}}             \\
		                 & \leq \sum_{i = 1}^{n} \norm{u_{i}} \norm[\Big]{\sum_{j = 1}^{n} A_{ij} u_{j}}        \\
		                 & \leq \sqrt{\sum_{i = 1}^{n} \norm[\Big]{\sum_{j = 1}^{n} A_{ij} u_j}^2}              \\
		                 & \leq \sqrt{\sum_{i = 1}^{n} \Big(\sum_{j = 1}^{n} \opnorm{A_{ij}} \norm{u_j}\Big)^2} \\
		                 & \leq \sqrt{\sum_{i = 1}^{n} n \cdot \max_{j} \opnorm{A_{ij}}^2}                      \\
		                 & \leq n \cdot \max_{i, j} \opnorm{A_{ij}},
	\end{align*}
	by repeatedly applying the Cauchy-Schwarz and H{\"o}lder inequalities.
\end{proof}

\begin{theorem}[{\citep[Corollary~5.4 and Lemma~5.9]{davis2024gradient}}]
	\label{thm:ddj-ravine-step-input}
	Let \(f\) be \(C^2\) near a minimizer \(\bar x\), let \(\mathcal{S}^{\star}\) be the set of minimizers of \(f\), and suppose that
	Assumption~A of \citet{davis2024gradient} holds at \(\bar x\) with respect to a ravine \( \mathcal{M} \), with \(C^{2}\) projection \( P_{\mathcal{M}} \)
	near \(\bar x\). Let
	\[
		f(x) - f^\star = \underbrace{f(P_{\mathcal{M}}(x))}_{f_{T}(x)} + \underbrace{f(x) - f(P_{\mathcal{M}}(x))}_{f_{N}(x)}
	\]
	denote the corresponding normal/tangent decomposition, and let \(p\) be the order appearing in
	Assumption~A. Then there exist a neighborhood \(U\) of \(\bar x\) and constants
	\[
		D_{\mathrm{lb}},D_{\mathrm{ub}},\beta_{\mathrm{lb}},\beta_{\mathrm{ub}}>0,
		\qquad
		q_{\mathrm P}\in(0,1),
		\qquad
		C_{\mathrm P}>0,
	\]
	such that the following hold:
	\begin{enumerate}[(i)]
		\item \label{item:tangent-part-p-growth} For every \(y\in U\cap \mathcal{M}\), the following holds:
		      \begin{subequations}
			      \begin{align}
				      D_{\mathrm{lb}} \dist(y,\mathcal{S}^{\star})^p        & \le f_T(y)-f^\star \le D_{\mathrm{ub}} \dist(y,\mathcal{S}^{\star})^p,                \\
				      \beta_{\mathrm{lb}}\dist(y,\mathcal{S}^{\star})^{p-1} & \leq \norm{\grad f_T(y)} \leq \beta_{\mathrm{ub}} \dist(y,\mathcal{S}^{\star})^{p-1}.
			      \end{align}
		      \end{subequations}
		\item \label{item:small-tangent-gradient-implies-descent} For every \(x\in U\), setting \( y := P_{\mathcal{M}}x \), let
		      \[
			      x^+ := x- \frac{f(x)-f^\star}{\norm{\nabla f(x)}^2} \nabla f(x),
			      \quad
			      y^+ := P_{\mathcal{M}} x^+,
		      \]

		      Then if $\norm{\grad f_{N}(x)} \leq \tfrac{1}{100} \norm{\grad f_{T}(y)}$, the following holds:

		      \begin{equation*}
			      \dist(y^+,\mathcal{S}^{\star}) \leq q_{\mathrm{P}} \dist(y,\mathcal{S}^{\star}),
			      \quad \text{and} \quad
			      \dist(x^+,\mathcal{M}) \leq C_{\mathrm{P}} \dist(y,\mathcal{S}^{\star}).
		      \end{equation*}
	\end{enumerate}
\end{theorem}

\medskip

\begin{corollary}[{\cite[Lemma 6.2]{davis2024gradient}}]
	\label{cor:ravine-short-step-input}
	Under the assumptions of Theorem~\ref{thm:ddj-ravine-step-input}, for every \(\rho>0\), after shrinking \(U\) if necessary there exist
	constants \(c_{\mathrm G}, \eta_{\mathrm G}>0\) such that the following holds for every \(x\in U\) and
	every \(0<\eta<\eta_{\mathrm G}\):
	\[
		\norm{x - \eta \nabla f(x) - \proj_{\mathcal{M}}(x)}
		\leq
		(1 - c_{\mathrm{G}} \eta) \dist_{\mathcal{M}}(x)
		+
		\rho \eta \cdot \dist^{p-1}(P_{\mathcal{M}}(x), \mathcal{S}^{\star}).
	\]
\end{corollary}

\section{Ravine geometry under exact clustering}
\label{app:identified-ravine}

In this section, we study the geometry of the loss landscape near the optimal solution set $\Theta^{\star}$.
Before we proceed, we establish some notation.
Following~\cite{zhou2025global}, we define $S_{\ell} \subset [n]$ as the set of mean vectors closest
to the mean of the $\ell^{\text{th}}$ teacher component:
\begin{equation}
	S_{\ell} := \set[\big]{\mu \in \R^d \mid \norm{\mu - \mu_{\ell}^{\star}} \leq \min_{j \neq \ell}\,\norm{\mu - \mu_{j}^{\star}}}.
	\label{eq:nearest-mean-assignment}
\end{equation}
Additionally, we write $S_{\ell}(\delta) := S_{\ell} \cap \mathcal{B}(\mu_{\ell}^{\star}; \delta)$ for the
$\delta$-close elements in $S_{\ell}$.

Any optimal parameterization $\theta^{\star} \in \Theta^{\star}$ induces a partition $[n] = \bigcup_{\ell=1}^m S_{\ell}$.
When that parameterization contains components with $\pi_{i} = 0$, or when $S_{\ell} \cap S_{k} \neq \emptyset$,
such a partition need not be unique. The following definition rules out such cases.

\begin{definition}[Exact clustering at minimizer] \label{definition:exact-grouping}
	Fix $\theta^{\star} = (\widetilde{\mu}, \widetilde{\pi}) \in \Theta^{\star}$. We say that $\theta^{\star}$
	induces an \emph{exact clustering}, $\set{S_{\ell}}_{\ell=1}^m$, if the following conditions hold:
	\begin{enumerate}[(i)]
		\item \textbf{Mean alignment:} for every \(\ell\in[m]\) and every \(i\in S_\ell\),

		      \begin{equation}
			      \widetilde{\mu}_i = \mu_\ell^\star.
		      \end{equation}
		\item \textbf{Weight aggregation:} for every \(\ell = 1, \dots, m\),

		      \begin{equation}
			      \sum_{i\in S_\ell}\widetilde{\pi}_i = \pi_\ell^\star.
		      \end{equation}
	\end{enumerate}
	In this case, we have an exact partition $[n] = \bigsqcup_{\ell=1}^m S_{\ell}$.
\end{definition}

\begin{definition}[Cluster averaging operator and tangent subspace]
	\label{def:grouped-averaging}
	Fix an exact clustering according to~\cref{definition:exact-grouping} and its corresponding
	partition and define the averaging operator
	\[
		\cavg: \mathbb{R}^{dn} \to \mathbb{R}^{dm},
		\qquad
		\cavg_\ell(u):=\sum_{i\in S_\ell}\pi_i u_i,
		\qquad \ell\in[m].
	\]
	We denote the corresponding cluster-mean-preserving subspace by
	\begin{equation}
		T:=\ker(\cavg)
		=
		\left\{
		u=(u_1,\ldots,u_n)\in\mathbb R^{dn}:
		\sum_{i\in S_\ell}\pi_i u_i=0
		\text{ for every }\ell\in[m]
		\right\}.
		\label{eq:mean-preserving-subspace}
	\end{equation}
\end{definition}

Definition~\ref{definition:exact-grouping} specifies the exact clustering structure at the minimizer, while Definition~\ref{def:grouped-averaging} extracts the corresponding linear geometry. The operator \(\cavg\) records the weighted cluster means; its kernel comprises all perturbations that preserve cluster means. The next result shows that, at such minimizers, the subspace $T$ coincides with the nullspace of the Hessian of $\loss$.

\begin{proposition}[Hessian block formula at a minimizer]
	\label{prop:hessian-block-identified}
	Assume \(\theta^\star=(\mu^\star,\pi)\in\Theta_\star\) admits an exact clustering with induced partition \([n]=\bigsqcup_{\ell=1}^m S_\ell\). Then, for $i, j \in [n]$,
	we have
	\[
		\grad^2_{\mu_i \mu_j} \loss(\theta^{\star})
		=
		\E_{X\sim p^\star}\!\Big[
			\psi_i(X;\theta^\star)\psi_j(X;\theta^\star)
			(X-\mu_i^\star)(X-\mu_j^\star)^\T
			\Big].
	\]
	In particular, if \(i\in S_\ell\) and \(j\in S_k\), it follows that
	\[
		\grad^2_{\mu_i\mu_j} \loss(\theta^\star)
		=
		\frac{\pi_i\pi_j}{\pi_\ell^\star\pi_k^\star}\Gamma_{\ell k},
		\qquad
		\Gamma_{\ell k}
		:=
		\E_{X\sim p^\star}\!\Big[
			\gamma_\ell(X)\gamma_k(X)
			(X-\mu_\ell^\star)(X-\mu_k^\star)^\T
			\Big],
	\]
	where we define the reduced model responsibilities $\gamma_{\ell}$ as
	\[
		\gamma_\ell(x):=\frac{\pi_\ell^\star \phi(x\mid \mu_\ell^\star)}{p^\star(x)}.
	\]
\end{proposition}

\begin{proof}[Proof of Proposition~\ref{prop:hessian-block-identified}]
	Below, we write $H_{ij}(\theta) := \grad^2_{\mu_i \mu_j} \loss(\theta)$.
	By~\cref{lemma:mu-gradient}, we have that
	\[
		\nabla_{\mu_i}\loss(\theta)
		=
		-\E_{X\sim p^\star}\!\big[\psi_i(X;\theta)(X-\mu_i)\big].
	\]
	To differentiate the responsibility, write
	\[
		\psi_i(x;\theta)=\frac{N_i(x)}{Z(x)},
		\qquad
		N_i(x):=\pi_i\phi(x\mid \mu_i),
		\qquad
		Z(x):=\sum_{k=1}^n N_k(x).
	\]
	Then
	\[
		\nabla_{\mu_j}N_i(x)=\mathbf{1}_{\{i=j\}}N_i(x)(x-\mu_j),
		\qquad
		\nabla_{\mu_j}Z(x)=N_j(x)(x-\mu_j).
	\]
	Hence, by the quotient rule,
	\[
		\begin{aligned}
			\nabla_{\mu_j}\psi_i(x;\theta)
			 & =
			\frac{\nabla_{\mu_j}N_i(x)}{Z(x)}
			-\frac{N_i(x)\nabla_{\mu_j}Z(x)}{Z(x)^2}         \\
			 & =
			\mathbf{1}_{\{i=j\}}\frac{N_i(x)}{Z(x)}(x-\mu_j)
			-\frac{N_i(x)}{Z(x)}\frac{N_j(x)}{Z(x)}(x-\mu_j) \\
			 & =
			\psi_i(x;\theta)\big(\mathbf{1}_{\{i=j\}}-\psi_j(x;\theta)\big)(x-\mu_j).
		\end{aligned}
	\]
	Continuing with the calculation of the second derivative, we obtain
	\[
		\begin{aligned}
			H_{ij}(\theta)
			 & =
			\nabla_{\mu_j}(\nabla_{\mu_i}\loss(\theta)) \\
			 & =
			-\E_{X\sim p^\star}\!\Big[
				\nabla_{\mu_j}\big(\psi_i(X;\theta)(X-\mu_i)\big)
			\Big]                                       \\
			 & =
			-\E_{X\sim p^\star}\!\Big[
				(X-\mu_i)\big(\nabla_{\mu_j}\psi_i(X;\theta)\big)^\T
				-\mathbf{1}_{\{i=j\}}\psi_i(X;\theta)I_d
			\Big]                                       \\
			 & =
			-\,\E_{X\sim p^\star}\!\Big[
				\psi_i(X;\theta)\big(\mathbf{1}_{\{i=j\}}-\psi_j(X;\theta)\big)
				(X-\mu_i)(X-\mu_j)^\T
			\Big]                                       \\
			 & \qquad
			+\,\mathbf{1}_{\{i=j\}}\E_{X\sim p^\star}\!\big[\psi_i(X;\theta)\big]I_d.
		\end{aligned}
	\]
	We now carry out the calculation for \(\theta \equiv \theta^\star\in\Theta_\star\). Since \(p_{\theta^\star}=p^\star\),
	\[
		\begin{aligned}
			\E_{X\sim p^\star}\!\Big[\psi_i(X;\theta^\star)(X-\mu_i^\star)(X-\mu_i^\star)^\T\Big]
			 & =
			\int \pi_i\phi(x\mid \mu_i^\star)(x-\mu_i^\star)(x-\mu_i^\star)^\T\,dx \\
			 & =
			\E_{X\sim p^\star}\!\big[\psi_i(X;\theta^\star)\big]I_d.
		\end{aligned}
	\]
	Substituting this identity into the diagonal blocks yields
	\[
		H_{ij}(\theta^\star)
		=
		\E_{X\sim p^\star}\!\Big[
			\psi_i(X;\theta^\star)\psi_j(X;\theta^\star)
			(X-\mu_i^\star)(X-\mu_j^\star)^\T
			\Big].
	\]
	Finally, if \(i\in S_\ell\) and \(j\in S_k\), the exact clustered representation gives \(\mu_i^\star=\mu_\ell^\star\) and \(\mu_j^\star=\mu_k^\star\). Hence
	\[
		\psi_i(x;\theta^\star)=\frac{\pi_i}{\pi_\ell^\star}\gamma_\ell(x),
		\qquad
		\psi_j(x;\theta^\star)=\frac{\pi_j}{\pi_k^\star}\gamma_k(x),
	\]
	which gives the factorized expression.
\end{proof}

\begin{lemma}[Positive definiteness of the collapsed model Hessian]
	\label{lem:collapsed-hessian-pd}
	Define
	\[
		\loss_{\mathrm{coll}}(\nu)
		:=
		\KL{p^\star}{\sum_{\ell=1}^m \pi_\ell^\star \phi(\,\cdot\,\mid \nu_\ell)},
		\qquad
		\nu=(\nu_1,\dots,\nu_m)\in \R^{dm},
	\]
	and let
	\(
	\nu^\star:=(\mu_1^\star,\dots,\mu_m^\star)
	\)
	denote the exact teacher mean vector. Assume that Item~\ref{assm:item:nondegeneracy} holds;
	then the Hessian of the collapsed exact-parametrized loss at \(\nu^\star\) is positive definite:
	\[
		\nabla^2 \loss_{\mathrm{coll}}(\nu^\star)\succ 0.
	\]
\end{lemma}
\begin{proof}
	A standard statistical result states that the Hessian of the KL divergence evaluated at the optimal parameters $\nu^\star$ coincides with the Fisher Information Matrix (FIM) of this model \citep{kullback1997information}.
	Therefore, proving $\nabla^2 \loss_{\mathrm{coll}}(\nu^\star) \succ 0$ is equivalent to proving that the FIM is non-singular. For parametric models satisfying standard smoothness and regularity conditions (which the Gaussian family easily satisfies), the FIM is strictly positive definite if and only if the model is locally identifiable \citep{rothenberg1971identification}.

	A classic result by \citet{teicher1963identifiability} establishes that finite mixtures of Gaussians are strictly identifiable. Under Item~\ref{assm:item:nondegeneracy}, the mapping from the means $\nu$ to the density $p_\nu$ is unique. This strict identifiability guarantees that the Fisher Information Matrix is positive definite, concluding the proof.
\end{proof}

\begin{theorem}[Ravine geometry at minimizer]
	\label{thm:exact-grouped-ravine-geometry}
	Assume \(\theta^\star=(\mu^\star,\pi)\in\Theta_\star\) admits an exact clustering with partition \([n]=\bigsqcup_{\ell=1}^m S_\ell\) and let \(\cavg\) be the operator from Definition~\ref{def:grouped-averaging}. Then

	\[
		\ker\bigl(\nabla^2_{\mu\mu}\loss(\theta^\star)\bigr)= \ker(\cavg).
	\]
\end{theorem}
\begin{proof}
	We introduce the following notation for simplicity:
	\[
		H:=\nabla^2_{\mu\mu}\loss(\theta^\star),
		\qquad
		\nu^\star:=(\mu_1^\star,\dots,\mu_m^\star),
		\qquad
		\bar H:=\nabla^2 \loss_{\mathrm{coll}}(\nu^\star).
	\]
	By Lemma~\ref{lem:collapsed-hessian-pd}, we have \(\bar H\succ 0\). For \(u=(u_1,\dots,u_n)\in(\R^d)^n\), define
	\[
		v_\ell:=\frac{1}{\pi_\ell^\star}\sum_{i\in S_\ell}\pi_i u_i
		=
		\frac{1}{\pi_\ell^\star}\cavg_{\ell}(u),
		\qquad
		v=(v_1,\dots,v_m)\in(\R^d)^m.
	\]
	Fix \(i\in S_\ell\). Then Proposition~\ref{prop:hessian-block-identified} gives
	\[
		\begin{aligned}
			(Hu)_i
			 & =
			\sum_{k=1}^m\sum_{j\in S_k}H_{ij}u_j                           \\
			 & =
			\sum_{k=1}^m\sum_{j\in S_k}
			\frac{\pi_i\pi_j}{\pi_\ell^\star\pi_k^\star}\Gamma_{\ell k}u_j \\
			 & =
			\frac{\pi_i}{\pi_\ell^\star}\sum_{k=1}^m
			\Gamma_{\ell k}
			\left(
			\frac{1}{\pi_k^\star}\sum_{j\in S_k}\pi_j u_j
			\right)                                                        \\
			 & =
			\frac{\pi_i}{\pi_\ell^\star}\sum_{k=1}^m \Gamma_{\ell k}v_k.
		\end{aligned}
	\]
	Clearly, if $u \in \ker(\cavg)$, we have $v = 0$. Consequently, $Hu = 0$ and thus $\mathcal{T} \subset \ker(H)$.
	Conversely,
	\[
		u \in \ker(H) \implies 0 = (Hu)_i = \frac{\pi_i}{\pi_\ell^\star}\sum_{k=1}^m \Gamma_{\ell k}v_k
		\qquad
		\text{for every }i\in S_\ell.
	\]
	Since \(\pi_i>0\) by assumption, this implies
	\[
		\sum_{k=1}^m \Gamma_{\ell k}v_k=0
		\qquad
		\text{for every }\ell\in[m].
	\]
	Applying Proposition~\ref{prop:hessian-block-identified} to the exact-parametrized \(m\)-component model shows that the \((\ell,k)\) block of \(\bar H\) is exactly \(\Gamma_{\ell k}\), so \(\bar H v=0\). Because \(\bar H\succ0\), we must have \(v=0\), which is equivalent to
	\[
		\sum_{i\in S_\ell}\pi_i u_i=0
		\qquad
		\text{for every }\ell\in[m].
	\]
	Hence \(u\in\mathcal T\), proving the reverse inclusion.

	Finally, we prove that \(\cavg: \R^{dn}\to \R^{dm}\) has rank \(md\): for each \(\ell\), the map \((u_i)_{i\in S_\ell}\mapsto \sum_{i\in S_\ell}\pi_i u_i\) is surjective onto \(\R^d\) because \(\pi_i>0\) and \(S_\ell\neq\varnothing\). Therefore
	\[
		\dim \mathcal T
		=
		nd-md
		=
		(n-m)d.
	\]
\end{proof}

\begin{corollary}[Cluster bias / dispersion decomposition]
	\label{cor:grouped-bias-dispersion}
	For any \(\mu \in \R^{dn}\), we have
	\begin{align}
		                           & \sum_{\ell=1}^m\sum_{i\in S_\ell}\pi_i\norm{\mu_i-\mu_\ell^\star}^2
		=
		B(\mu)+D(\mu), \notag                                                                                       \\
		\text{where} \qquad B(\mu) & := \sum_{\ell=1}^m \widehat\pi_\ell\norm{\bar\mu_\ell-\mu_\ell^\star}^2, \quad
		D(\mu) := \sum_{\ell=1}^m\sum_{i\in S_\ell}\pi_i\norm{\mu_i-\bar\mu_\ell}^2.
		\label{eq:bias-dispersion-defn}
	\end{align}
\end{corollary}

\begin{proof}[Proof of Corollary~\ref{cor:grouped-bias-dispersion}]
	For each \(\ell\in[m]\), write
	\[
		\mu_i-\mu_\ell^\star=(\mu_i-\bar\mu_\ell)+(\bar\mu_\ell-\mu_\ell^\star).
	\]
	Expanding the squared norm, multiplying by \(\pi_i\), and summing over \(i\in S_\ell\) gives
	\[
		\sum_{i\in S_\ell}\pi_i\norm{\mu_i-\mu_\ell^\star}^2
		=
		\widehat\pi_\ell\norm{\bar\mu_\ell-\mu_\ell^\star}^2
		+
		\sum_{i\in S_\ell}\pi_i\norm{\mu_i-\bar\mu_\ell}^2,
	\]
	because the cross term vanishes:
	\[
		\sum_{i\in S_\ell}\pi_i(\mu_i-\bar\mu_\ell)=0.
	\]
	Summing over \(\ell\in[m]\) yields the claimed decomposition.
\end{proof}

\begin{theorem}[Local quartic growth]
	\label{thm:exact-fixed-weight-fourth-order}
	Assume the hypotheses of Theorem~\ref{thm:mo-identifiability-input}, and let
	\[
		f_\pi(\mu):=\loss(\mu,\pi),
	\]
	where \(\theta^\star=(\mu^\star,\pi)\in\Theta_\star\) is a global minimizer inducing a partition as in~\cref{definition:exact-grouping}. Then there exist a neighborhood \(U\) of \(\mu^\star\) and a constant \(c_4>0\) such that
	\[
		f_\pi(\mu)-f_\pi(\mu^\star)\ge c_4\|\mu-\mu^\star\|_2^4
		\qquad
		\forall \mu\in U.
	\]
\end{theorem}
\begin{proof}
	Because \(\theta^\star\) is a global minimizer, \(f_\pi(\mu^\star)=\loss^\star\). By continuity of \(f_\pi\) and of the student means, we may choose a neighborhood \(U\) of \(\mu^\star\) such that for every \(\mu\in U\), the partition \([n]=\bigsqcup_{\ell=1}^m S_\ell\) remains valid and
	\[
		f_\pi(\mu)-f_\pi(\mu^\star)\le \varepsilon_0,
	\]
	where \(\varepsilon_0\) is the threshold in Theorem~\ref{thm:mo-identifiability-input}. Then, for any $\mu \in U$, we have
	\[
		\sum_{\ell=1}^m\sum_{i\in S_\ell}\pi_i\|\mu_i-\mu_\ell^\star\|_2^2
		\le
		C\sqrt{\loss(\mu, \pi)}.
	\]
	On the other hand, we have the following lower bound:
	\[
		\sum_{\ell=1}^m\sum_{i\in S_\ell}\pi_i\|\mu_i-\mu_\ell^\star\|_2^2
		\ge
		\pi_{\min}\|\mu-\mu^\star\|_2^2,
		\qquad
		\pi_{\min}:=\min_{i\in[n]}\pi_i>0.
	\]
	Combining the last two displays and squaring both sides yields
	\[
		\loss(\mu, \pi) \gtrsim
		\pi_{\min}^2 \norm{\mu-\mu^\star}^4.
	\]
	This proves the claim.
\end{proof}

\begin{corollary}[Existence of ravine]
	\label{cor:exact-fixed-weight-ravine}
	Under the assumptions of Theorem~\ref{thm:exact-fixed-weight-fourth-order}, \(f_\pi\) admits a local \(C^\infty\) Morse ravine near $\mu^{\star}$ and satisfies~\citep[Assumption A]{davis2024gradient} at that point.
\end{corollary}
\begin{proof}
	Theorem~\ref{thm:exact-fixed-weight-fourth-order} implies that \(\mu^\star\) is the unique minimizer of \(f_\pi\) in a sufficiently small neighborhood of \(\mu^\star\). Since \(f_\pi\) is \(C^\infty\), the local solution set is the singleton \(\{\mu^\star\}\), so the Hessian rank is constant on the solution set. Corollary~\ref{cor:ddj-morse-ravine-input} therefore applies.
\end{proof}

\subsection{Single-teacher geometry}
\label{app:sec:subsec:single-teacher-ravine}
In this section, we analyze the geometry of a simpler problem; namely, learning a single
Gaussian with an over-parameterized Gaussian mixture. We normalize the single teacher to
\[
	p^\star(x)=\phi(x\mid 0),
\]
which is without loss of generality by translation. The student weights \(\pi\in\Delta^{n-1}\) are fixed and strictly positive, and the corresponding fixed-weight minimizer is
\[
	\theta^\star=((0,\dots,0),\pi).
\]
In this case, the averaging operator from~\cref{def:grouped-averaging} reduces to the weighted mean.

\begin{proposition}[Single-teacher Hessian and tangent space]
	\label{prop:single-teacher-hessian}
	At the single-teacher minimizer \(\theta^\star=((0,\dots,0),\pi)\), the Hessian with respect to the student means satisfies
	\[
		\nabla^2_{\mu\mu}\loss(\theta^\star)
		=
		(\pi\pi^\T)\otimes I_d.
	\]
	Consequently, we have that $\rank(\grad^2_{\mu\mu}\loss(\theta^{\star})) = d$, with
	\[
		\ker\bigl(\nabla^2_{\mu\mu}\loss(\theta^\star)\bigr)
		=
		\Bigl\{
		u=(u_1,\dots,u_n)\in(\R^d)^n:
		\sum_{i=1}^n \pi_i u_i=0
		\Bigr\}.
	\]
\end{proposition}
\begin{proof}
	At \(\theta^\star\), every student mean equals the teacher mean. Consequently,
	\[
		\psi_i(x;\theta^\star)
		=
		\frac{\pi_i\phi(x\mid 0)}{\sum_{k=1}^n \pi_k\phi(x\mid 0)}
		=
		\pi_i.
	\]
	Applying Proposition~\ref{prop:hessian-block-identified} with \(m=1\) therefore gives
	\[
		H_{ij}(\theta^\star)
		=
		\E_{X\sim p^\star}\!\big[\pi_i\pi_j XX^\T\big]
		=
		\pi_i\pi_j I_d,
	\]
	because \(X\sim\mathcal N(0,I_d)\) under \(p^\star\). Hence
	\[
		\nabla^2_{\mu\mu}\loss(\theta^\star)
		=
		(\pi\pi^\T)\otimes I_d = (\pi \otimes I_d)(\pi \otimes I_d)^{\T}.
	\]
	By standard properties of the Kronecker product, for any $u \in \R^{dn}$, we have
	\[
		(\pi \otimes I_d)^{\T} u = \sum_{i = 1}^{n} \pi_{i} u_{i} \implies
		\ker(\grad_{\mu\mu}^2 \loss(\theta^{\star})) = \ker(\pi^{\T} \otimes I_d) =
		\set*{u \in \R^{dn} \mid \sum_{i = 1}^{n} \pi_{i} u_{i} = 0}.
	\]
	The rank property is immediate from $\rank(A \otimes B) = \rank(A) \cdot \rank(B)$.
\end{proof}

For \(\mu=(\mu_1,\dots,\mu_n)\in(\R^d)^n\), define the weighted mean
\[
	\bar\mu:=\sum_{i=1}^n \pi_i\mu_i.
\]
The following Corollary is an easy consequence of~\cref{cor:single-teacher-bias-dispersion}.

\begin{corollary}[Single-teacher bias/dispersion decomposition]
	\label{cor:single-teacher-bias-dispersion}
	For every \(\mu \in \R^{dn}\),
	\[
		\sum_{i=1}^n \pi_i \norm{\mu_i-\mu^\star}_2^2
		=
		\norm{\bar\mu-\mu^\star}_2^2
		+
		\sum_{i=1}^n \pi_i \norm{\mu_i-\bar\mu}_2^2.
	\]
\end{corollary}

\begin{proposition}[Quartic growth for fixed weights]
	\label{prop:single-teacher-quartic-general}
	Define the quantities
	\[
		\pi_{\min}:=\min_{i\in[n]}\pi_i,
		\qquad
		i_{\max}\in\arg\max_{i\in[n]}\norm{\mu_i}_2,
		\qquad
		\mu_{\max}:=\mu_{i_{\max}}.
	\]
	Then there exists a constant \(C>0\), depending only on the model parameters, such that
	\[
		\loss(\mu,\pi)
		\geq
		C\bigl(\pi_{\min}\norm{\mu_{\max}}_2^2\bigr)^2.
	\]
\end{proposition}

\begin{proof}[Proof of Proposition~\ref{prop:single-teacher-quartic-general}]
	From~\citet[Lemma B.5]{zhou2025global}, which is applicable as long as the weights are bounded, there exists a constant \(D_{\max}>0\) such that for every \(v\in\R^d\) with \(\norm{v}_2=1\),
	\[
		\loss(\mu,\pi)
		\geq
		D_{\max}^{-4}
		\left(
		\sum_{i=1}^n \pi_i \langle \mu_i,v\rangle^2
		\right)^2.
	\]
	If \(\mu_{\max}=0\), then every \(\mu_i=0\), so both sides of the desired inequality vanish and there is nothing to prove. Hence we may assume \(\mu_{\max}\neq 0\) in what follows.
	Choosing
	\[
		v:=\frac{\mu_{\max}}{\norm{\mu_{\max}}_2},
	\]
	we deduce the following inequality:
	\[
		\sum_{i=1}^n \pi_i \langle \mu_i,v\rangle^2
		\geq
		\pi_{\min}\max_{i\in[n]}\langle \mu_i,v\rangle^2
		\geq
		\pi_{\min}\langle \mu_{\max},v\rangle^2
		=
		\pi_{\min}\norm{\mu_{\max}}_2^2.
	\]
	Substituting this bound into the previous display gives
	\[
		\loss(\mu,\pi)
		\geq
		D_{\max}^{-4}
		\bigl(\pi_{\min}\norm{\mu_{\max}}_2^2\bigr)^2.
	\]
	The claim follows with \(C:=D_{\max}^{-4}\).
\end{proof}

\section{Geometry under weight perturbation}
\label{app:reference-weight-perturbation}

In this section, we show that small mismatches in the grouped student weight vector $\widehat{\pi} \in \Delta^{m-1}$ (relative
to the ground truth weights, $\pi^{\star} \in \Delta^{m-1}$) induce small perturbations in the loss function and
its derivatives. Assuming that $\loss(\mu, \pi) \leq \varepsilon_0$ (as it appears in~\cref{thm:mo-identifiability-input})
and $\pi \in \intr{\Delta^{n-1}}$, so that every coordinate $\pi_{i} > 0$, we define the following quantities:
\begin{subequations}
	\begin{align}
		\widehat{\pi}_\ell    & := \sum_{i\in S_\ell}\pi_i; \label{eq:grouped-weight}                                                                 \\
		\Delta_{\pi}^{(\ell)} & := \frac{\abs{\widehat{\pi}_{\ell} - \pi_{\ell}^{\star}}}{\pi_{\ell}^{\star}}; \label{eq:intra-group-weight-mismatch} \\
		\Delta_{\pi}          & := \max_{\ell\in[m]} \Delta_{\pi}^{\ell}. \label{eq:worst-case-group-weight-mismatch}
	\end{align}
\end{subequations}
Since $\pi \in \intr{\Delta^{n-1}}$, it is immediate that $\widehat{\pi} \in \intr{\Delta^{m-1}}$.

The basis of our comparison is a \emph{reference} loss, in which the weight vector $\pi$ is replaced by
a scaled version $\bar{\pi}$ that matches the ground truth $\pi^{\star}$ over each group.

\begin{definition}[Reference group weights]
	\label{def:reference-exact-grouped-weights}
	For any pair $\pi \in \intr{\Delta^{n-1}}$, we define
	the \emph{reference group weights} \( \bar{\pi} \in \intr{\Delta^{n-1}} \) as follows:
	\begin{equation}
		\bar{\pi}_i := \pi_{i} \cdot \frac{\pi_{\ell}^{\star}}{\sum_{j \in S_{\ell}} \pi_{j}}
		= \pi_{i} \cdot \frac{\pi_{\ell}^{\star}}{\widehat{\pi}_{\ell}}, \quad
		\text{for all $i \in S_{\ell}$ and $\ell \in [m]$.}
		\label{eq:reference-group-weights}
	\end{equation}
  In particular, we have $\sum_{i \in S_{\ell}} \bar{\pi}_i = \pi_{\ell}^{\star}$.
\end{definition}
The next Lemma bounds the $\ell_{1}$ distance between $\pi$ and $\bar{\pi}$ in terms of worst-case group mismatch.
\begin{lemma}[Distance to reference group weights] \label{lemma:distance-exact-group-weights}
	For $\bar{\pi}$ defined in~\cref{def:reference-exact-grouped-weights}, it holds that
	\begin{equation}
		\label{eq:distance-to-group-ref-weights}
		\norm{\bar{\pi} - \pi}_1 = \sum_{\ell = 1}^{m} \abs{\widehat{\pi}_{\ell} - \pi_{\ell}^{\star}} \leq \Delta_{\pi}.
	\end{equation}
\end{lemma}
\begin{proof}
	Expanding the definition of $\bar{\pi}$, we obtain
	\begin{align*}
		\norm{\bar{\pi}-\pi}_1
		 & =
		\sum_{\ell = 1}^{m} \sum_{i \in S_{\ell}} \abs{\bar{\pi}_i - \pi}                                                                 \\
		 & = \sum_{\ell = 1}^{m} \sum_{i \in S_{\ell}} \abs[\Big]{\pi_{i} \Big(\tfrac{\pi_{\ell}^{\star}}{\widehat{\pi}_{\ell}} - 1\Big)} \\
		 & = \sum_{\ell = 1}^{m} \abs[\big]{\tfrac{\pi_{\ell}^{\star}}{\widehat{\pi}_{\ell}} - 1} \sum_{i \in S_{\ell}} \pi_{i}           \\
		 & = \sum_{\ell = 1}^{m} \abs{\widehat{\pi}_{\ell} - \pi_{\ell}^{\star}}                                                          \\
		 & = \sum_{\ell = 1}^{m} \Delta_{\pi}^{(\ell)} \cdot \pi_{\ell}^{\star}                                                           \\
		 & \leq \Delta_{\pi},
	\end{align*}
	where the second equality follows from the fact that $\frac{\widehat{\pi}_{\ell}}{\pi_{\ell}^{\star}}$ is constant throughout $S_{\ell}$,
	the fourth equality follows from~\eqref{eq:grouped-weight},
	and the last two lines follow from~\cref{eq:intra-group-weight-mismatch,eq:worst-case-group-weight-mismatch}.
\end{proof}

Equipped with~\cref{lemma:distance-exact-group-weights}, we derive a uniform
bound between $\loss(\mu, \pi)$ and the ``reference'' $\loss(\mu, \bar{\pi})$ that also extends to their first two derivatives after
suitable rescaling.

\begin{proposition}[Loss perturbation]
	\label{proposition:loss-perturbation}
	Fix a compact \(U\subset \R^{dn} \) and suppose that \(\Delta_{\pi} \leq \tfrac{1}{2}\). Define
	\begin{equation}
		M_U:=\max_{\mu\in U}\max_{i\in[n]}\|\mu_i\|_2,
		\quad
		m_1:=\expec[X\sim p^\star]{\norm{X}}
		\;\; \text{and} \;\;
		m_2:=\expec[X \sim p^{\star}]{\norm{X}^2}.
	\end{equation}
	Let $c_{0} := \frac{1}{2} \pi_{\min}^{\star}$ and $C_{0} = 2 (1 + \tfrac{1}{\pi_{\min}^{\star}})$.
	There are constants $\const^{(0)}$, $\const^{(1)}$ and $\const^{(2)}$ such that
	\begin{subequations}
		\begin{align}
			\sup_{\mu\in U}\abs{\loss(\mu,\pi)-\loss(\mu,\bar\pi)}                            & \leq \const^{(0)}\Delta_\pi; \label{eq:loss-perturbation}    \\
			\sup_{\mu\in U}\norm{\grad_{\mu} \loss(\mu,\pi) - \grad_{\mu} \loss(\mu,\bar\pi)} & \leq \const^{(1)}\Delta_\pi; \label{eq:grad-perturbation}    \\
			\sup_{\mu\in U} \opnorm{\grad_{\mu\mu}^2 \loss(\mu,\pi) - \grad_{\mu\mu}^2 \loss(\mu,\bar\pi)}
			                                                                                  & \leq \const^{(2)}\Delta_\pi. \label{eq:hessian-perturbation}
		\end{align}
	\end{subequations}
	In particular, the constants are given by
	\begin{equation}
		\const^{(0)} = \frac{1}{c_0}, \; \const^{(1)} = C_{0}\sqrt{2n(m_2 + M_{U}^2)}, \;
    \const^{(2)} = C_{0} n (1 + 2 (2 + C_0 \Delta_{\pi}) (m_2 + M_{U}^2)).
		\label{eq:perturbation-constants}
	\end{equation}
\end{proposition}
\begin{proof}
	Since $\loss(\mu, \pi) = \expec[X \sim p^{\star}]{\log(p^{\star}(X) / p_{\mu, \pi}(X))}$, we
	control deviations by comparing the densities induced by $\pi$ and $\bar{\pi}$.
	For $\mu \in \R^{dn}$ and corresponding partition $[n] = \bigsqcup_{\ell} S_{\ell}$, let
	\begin{equation}
		q_\ell(x;\mu):=\frac{1}{\widehat\pi_\ell}\sum_{i\in S_\ell}\pi_i\phi(x\mid \mu_i).
		\;\; \text{for $\ell \in [m]$}
		\label{eq:group-partition}
	\end{equation}
	With~\eqref{eq:group-partition} at hand, we rewrite both densities as
	\begin{align}
		p_{\mu,\pi}(x)
		 & =
		\sum_{\ell=1}^m \widehat\pi_\ell q_\ell(x;\mu), \label{eq:mismatched-density-q} \\
		p_{\mu,\bar{\pi}}(x)
		 & =
		\sum_{\ell=1}^m \pi_\ell^\star q_\ell(x;\mu). \label{eq:matched-density-q}
	\end{align}
	Using the preceding decomposition and Jensen's inequality, we obtain
	\begin{equation}
		\abs{\loss(\mu, \pi) - \loss(\mu, \bar{\pi})} =
		\abs[\big]{\expec[X \sim p^{\star}]{\log\left(\tfrac{p_{\mu, \bar{\pi}}(X)}{p_{\mu, \pi}(X)}\right)}} \leq
		\expec[X \sim p^{\star}]{\abs*{\log\left(\tfrac{p_{\mu, \bar{\pi}}(X)}{p_{\mu, \pi}(X)}\right)}},
		\label{eq:log-diff}
	\end{equation}
	We now bound the expression inside the expectation in~\eqref{eq:log-diff}. For
	any finite $X$, \cref{lemma:log-lipschitz} yields

	\begin{equation}
		\abs{\log(\tfrac{p_{\mu, \bar{\pi}}(X)}{p_{\mu, \pi}(X)})}
		\leq \frac{\abs{p_{\mu, \pi}(X) - p_{\mu, \bar{\pi}}(X)}}{\min\set*{
				p_{\mu, \pi}(X), p_{\mu, \bar{\pi}}(X)
			}}
		\label{eq:log-density-ratio-bound}
	\end{equation}
    
	To further bound the right-hand side in~\eqref{eq:log-density-ratio-bound}, we proceed as follows:
	\begin{itemize}
		\item We lower bound both densities uniformly. Indeed, since $\Delta_{\pi} < \tfrac{1}{2}$ by assumption,
		      \begin{equation}
			      \widehat{\pi}_{\ell} \geq (1 - \Delta_{\pi}^{\ell}) \cdot \pi_{\ell}^{\star}
			      \geq \frac{1}{2} \pi_{\ell}^{\star}.
			      \label{eq:adjusted-density-bound}
		      \end{equation}
		      From the above and~\cref{eq:mismatched-density-q,eq:matched-density-q}, we deduce the lower bounds
		      \begin{align}
			      p_{\mu, \bar{\pi}}(X) & =
			      \sum_{\ell = 1}^{m} \pi_{\ell}^{\star} q_{\ell}(X; \mu) \geq
			      \pi_{\min}^{\star} \sum_{\ell=1}^{m} q_{\ell}(X; \mu); \label{eq:reference-density-lb-q} \\
			      p_{\mu, \pi}(X)       & =
			      \sum_{\ell = 1}^{m} \widehat{\pi}_{\ell} q_{\ell}(X; \mu) \geq
			      \frac{1}{2} \pi_{\min}^{\star} \sum_{\ell=1}^{m} q_{\ell}(X; \mu).
			      \label{eq:mismatched-density-lb-q}
		      \end{align}
		\item We upper bound their difference for a fixed $X$. Indeed,~\cref{lemma:distance-exact-group-weights} supplies the bound
		      \begin{align}
			      \abs{p_{\mu,\pi}(X)-p_{\mu,\bar\pi}(X)}
			       & =
			      \abs[\Big]{\sum_{\ell=1}^m (\widehat\pi_\ell-\pi_\ell^\star)q_\ell(X;\mu)} \notag \\
			       & \le
			      \sum_{\ell=1}^m \abs{\widehat{\pi}_\ell-\pi_{\ell}^\star} q_\ell(X;\mu)   \notag  \\
			       & \leq
			      \Delta_{\pi} \sum_{\ell=1}^m q_\ell(X;\mu).
			      \label{eq:density-difference-under-weight-mismatch}
		      \end{align}
	\end{itemize}
	Plugging the previous estimates into~\eqref{eq:log-density-ratio-bound} and taking
	expectations, noting $\norm{X} < \infty$ a.s., yields
	\[
		\expec[X \sim p^{\star}]{\abs{\log(p_{\mu, \pi}(X)) - \log(p_{\mu, \bar{\pi}}(X))}} \leq \frac{2 \Delta_{\pi}}{\pi_{\min}^{\star}}.
	\]
	Substituting into~\eqref{eq:log-diff} completes the proof of~\eqref{eq:loss-perturbation}.

	We now turn to the gradient bound. Fix $i \in S_{\ell}$; by the definition of $\bar{\pi}_i$ and~\eqref{eq:responsibilities}, we have
	\[
		\psi_i(x;\mu, \pi) = \frac{\pi_i\phi(x\mid \mu_i)}{p_{\mu,\pi}(x)},
		\quad
		\psi_i(x;\mu,\bar{\pi}) = \frac{\bar\pi_i\phi(x\mid \mu_i)}{p_{\mu,\bar\pi}(x)}
		=
		\frac{\pi_\ell^\star}{\widehat\pi_\ell} \cdot \frac{p_{\mu, \pi}(x)}{p_{\mu,\bar\pi}(x)} \cdot
		\psi_{i}(x; \mu, \pi).
	\]
	This implies the following bound on the difference between responsibilities:
	\begin{align*}
		\abs{\psi_{i}(X; \mu, \pi) - \psi_{i}(X; \mu, \bar{\pi})} & =
		\psi_{i}(X; \mu, \pi) \cdot \abs[\Big]{
			1 - \frac{\pi_{\ell}^{\star}}{\widehat{\pi}_{\ell}} \cdot \frac{p_{\mu, \pi}(X)}{p_{\mu, \bar{\pi}}(X)}
		}                                                                \\
		                                                          & =
		\psi_{i}(X; \mu, \pi) \cdot \abs[\Big]{
			\frac{\widehat{\pi}_{\ell} p_{\mu, \bar{\pi}}(X) - \pi_{\ell}^{\star} p_{\mu, \pi}(X)}{
				\widehat{\pi}_{\ell} p_{\mu, \bar{\pi}}(X)
			}
		}                                                                \\
		                                                          & \leq
		2\,\psi_{i}(X; \mu, \pi) \cdot \abs[\Big]{
			\frac{\widehat{\pi}_{\ell} p_{\mu, \bar{\pi}}(X) - \pi_{\ell}^{\star} p_{\mu, \pi}(X)}{
				{\pi}_{\ell}^{\star} p_{\mu, \bar{\pi}}(X)
			}
		}                                                                \\
		                                                          & \leq
		2\, \psi_{i}(X; \mu, \pi) \cdot \left(
		\Delta_{\pi}^{(\ell)} +
		\frac{\abs{p_{\mu, \pi}(X) - p_{\mu, \bar{\pi}}(X)}}{p_{\mu, \bar{\pi}}(X)}
		\right)                                                          \\
		                                                          & \leq
		2\, \psi_{i}(X; \mu, \pi) \cdot \left(
		\Delta_{\pi}^{(\ell)} + \frac{\Delta_{\pi}}{\pi_{\min}^{\star}}
		\right)                                                          \\
		                                                          & \leq
		2\,\Delta_{\pi} \left(1 + \frac{1}{\pi_{\min}^{\star}}\right)    \\
		                                                          & =
                                                              C_{0} \Delta_{\pi},
	\end{align*}
	where the first inequality follows from the bound~\eqref{eq:adjusted-density-bound},
	the second inequality follows from the definition of $\Delta_{\pi}^{(\ell)}$,
	the penultimate inequality follows from~\eqref{eq:density-difference-under-weight-mismatch}
	and~\eqref{eq:reference-density-lb-q},
	and the last inequality follows from $\Delta_{\pi}^{(\ell)} \leq \Delta_{\pi}$ and the fact
	that $\set{\psi_{i}(X; \mu, \pi)}_{i=1}^n \in \Delta^{n-1}$.

	By the gradient identity furnished by~\cref{lemma:mu-gradient}, we have
	\begin{align*}
		\norm{\grad_{\mu_i}\loss(\mu,\pi) - \grad_{\mu_{i}}\loss(\mu,\bar{\pi})}
		 & =
		\norm{\expec[X\sim p^\star]{(\psi_i(X;\mu, \pi) - \psi_{i}(X; \mu, \bar{\pi})) \cdot (X-\mu_i)}}     \\
		 & \leq
    C_{0} \Delta_{\pi} \cdot \Big(\expec[X \sim p^{\star}]{\norm{X - \mu_{i}}^2}\Big)^{1/2} \\
		 & \leq
     C_{0} \Delta_{\pi} \sqrt{2(m_{2} + M_{U}^2)},
	\end{align*}
	where the penultimate inequality follows from Cauchy-Schwarz and the fact that
	\begin{align*}
		\big(\expec[X \sim p^{\star}]{\norm{X - \mu_{i}}^2})^{1/2} & \leq
		\big(2 \expec[X \sim p^{\star}]{\norm{X}^2} + 2 \norm{\mu_{i}}^2\big)^{1/2}                                      \\
		                                                           & \leq 2^{1/2} \cdot \big(m_{2} + M_{U}^2\big)^{1/2}.
	\end{align*}
	Taking the Euclidean norm over the $n$ blocks of the gradient proves~\eqref{eq:grad-perturbation}.

	Finally, let \(H^\pi_{ij}(\mu)\) and \(H^{\bar\pi}_{ij}(\mu)\) denote the \((i,j)\)
	blocks of \(\grad^2_{\mu\mu}\loss(\mu,\pi)\) and \(\grad^2_{\mu\mu}\loss(\mu,\bar\pi)\).
	By the derivation in the proof of Proposition~\ref{prop:hessian-block-identified}, we obtain
	\[
		H^\pi_{ij}(\mu)
		=
		-\expec{\psi_i(X; \mu, \pi)(\delta_{ij}-\psi_j^\pi)
			(X-\mu_i)(X-\mu_j)^\T
		}
		+
		\delta_{ij}\expec{\psi_i(X; \mu, \pi)}I_d.
	\]
	The formula for \(H^{\bar\pi}_{ij}(\mu)\) follows \emph{mutatis-mutandis}. Writing
	$\psi_{i}^{\pi} := \psi_{i}(X; \mu, \pi)$ for brevity, we have
	\begin{align*}
		\abs{\psi_i^{\pi}(\delta_{ij}-\psi_j^\pi)-\psi_i^{\bar\pi}(\delta_{ij}-\psi_j^{\bar\pi})}
		 & \leq
		\delta_{ij} \abs{\psi_{i}^{\pi} - \psi_{i}^{\bar{\pi}}} + \abs{
			\psi_{i}^{\pi} \psi_{j}^{\pi} - \psi_{i}^{\bar{\pi}} \psi_{j}^{\bar{\pi}}
		}                                                                \\
		 & \leq
		\abs{\psi_{i}^{\pi} - \psi_{i}^{\bar{\pi}}} +
		\psi_j^{\pi} \abs{\psi_{i}^{\pi} - \psi_{i}^{\bar{\pi}}} +
		\psi_{i}^{\bar{\pi}} \abs{\psi_{j}^{\pi} - \psi_{j}^{\bar{\pi}}} \\
		 & \leq
		C_{0} \Delta_{\pi} \left(
		1 + \psi_{j}^{\bar{\pi}} + \psi_{i}^{\bar{\pi}} + \psi_{j}^{\bar{\pi}} - \psi_{j}^{\pi}
		\right)                                                          \\
		 & \leq
		C_{0} \Delta_{\pi} \left(2 + C_{0} \Delta_{\pi}\right) =: C_{1} \Delta_{\pi},
	\end{align*}
	using the fact that $\set{\psi_{i}^{\bar{\pi}}}_{i=1}^n \in \Delta^{n-1}$ in the last inequality.
	Therefore, we deduce
	\begin{align*}
		\opnorm{H^\pi_{ij}(\mu)-H^{\bar\pi}_{ij}(\mu)} & \leq
		C_{1} \Delta_{\pi}
		\expec{\norm{X-\mu_i}\norm{X-\mu_j}} + \delta_{ij} \expec{\abs{\psi_{i}^{\pi} - \psi_{i}^{\bar{\pi}}}}          \\
		                                               & \leq
		C_{1} \Delta_{\pi} (\expec{\norm{X - \mu_{i}}^2})^{1/2} (\expec{\norm{X - \mu_j}^2})^{1/2} + C_{0} \Delta_{\pi} \\
		                                               & \leq
		C_{0} \Delta_{\pi} + 2C_{1} \Delta_{\pi} \left(m_{2} + M_{U}^2\right) \\
                                                   &=
    C_0 \Delta_{\pi} \left(
      1 + 2(m_2 + M_{U}^2) (2 + C_{0} \Delta_{\pi})
    \right),
	\end{align*}
	again using Cauchy-Schwarz.
	Invoking~\cref{lemma:block-matrix-spectral-norm} supplies
	the bound for the full Hessian.
\end{proof}
\Cref{proposition:loss-perturbation} suggests that $\loss(\mu, \pi)$ tracks
the reference objective $\loss(\mu, \bar{\pi})$
up to additive error $\Delta_{\pi}$. However, its optimal value is an order of magnitude closer to $\loss(\theta^{\star})$.
We record this bound in the next Corollary.
\begin{corollary}[Optimal value bound of mismatched loss]
	\label{corollary:mismatched-local-floor}
	Let \(U\) be a compact neighborhood of \(\mu^\star\) contained in the local identified regime, and define
	\[
		\loss_{\pi}^{\dag} := \min_{\mu \in U} \loss(\mu,\pi).
	\]
	Then there exists a neighborhood $V \subset \intr{\Delta^{m-1}}$ of $\pi^{\star}$
	and a constant $\cfloor > 0$ satisfying
	\begin{equation}
		\label{eq:mismatched-local-floor}
		\loss_{\pi}^{\dag} - \loss(\theta^{\star}) \leq \cfloor \Delta_{\pi}^2, \;\;
		\text{for all $\pi \in \set{\pi \in \Delta^{n-1} \mid
					\widehat{\pi} \in V
				}$}.
	\end{equation}
\end{corollary}
\begin{proof}
  Recall the collapsed model loss from~\cref{lem:collapsed-hessian-pd}; for any $\nu \in \R^{dm}$ and $\rho \in \Delta^{m-1}$, let
	\[
		\loss_{\mathrm{coll}}(\nu, \rho) := \KL{p^\star}{\sum_{\ell=1}^m \rho_{\ell}\, \cN(\nu_{\ell}, I_{d})}.
	\]
	The function $\rho \mapsto \loss_{\coll}(\nu^{\star}, \rho)$ is minimized at
  $\rho = \pi^{\star} \in \intr{\Delta^{m-1}}$. As a result, we have
  \begin{equation}
		\left.\grad_{\rho}\loss_{\coll}(\nu^{\star}, \rho) \right|_{\rho = \pi^{\star}} = 0.
    \label{eq:collapsed-weight-gradient-vanishes}
  \end{equation}
	Now, let $\nu^{\star}: U \to \R^{dn}$ be the following map:
	\begin{align*}
		\nu^{\star}(\mu) = \bmx{
		\nu_{1}^{\star}(\mu) \\
		\vdots               \\
			\nu_{n}^{\star}(\mu)
		}, \quad
		[\nu^{\star}(\mu)]_{i} = \nu^{\star}_{\ell}, \;\; \text{for any $i \in S_{\ell}$}.
	\end{align*}
  For any $\pi$ such that $\norm{\widehat{\pi} - \pi^{\star}}$ is sufficiently small, we have that
	\begin{align}
		\loss_{\pi}^{\dag} - \loss(\theta^{\star}) & = \min_{\mu \in U} \loss(\mu, \pi) - \loss(\theta^{\star})                       \notag \\
		                                           & \leq
		\loss(\nu^{\star}(\mu), \pi) - \loss(\theta^{\star})                                                                           \notag \\
		                                           & =
                                               \KL{p^{\star}}{\sum_{\ell = 1}^{m} \sum_{i \in S_{\ell}} \pi_{i} \mathcal{N}(\nu_{\ell}^{\star}, I_d)} - \loss(\theta^{\star}) \label{eq:collapsed-loss-difference} \\
		                                           & =
                                               \loss_{\coll}(\nu^{\star}, \widehat{\pi}) - \loss(\theta^{\star})                                                      \notag \\
		                                           & =
                                               \underbrace{\loss_{\coll}(\nu^{\star}, \pi^{\star})}_{\loss(\theta^{\star})} + \ip{\grad_{\rho}\loss(\nu^{\star}, \rho)|_{\rho = \pi^{\star}}, \widehat{\pi} - \pi^{\star}} + O(\norm{\widehat{\pi} - \pi^{\star}}^2) - \loss(\theta^{\star}) \notag \\
                                               &\lesssim \norm{\widehat{\pi} - \pi^{\star}}_1^2 \notag \\
                                               &= \Delta_{\pi}^2, \label{eq:collapsed-loss-small-bound}
	\end{align}
  where the last inequality follows from~\eqref{eq:collapsed-weight-gradient-vanishes} and norm equivalence.
\end{proof}

\section{Local acceleration}
\label{sec:local-scheduled-acceleration}

In this section, we prove that the second phase of~\cref{alg:switching-gdpolyak} accelerates.
As we focus on a single epoch of \(\gdpolyak\), we drop all superscripts referencing the epoch index $k$ for simplicity.

Setting the stage, we introduce the notation necessary to carry out the argument. We write $\bar{\pi}$ for the reference weights
associated with $\pi$, as defined in~\cref{def:reference-exact-grouped-weights}, and
write $\mathcal{M}$ for the local $C^{\infty}$ ravine of the reference objective
\[
	\mu \mapsto \loss(\mu, \bar{\pi}),
\]
which is guaranteed to exist by Corollary~\ref{cor:exact-fixed-weight-ravine}. We write
\begin{equation}
	\bar{\mu}_{\ell} := \frac{1}{\widehat{\pi}_{\ell}} \sum_{i \in S_{\ell}} \pi_{i} \mu_{i}
	\quad \text{and} \quad
	\disp(\mu) := \sum_{\ell = 1}^{m} \sum_{i \in S_{\ell}} \pi_{i} \norm{\mu_{i} - \bar{\mu}_{\ell}}^2
	\label{eq:group-average-dispersion}
\end{equation}
for the cluster average and dispersion from Corollary~\ref{cor:grouped-bias-dispersion}. Unless specified otherwise, we write
\begin{equation}
	y := P_{\mathcal{M}}(\mu),
	\label{eq:ravine-projection}
\end{equation}
where $P_{\mathcal{M}}$ is guaranteed to be single-valued near $\mathcal{M}$. 

The forthcoming analysis shows that the iterates produced by applying $\gdpolyak$ to the weight-mismatched objective $\mu \mapsto \loss(\mu, \pi)$ approach
a tube around the ravine $\mathcal{M}$ at a linear rate; upon entering that tube, Polyak steps lead to a contraction up to
a prescribed floor level $O(\Delta_{\pi}^{2/3})$.
Along \(\mathcal M\), the tangential loss $\loss \circ P_{\cM}$ behaves like a quartic. Consequently, for \(y\in \mathcal M\), the tangential gradient has scale $\disp(y)^{3/2}$.
Specialized to our setting, the proof strategy of~\cite{davis2024gradient} motivates the tube condition
\[
	\dist_{\cM}(\mu) \lesssim \disp(y)^{3/2}.
\]
To formalize this, fix a proximity parameter \( \tau > 0 \) and threshold $\kappa > 0$.
We define the following sets, which correspond to different phases of local convergence:
\begin{subequations}
	\begin{align}
		\mathcal{T}_{\tau}        & := \set{\mu \in \R^{dn} \mid \dist_{\cM}(\mu) \leq \tau \disp(y)^{3/2}} \label{eq:ravine-tube} \\
		\mathcal{C}_{\kappa}      & := \set{\mu \in \R^{dn} \mid \disp(y) \leq \kappa \Delta_{\pi}^{2/3}} \label{eq:split-core}    \\
		\mathcal{A}_{\tau,\kappa} & :=\mathcal{T}_{\tau} \setD \mathcal{C}_{\kappa}. \label{eq:tangential-annulus}
	\end{align}
\end{subequations}
The local analysis then separates into two phases:
\begin{enumerate}[(i)]
	\item The \emph{normal phase}, where \( \dist_{\cM}(\mu) \gg \disp(y)^{3/2} \), so iterates lie outside the tube $\mathcal{T}_{\tau}$.
	\item The \emph{tangential phase}, where \( \disp(y) \gtrsim \max\set{\dist_{\cM}^{2/3}(\mu), \kappa \Delta_{\pi}^{2/3}} \), so iterates lie in $\mathcal{A}_{\tau, \kappa}$.
\end{enumerate}
During the normal phase, short gradient descent steps approach the tube $\mathcal{T}_{\tau}$, whereupon Polyak steps
successfully reduce the distance to $\mu^{\star}$ along the manifold by a constant factor. The threshold $\Delta_{\pi}$ saturates other
quantities inside $\mathcal{C}_{\kappa}$.

\medskip

Before we proceed to the main proof we establish a few auxiliary results, specific to the local acceleration phase, that are frequently needed in our analysis. We will make repeated use of the fact that $\disp(\mu) \asymp \disp(P_{\mathcal{M}}(\mu))$ when $\mu \in \mathcal{T}_{\tau}$:
\begin{lemma}
	\label{lemma:disp-comparison}
	There exists a neighborhood $U$ of $\mu^{\star}$ and constants $c_1, c_2 > 0$ such that
	\begin{equation}
		\label{eq:disp-y-disp-mu-comparison}
		c_1 \disp(y) \leq \disp(\mu) \leq c_2 \disp(y), \;\; \text{for all $\mu \in U \cap \mathcal{T}_{\tau}$.}
	\end{equation}
\end{lemma}
\begin{proof}
	The function $\disp(\cdot)$ is $C^{\infty}$, thus Lipschitz on any compact set. It follows that
	\[
		\abs{\disp(y)-\disp(\mu)}
		=
		\abs{\disp(P_{\mathcal{M}}(\mu)) - \disp(\mu)}
		\leq C \cdot \norm{P_{\mathcal{M}}(\mu) - \mu}
		=
		C \cdot \dist_{\mathcal{M}}(\mu).
	\]
	Therefore, from the inclusion $\mu \in \mathcal{T}_{\tau}$, it follows that
	\begin{align*}
		\abs{\disp(y)-\disp(\mu)}
		\leq C\tau \disp(y)^{3/2}
		\Rightarrow
		\bigl(1 - C \tau \sqrt{\disp(y)}\bigr) \disp(y)
		\leq \disp(\mu)
		\leq \disp(y) \bigl(1 + C \tau \sqrt{\disp(y)}\bigr).
	\end{align*}
	By shrinking $U$ if necessary, we have $\sqrt{\disp(y)} \leq \tfrac{1}{2 C \tau}$,
	proving the claim for $c_1 = \tfrac{1}{2}$, $c_2 = \tfrac{3}{2}$.
\end{proof}
At the same time, we will need the following comparison that allows us to convert between the distance from $P_{\mathcal{M}}(\mu)$ and the distance from $\mu$ itself to $\mu^{\star}$.

\begin{lemma}
    \label{lemma:distance-conversion}
    There is a neighborhood $U$ of $\mu^{\star}$ and a constant $C_{U} > 0$ such that
    \begin{equation}
        \norm{P_{\mathcal{M}}(\mu) - \mu^{\star}} \leq
        \norm{\mu - \mu^{\star}} + C_{U} \norm{\mu - \mu^{\star}}^2,
        \quad \text{for all $\mu \in U$.}
        \label{eq:distance-conversion}
    \end{equation}
\end{lemma}
\begin{proof}
    For $\mu$ near $\mu^{\star} \in \mathcal{M}$, the projection $P_{\mathcal{M}}$ is $C^{1}$. In particular, we have that
	\begin{align}
		y - \mu^{\star} & = P_{\mathcal{M}}(\mu) - \mu^{\star}                                                                                                 \notag \\
		                & = P_{\mathcal{M}}(\mu^{\star}) - \mu^{\star} + \grad P_{\mathcal{M}}(\mu^{\star})(\mu - \mu^{\star}) + O(\norm{\mu - \mu^{\star}}^2) \notag \\
		                & = P_{\mathcal{T}_{\mu^{\star}}\mathcal{M}}(\mu - \mu^{\star}) + O(\norm{\mu - \mu^{\star}}^2),
                        \label{eq:distance-along-ravine-taylor-expansion}
	\end{align}
	where the second equality follows by smoothness of the projection and the third equality follows
	from the fact that $\grad P_{\mathcal{M}}(\bar{\mu}) = P_{\mathcal{T}_{\bar{\mu}} \mathcal{M}}$ for any
	$\bar{\mu} \in \mathcal{M}$. Extracting the constant from the remainder
    term and labeling at $C_{U}$ completes the proof.
\end{proof}

\subsection{Trajectory analysis}

The reference objective \(\mu \mapsto \loss(\mu,\bar{\pi})\) is compatible with the theory developed in~\cref{app:identified-ravine};
Theorem~\ref{thm:exact-fixed-weight-fourth-order} supplies local fourth-order growth,
Corollary~\ref{cor:exact-fixed-weight-ravine} verifies the hypotheses used in~\cite{davis2024gradient}, and~\cref{thm:ddj-ravine-step-input}
provides a one-step contraction along the ravine for Polyak steps. We record this conclusion in the next Proposition:

\begin{proposition}[Polyak step contraction for reference loss]
	\label{proposition:polyak-step-contraction-reference-loss}
	Let $P_{\mathcal{M}}$ denote the local projection onto $\mathcal{M}$ on a neighborhood $U$ of $\mu^{\star}$ and write
	\begin{equation}
		\loss(\mu, \bar{\pi}) - \loss(\theta^{\star}) =
		\underbrace{\loss(P_{\mathcal{M}}(\mu), \bar{\pi})}_{\loss_{T}(\mu)} +
		\underbrace{\loss(\mu, \bar{\pi}) - \loss(P_{\mathcal{M}}(\mu), \bar{\pi})}_{\loss_{N}(\mu)}.
		\label{eq:ref-loss-decomposition}
	\end{equation}
	Moreover, define the reference Polyak update and ``shadow'' iterates
	\begin{align}
		\mu^{+} & := \mu - \frac{\loss(\mu, \bar{\pi})}{\norm{\grad_{\mu} \loss(\mu, \bar{\pi})}} \cdot
		\frac{\grad_{\mu} \loss(\mu, \bar{\pi})}{\norm{\grad_{\mu} \loss(\mu, \bar{\pi})}},
		\quad y := P_{\mathcal{M}}(\mu), \;\; y^{+} := P_{\mathcal{M}}(\mu^{+}). \label{eq:ref-polyak-update}
	\end{align}
	By shrinking the set $U$ if necessary, there exist numbers $q \in (0, 1)$ and $C > 0$ such that
	\begin{equation}
		\norm{\grad \loss_{N}(\mu)}
		\leq \frac{1}{100}\norm{\grad \loss_{T}(y)} \implies
		\left\{ \;
		\begin{aligned}
			\norm{y^{+} - \mu^{\star}} & \leq q \norm{y - \mu^{\star}}, \;\; \text{and} \\
			\dist_{\mathcal{M}}(\mu^+) & \leq C \norm{y - \mu^{\star}}.
		\end{aligned}
		\right.
		\label{eq:small-normal-grad-implies-contraction}
	\end{equation}
\end{proposition}
\begin{proof}
	The objective $\loss(\mu, \bar{\pi})$ admits a local solution set $\set{\mu^{\star}}$, which is a singleton.
	The conclusion follows immediately by applying~\cref{thm:ddj-ravine-step-input} with
	\( f^{\star} = \loss(\theta^{\star}) = 0 \).
\end{proof}

Naturally, \Cref{alg:switching-gdpolyak} has no access to the
reference weights $\bar{\pi}$. Nevertheless, we will argue that the gradient used in~\cref{alg:switching-gdpolyak} is not far from
the gradient of the ``reference'' loss. The first ingredient is a collection of results showing that the
tangent part of the loss behaves almost exactly like a quartic function inside $\mathcal{T}_{\tau}$.
In what follows, we always use $\loss_{T}$ and $\loss_{N}$ to refer to the tangent and normal parts of the
\emph{reference} loss, as defined in~\cref{proposition:polyak-step-contraction-reference-loss}.

\begin{lemma}[Tangent part scaling and gradient perturbation]
	\label{lem:relative-gradient-annulus}
	For any sufficiently small \(\tau>0\), there exists a neighborhood \(V \subset U \) of $\mu^{\star}$ and constants
	$C \geq c > 0$ such that:
	\begin{enumerate}[(i)]
		\item \label{item:reference-ravine-tangent-growth} For every \( y \in V \cap \mathcal{M} \), we have that
		      \begin{subequations}
			      \begin{align}
				      c \cdot \norm{y - \mu^{\star}}^2 & \leq \disp(y) \leq C \cdot \norm{y - \mu^{\star}}^2 \label{eq:dispersion-approx-distance-squared}                 \\
				      c \cdot \disp(y)^2               & \leq \loss_{T}(y) \leq C \cdot \disp(y)^2 \label{eq:tangent-loss-approx-squared-dispersion}                       \\
				      c \cdot \disp(y)^{\tfrac{3}{2}}  & \leq \norm{\grad \loss_{T}(y)} \leq C \cdot \disp(y)^{\tfrac{3}{2}} \label{eq:tangent-grad-approx-dispersion-pow}
			      \end{align}
		      \end{subequations}
		\item \label{item:reference-gradient-pl-like} For every \(\mu \in V \cap \mathcal{T}_{\tau} \), we have that
		      \begin{subequations}
			      \begin{align}
				      \norm{\grad_{\mu} \loss(\mu, \bar{\pi})} & \geq c \cdot \disp(y)^{\tfrac{3}{2}}
				      \label{eq:reference-gradient-pl-like-a}                                         \\
				      \norm{\grad \loss_{N}(\mu)}              & \leq C \tau \disp(y)^{\tfrac{3}{2}}.
				      \label{eq:reference-gradient-pl-like-b}
			      \end{align}
		      \end{subequations}
		\item \label{item:improved-gradient-perturbation-in-annulus} For every \(\mu \in V \cap \mathcal{A}_{\tau,\kappa}\), we have that
		      \begin{equation}
			      \norm{\grad_{\mu} \loss(\mu, \pi) - \grad_{\mu} \loss(\mu, \bar{\pi})}
			      \leq C \kappa^{-\tfrac{3}{2}} \norm{\grad_{\mu} \loss(\mu, \bar{\pi})}.
			      \label{eq:improved-gradient-perturbation-in-annulus}
		      \end{equation}
	\end{enumerate}
\end{lemma}

\begin{proof}
	We first identify the tangent space of the reference loss ravine at \(\mu^\star\). Since
	\[
		\bar{\pi}_{i} = \pi_{i} \cdot \frac{\pi_{\ell}^{\star}}{\widehat{\pi}_{\ell}}, \;\; \text{for all $i \in S_{\ell}$},
	\]
	the following linear subspaces are equal:
	\begin{align}
		\mathcal T
		 & := \set{u \in \R^{dn} \mid \sum_{i \in S_{\ell}} \pi_{i} u_{i} = 0, \; \text{for all $\ell \in [m]$}}       \\
		 & = \set{u \in \R^{dn} \mid \sum_{i \in S_{\ell}} \bar{\pi}_{i} u_{i} = 0, \; \text{for all $\ell \in [m]$}}.
		\label{eq:equal-tangent-space}
	\end{align}
	By~\cref{thm:exact-grouped-ravine-geometry,cor:exact-fixed-weight-ravine} applied to
	$\mu \mapsto \loss(\mu, \bar{\pi})$, $\mathcal{T}$ is the tangent space of $\mathcal{M}$ at $\mu^{\star}$. Therefore, by~\cref{fact:local-C2-expansion},
	we deduce that $y \in \mathcal{M}$ near $\mu^{\star}$ can be written as
	\[
		y = \mu^{\star} + u + O(\norm{y - \mu^{\star}}^2), \;\; \text{where $u \in \mathcal{T}_{\mu^{\star}} \mathcal{M}$}.
	\]
	In particular, the dispersion $\disp(y)$ does not change up to first-order: indeed,
	\begin{align*}
		\disp(\mu^{\star} + u) & =
		\sum_{\ell = 1}^{m} \sum_{i \in S_{\ell}} \pi_{i} \norm[\Big]{
			\mu^{\star}_{i} + u_{i} - \frac{1}{\widehat{\pi}_{\ell}} \sum_{j \in S_{\ell}} \pi_{j} (\mu^{\star}_j + u_{j})
		}^2 =
		\sum_{\ell = 1}^{m} \sum_{i \in S_{\ell}} \pi_{i} \norm{u_{i}}^2,
	\end{align*}
	where the second equality follows from~\eqref{eq:equal-tangent-space}. Consequently,
	\begin{equation}
		\disp(y) = \sum_{i = 1}^{n} \pi_{i} \norm{u_{i}}^2 + O(\norm{u}^3).
		\label{eq:dispersion-near-mustar}
	\end{equation}
	Therefore, after shrinking the local neighborhood if necessary, we have
	\[
		c \cdot \norm{u}^2 \leq \disp(y) \leq C \cdot \norm{u}^2.
	\]
	Finally, since \(\norm{y - \mu^\star} \asymp \norm{u}\), this proves~\eqref{eq:dispersion-approx-distance-squared}.

	We now apply Item~\ref{item:tangent-part-p-growth} from~\cref{thm:ddj-ravine-step-input} with $p = 4$,
	which yields
	\begin{align*}
		c \cdot \norm{y - \mu^{\star}}^4 & \leq \loss_{T}(y) \leq C \cdot \norm{y - \mu^{\star}}^4,              \\
		c \cdot \norm{y - \mu^{\star}}^3 & \leq \norm{\grad \loss_{T}(y)} \leq C \cdot \norm{y - \mu^{\star}}^3,
	\end{align*}
	for all $y \in V \cap \mathcal{M}$. Since $\disp(y) \asymp \norm{y - \mu^{\star}}^2$ by~\eqref{eq:dispersion-approx-distance-squared},
	this proves~\cref{eq:tangent-loss-approx-squared-dispersion,eq:tangent-grad-approx-dispersion-pow}.

	\medskip

	We now turn to the proof of Item~\eqref{item:reference-gradient-pl-like}. In particular, we argue that the gradient of the tangent part
	dominates near $\mu^{\star}$ via the decomposition furnished by~\cref{proposition:polyak-step-contraction-reference-loss}. To that end, note that $\loss_{N}$ is $C^{1}$ and vanishes along
	$\mathcal{M}$. Therefore, its gradient is Lipschitz near $\mathcal{M}$ and satisfies
	\begin{equation}
		\norm{\grad \loss_{N}(\mu)} \leq C \cdot \dist_{\mathcal{M}}(\mu) \leq C \tau \disp(y)^{3/2},
		\label{eq:normal-grad-ub}
	\end{equation}
	where the second inequality follows from the inclusion $\mu \in \mathcal{T}_{\tau}$; this proves~\eqref{eq:reference-gradient-pl-like-b}.

	On the other hand,~\cref{eq:tangent-grad-approx-dispersion-pow} suggests that
	\begin{equation}
		\norm{\grad \loss_{T}(y)} \geq c \cdot \disp(y)^{3/2}.
		\label{eq:tangent-grad-lb}
	\end{equation}
	Finally, using the tangent-normal decomposition from~\cref{proposition:polyak-step-contraction-reference-loss}, we obtain
	\begin{align*}
		\norm{\grad \loss(\mu, \bar{\pi})} & \geq
		\norm{\grad \loss_{T}(\mu)} - \norm{\grad \loss_{N}(\mu)}                                   \\
		                                   & \geq
		\norm{\grad \loss_{T}(y)}(1 - C \dist_{\mathcal{M}}(\mu)) - C \tau \dist_{\mathcal{M}}(\mu) \\
		                                   & \gtrsim
		\norm{\grad \loss_{T}(y)} - C' \tau \norm{\grad \loss_{T}(y)}                               \\
		                                   & \gtrsim \norm{\grad \loss_{T}(y)},
	\end{align*}
	where the first step follows from the reverse triangle inequality, the second
	inequality follows from~\cite[Lemma 5.6]{davis2024gradient} and~\eqref{eq:normal-grad-ub},
	and the last two inequalities follow from~\eqref{eq:tangent-grad-lb} and shrinking $\tau$ if necessary.
	This proves~\eqref{eq:reference-gradient-pl-like-a} and Item~\eqref{item:reference-gradient-pl-like}.

	\medskip

	Finally, we turn to the proof of Item~\eqref{item:improved-gradient-perturbation-in-annulus}. Indeed, we have
	\begin{align*}
		\norm{\grad_{\mu} \loss(\mu,\pi)-\grad_{\mu} \loss(\mu,\bar{\pi})} & \leq \const^{(1)} \Delta_{\pi}                                                               \\
		                                                                   & \leq \const^{(1)} \kappa^{-{3}/{2}} \disp(y)^{3/2}                                           \\
		                                                                   & \leq \frac{\const^{(1)}}{c \cdot \kappa^{{3}/{2}}} \norm{\grad_{\mu} \loss(\mu, \bar{\pi})}.
	\end{align*}
	Here the first inequality follows from~\eqref{eq:grad-perturbation}, the second inequality
	follows from the inclusion $\mu \in \mathcal{A}_{\tau, \kappa}$ and the last inequality follows from Item~\eqref{item:reference-gradient-pl-like}.
	This proves Item~\eqref{item:improved-gradient-perturbation-in-annulus}.
\end{proof}

\Cref{lem:relative-gradient-annulus} (in particular, its last item) allows us to relate
the denominators of the Polyak step applied to the mismatched and reference losses when $\mu$ lies inside the annulus $\mathcal{A}_{\tau, \kappa}$.
To relate the numerators, we use a sharper perturbation bound than the estimate from~\cref{proposition:loss-perturbation}
supplied by the next Lemma.

\begin{lemma}[{Perturbed loss on $\mathcal{A}_{\tau, \kappa}$}]
	\label{lemma:relative-value-annulus}
	There exist a constant $C > 0$ and a neighborhood $V$ around $\mu^{\star}$ such
	that for all $\mu \in V$, the following holds:
	\begin{align}
		\label{eq:relative-value-absolute}
		\abs{\loss(\mu, \pi) - \loss(\mu, \bar{\pi})} \leq C \Delta_{\pi} \left(
		\Delta_{\pi} + \disp(y) + \dist_{\mathcal{M}}(\mu)
		\right).
	\end{align}
	Consequently, when \(\mu \in \mathcal A_{\tau,\kappa}\), we have the improved bound:
	\begin{align}
		\label{eq:relative-value-relative}
		\abs{\loss(\mu,\pi)-\loss(\mu,\bar\pi)} \leq
		C \frac{\left(\loss(\mu, \bar{\pi}) - \loss(\theta^{\star}) \right)^{5/4}}{m \pi_{\min} \kappa^{3/2}}.
	\end{align}
	In particular, for fixed \(\tau\) and \(\kappa\), and after shrinking \(V\) if necessary,
	\begin{align}
		\label{eq:relative-value-relative-clean}
		\abs{\loss(\mu,\pi)-\loss(\mu,\bar\pi)} \leq
		\frac{C}{\kappa} \cdot \left(\loss(\mu, \bar{\pi}) - \loss(\theta^{\star})\right)
	\end{align}
	for every \(\mu\in V \cap \mathcal A_{\tau,\kappa}\), provided \(\disp(y)\) is sufficiently small.
\end{lemma}

\begin{proof}
	Define $h(\mu) := \loss(\mu, \pi) - \loss(\mu, \bar{\pi})$; at $\mu^{\star}$, $h(\mu) = \loss(\mu^{\star}, \pi)$ since
	the second loss term vanishes. We now analyze the growth of $h$ near $\mu^{\star}$. For
	any $y \in \mathcal{M}$ near $\mu^{\star}$, \cref{fact:local-C2-expansion} yields
	\begin{equation}
		y = \mu^{\star} + u + O(\norm{y - \mu^{\star}}^2), \;\; \text{where $u \in \mathcal{T}_{\mu^{\star}}\mathcal{M}$.}
		\label{eq:local-C2-expansion-again}
	\end{equation}
	Using this estimate in a second-order Taylor expansion of $h(\mu)$, we deduce that
	\begin{align}
		h(y) - h(\mu^{\star}) & = \ip{\grad h(\mu^{\star}), u} + \frac{1}{2} \ip{u, \grad^2 h(\mu^{\star}), u} + O(\Delta_{\pi} \norm{y - \mu^{\star}}^2), \label{eq:function-gap-second-order}
	\end{align}
	where the last term in~\eqref{eq:function-gap-second-order} follows from the following claim:
	\begin{claim}
		Fix $y \in \mathcal{M}$ with the expansion~\eqref{eq:local-C2-expansion-again}. Then
		\begin{equation}
			\ip{\grad h(\mu^{\star}), y - \mu^{\star} - u} = O(\Delta_{\pi} \norm{y - \mu^{\star}}^2)
		\end{equation}
	\end{claim}
	\begin{proof}
		Clearly, $\ip{\grad h(\mu^{\star}), y - \mu^{\star} - u} = O(\norm{\grad h(\mu^{\star})} \norm{y - \mu^{\star}}^2)$.
		Moreover, we have
		\begin{align*}
			\norm{\grad h(\mu^{\star})} & =
			\norm{\grad \loss(\mu^{\star}, \pi) - \grad \loss(\mu^{\star}, \bar{\pi})} \\
			                            & \leq \const^{(1)} \Delta_{\pi},
		\end{align*}
		where the inequality follows from~\eqref{eq:grad-perturbation}. This completes the proof of the Claim.
	\end{proof}

	We now simplify the terms on the RHS of~\eqref{eq:function-gap-second-order}. We focus on the gradient first.

	\begin{claim} \label{claim:local-C2-expansion-gradient-vanishes}
		For any $u \in \mathcal{T}_{\mu^{\star}}\mathcal{M}$, $\ip{\grad h(\mu^{\star}), u} = 0$.
	\end{claim}
	\begin{proof}
		From~\cref{cor:exact-fixed-weight-ravine} it follows that
		$\loss(\mu^{\star}, \bar{\pi}) = 0$ with $\mu^{\star}$ an isolated minimizer.
		Therefore, $\grad h(\mu^{\star}) = \grad \loss(\mu^{\star}, \pi)$. By~\cref{lemma:mu-gradient}, we have
		\begin{align*}
			[\grad \loss(\mu^{\star}, \pi)]_{i} & =
			\expec[X \sim p^{\star}]{\psi_{i}(X; \mu^{\star}, \pi)(\mu_{i}^{\star} - X)},                    \\
			\psi_{i}(X; \mu^{\star}, \pi)       & =
			\frac{\pi_{i} \phi(X \mid \mu^{\star}_{i})}{\sum_{q=1}^{n} \pi_{q} \phi(X \mid \mu_{q}^{\star})} \\
			                                    & =
			\frac{\pi_{i} \phi(X \mid \mu^{\star}_{\ell})}{
				\sum_{j = 1}^{m} \widehat{\pi}_{j} \phi(X \mid \mu_{j}^{\star})
			}                                                                                                \\
			                                    & =
			\frac{\pi_{i}}{\widehat{\pi}_{\ell}} \cdot \frac{\widehat{\pi}_{\ell} \phi(X \mid \mu^{\star}_{\ell})}{
				\sum_{j = 1}^{m} \widehat{\pi}_{j} \phi(X \mid \mu^{\star}_{j})
			}                                                                                                \\
			                                    & =
			\pi_{i} \cdot \frac{\gamma_{\ell}(X; \mu^{\star})}{\widehat{\pi}_{\ell}},
		\end{align*}
		writing $\gamma_{\ell}$ for the $\ell^{\text{th}}$ responsibility of the collapsed model where
		all students in $S_{\ell}$ have means equal to $\mu^{\star}_{\ell}$. The preceding display shows
		$\psi_{i}(X; \mu^{\star}, \pi) = \frac{\pi_i}{\pi_j} \psi_{j}(X; \mu^{\star}, \pi)$ whenever $i, j \in S_{\ell}$.
		Expanding the expression $\ip{\grad \loss(\mu^{\star}, \pi), u}$, we obtain
		\begin{align*}
			\ip{\grad \loss(\mu^{\star}, \pi), u} & =
			\expec[X \sim p^{\star}]{
				\sum_{i = 1}^{n} \psi_{i}(X; \mu^{\star}, \pi)\ip{\mu_{i}^{\star} - X, u_i}
			}                                         \\
			                                      & =
			\expec[X \sim p^{\star}]{
				\sum_{\ell = 1}^{m} \ip[\Big]{\mu^{\star}_{\ell} - X, \sum_{i \in S_{\ell}} \psi_{i}(X; \mu^{\star}, \pi) u_i}
			}                                         \\
			                                      & =
			\expec[X \sim p^{\star}]{
				\sum_{\ell = 1}^{m} \ip[\Big]{\mu^{\star}_{\ell} - X, \frac{\gamma_{\ell}(X; \mu^{\star})}{\widehat{\pi}_{\ell}} \sum_{i \in S_{\ell}} \pi_{i} u_i}
			}
		\end{align*}
		Finally, by an argument identical to the one used in the proof of~\cref{lem:relative-gradient-annulus}, we
		deduce that $\sum_{i \in S_{\ell}} \pi_{i} u_{i} = 0$, since $u \in \mathcal{T}_{\mu^{\star}}\mathcal{M}$.
		This completes the proof.
	\end{proof}

	With~\cref{claim:local-C2-expansion-gradient-vanishes} at hand, we turn to the quadratic
	form induced by the Hessian. We have
	\begin{align*}
		\abs{\ip{u, \grad^2 h(\mu^{\star}) u}} & =
		\abs{\ip{u, (\grad^2 \loss(\mu^{\star}, \pi) - \grad^2 \loss(\mu^{\star}, \bar{\pi})) u}} \\
		                                       & \leq
		\norm{u}^2 \cdot \const^{(2)} \Delta_{\pi}                                                \\
		                                       & \lesssim
		\norm{y - \mu^{\star}}^2 \const^{(2)} \Delta_{\pi}                                        \\
		                                       & \lesssim
		\disp(y) \cdot \Delta_{\pi},
	\end{align*}
	where the penultimate inequality follows from the tangent-normal decomposition of $y$ and
	the last inequality follows from~\eqref{eq:dispersion-approx-distance-squared}.
	We conclude that, for $y \in \mathcal{M} \cap V$, where $V$ is defined in~\cref{lem:relative-gradient-annulus},
	\begin{equation}
		h(y) - h(\mu^{\star}) = O(\Delta_{\pi} \cdot \disp(y))
		\label{eq:grad-diff-on-manifold}
	\end{equation}
	Now, suppose that $\mu \notin \mathcal{M}$ with $y = P_{\mathcal{M}}(\mu)$. By the mean value theorem, we have for $t\in[0,1]$
	\begin{align*}
		\abs{h(\mu) - h(y)} & = \abs{\ip{\grad h(y + t(\mu - y)), \mu - y}}                      \\
		                    & \leq \norm{\grad h(y + t(\mu - y))} \cdot \dist_{\mathcal{M}}(\mu) \\
		                    & \leq \const^{(1)} \Delta_{\pi} \cdot \dist_{\mathcal{M}}(\mu),
	\end{align*}
	where the last inequality follows from the estimate~\eqref{eq:grad-perturbation}. From this and~\eqref{eq:grad-diff-on-manifold}, we deduce
	\begin{equation}
		\abs{h(\mu) - h(\mu^{\star})}
		\leq C \Delta_{\pi}
		\left(\dist_{\mathcal{M}}(\mu) + \disp(y)\right).
		\label{eq:h-diff}
	\end{equation}
	Finally, the derivation
	used between~\eqref{eq:collapsed-loss-difference} and~\eqref{eq:collapsed-loss-small-bound} in~\cref{corollary:mismatched-local-floor}
	supplies the bound
	\[
		h(\mu^{\star}) = \loss(\mu^{\star}, \pi) - \loss(\theta^{\star}) \lesssim \Delta_{\pi}^2,
	\]
	from which~\eqref{eq:relative-value-absolute} immediately follows:
	\[
		\abs{h(\mu)}
		\leq C\Delta_{\pi}
		\left(\Delta_{\pi}+\disp(y)+\dist_{\mathcal M}(\mu)\right).
	\]

	To prove~\eqref{eq:relative-value-relative}, recall \(\mu \in \mathcal{A}_{\tau,\kappa}\) yields
	\[
		\disp(y) \geq \kappa \Delta_{\pi}^{2/3}, \;\; \text{and} \;\;
		\dist_{\mathcal{M}}(\mu) \leq \tau \disp(y)^{3/2}.
	\]
	Plugging these into the right-hand side of~\eqref{eq:relative-value-absolute} leads to
	\begin{align*}
		\abs{h(\mu)}
		 & \lesssim
		\kappa^{-{3}/{2}}\disp(y)^{{5}/{2}}
		\left(
		1 + \kappa^{-{3}/{2}}\disp(y)^{{1}/{2}}
		+ \tau \disp(y)^{{1}/{2}}
		\right)     \\
		 & \leq
		C \cdot \kappa^{-3/2} \disp(y)^{5/2},
	\end{align*}
	assuming $\disp(y)$ is sufficiently small. From proximity to $\mu^{\star}$, the growth bound from~\cref{thm:mo-identifiability-input} and the bias-dispersion
	decomposition from~\cref{cor:grouped-bias-dispersion}, it follows that
	\[
		\disp(y)^2 \lesssim \frac{1}{m \pi_{\min}} \left( \loss(\mu, \bar{\pi}) - \loss(\theta^{\star}) \right) \implies
		\abs{h(\mu)} \lesssim C \cdot \frac{(\loss(\mu, \bar{\pi}) - \loss(\theta^{\star}))^{5/4}}{\kappa^{3/2}}.
	\]
	This proves \eqref{eq:relative-value-relative}, from which~\eqref{eq:relative-value-relative-clean} easily follows.
\end{proof}

The previous two lemmas supply the relative denominator and numerator estimates needed to compare a Polyak
step on the weight-mismatched objective with a Polyak step on the reference objective inside the annulus $\mathcal{A}_{\tau, \kappa}$.
We record that comparison below.

\begin{lemma}[Polyak step comparison on $\mathcal{A}_{\tau, \kappa}$] \label{lemma:polyak-step-comparison} Define the following two quantities:
	\begin{subequations}
		\begin{align}
			\mu^{+}       & := \mu - \frac{\loss(\mu, \pi)}{\norm{\grad \loss(\mu, \pi)}^2} \cdot \grad \loss(\mu, \pi) \label{eq:polyak-actual}                       \\
			\bar{\mu}^{+} & := \mu - \frac{\loss(\mu, \bar{\pi})}{\norm{\grad \loss(\mu, \bar{\pi})}^2} \cdot \grad \loss(\mu, \bar{\pi}). \label{eq:polyak-reference}
		\end{align}
	\end{subequations}
	For any sufficiently small \( \tau > 0 \), and after shrinking \(V\) if necessary,
	there exists a constant \(C > 0\) and a threshold \(\kappa_{0} \geq 1\) such that for any $\kappa \geq \kappa_{0}$
	and $\mu \in V \cap \mathcal{A}_{\tau, \kappa}$, the following hold:
	\begin{align}
		\norm{\mu^{+} - \bar{\mu}^{+}} & \leq \frac{C}{\kappa^{3/2}} \cdot \frac{\loss(\mu, \bar{\pi})}{\norm{\grad \loss(\mu, \bar{\pi})}}
		\leq \frac{C}{\kappa} \norm{P_{\mathcal{M}}(\mu) - \mu^{\star}}. \label{eq:relative-polyak-comparison-distance}
	\end{align}
\end{lemma}
\begin{proof}
	Fix $\mu \in V \cap \mathcal{A}_{\tau, \kappa}$ and denote $h(\mu) := \loss(\mu, \pi) - \loss(\mu, \bar{\pi})$. We have
	\begin{subequations}
		\begin{align}
			\abs{h(\mu)}        & \leq \frac{C}{\kappa^{3/2}} \left(\loss(\mu, \bar{\pi}) - \loss(\theta^{\star})\right)^{5/4}, \label{eq:polyak-fval-diff} \\
			\norm{\grad h(\mu)} & \leq \frac{C}{\kappa^{3/2}} \norm{\grad \loss(\mu, \bar{\pi})}, \label{eq:polyak-grad-diff}
		\end{align}
	\end{subequations}
	with~\eqref{eq:polyak-fval-diff} following from~\eqref{eq:relative-value-relative} and~\eqref{eq:polyak-grad-diff} following from
	Item~(\ref{item:improved-gradient-perturbation-in-annulus}) of~\cref{lem:relative-gradient-annulus}.

	We now compare the Polyak steps themselves. Since $\loss(\theta^{\star}) = 0$, we have
	\begin{align}
		\norm{\mu^{+} - \bar{\mu}^{+}} & =
		\norm[\Big]{
			\underbrace{\frac{\loss(\mu, \pi)}{\norm{\grad \loss(\mu, \pi)}^2}}_{\zeta_{\pi}} \cdot \grad \loss(\mu, \pi) -
			\underbrace{\frac{\loss(\mu, \bar{\pi})}{\norm{\grad \loss(\mu, \bar{\pi})}^2}}_{\zeta_{\bar{\pi}}} \cdot \grad \loss(\mu, \bar{\pi})
		}                                     \notag \\
		                               & \leq
		\abs{\zeta_{\pi} - \zeta_{\bar{\pi}}} \norm{\grad \loss(\mu, \bar{\pi})}
		+ \zeta_{\pi} \cdot \norm{\grad h(\mu)}.
		\label{eq:polyak-step-difference}
	\end{align}
	We handle each term in~\eqref{eq:polyak-step-difference} separately. First, we argue that
	\begin{align}
		\zeta_{\pi} \norm{\grad h(\mu)} & =
		\left[ \frac{h(\mu)}{\norm{\grad \loss(\mu, \pi)}^2} +
		\frac{\loss(\mu, \bar{\pi})}{\norm{\grad \loss(\mu, \pi)}^2} \right] \cdot \norm{\grad h(\mu)}                                                      \notag          \\
		                                & \leq
		\left[ \frac{C \kappa^{-3/2} \loss(\mu, \bar{\pi})^{5/4} + \loss(\mu, \bar{\pi})}{\norm{\grad \loss(\mu, \pi)}^2} \right] \cdot \norm{\grad h(\mu)} \notag          \\
		                                & \leq
		(1 + \epsilon) \cdot \frac{\loss(\mu, \bar{\pi})}{\norm{\grad \loss(\mu, \pi)}^2} \cdot \frac{C}{\kappa^{3/2}} \norm{\grad \loss(\mu, \bar{\pi})}   \notag          \\
		                                & \leq
		\frac{(1 + \epsilon) C \kappa^{-3/2}}{(1 - C \kappa^{-3/2})^2} \cdot \zeta_{\bar{\pi}} \norm{\grad \loss(\mu, \bar{\pi})}                           \notag          \\
		                                & \leq \frac{C}{\kappa^{3/2}} \cdot \frac{\loss(\mu, \bar{\pi})}{\norm{\grad \loss(\mu, \bar{\pi})}}, \label{eq:polyak-step-diff-i}
	\end{align}
	where the first inequality follows from~\eqref{eq:polyak-fval-diff}, the second inequality follows
	from~\eqref{eq:polyak-grad-diff} and the fact that $\loss(\mu, \bar{\pi})$ dominates $\loss(\mu, \bar{\pi})^{5/4}$
	when $\mu$ is sufficiently close to $\mu^{\star}$, the penultimate inequality again follows from~\eqref{eq:polyak-grad-diff},
	and the last inequality follows from adjusting $\kappa$ and $C$ and relabeling.

	Now, we bound $\abs{\zeta_{\pi} - \zeta_{\bar{\pi}}}$. Indeed, we obtain the difference
	\begin{align}
		\abs{\zeta_{\pi} - \zeta_{\bar{\pi}}} & =
		\abs*{\frac{\loss(\mu, \pi)}{\norm{\grad \loss(\mu, \bar{\pi})}^2 (1 + O(C\kappa^{-3/2}))^2} -
		\frac{\loss(\mu, \bar{\pi})}{\norm{\grad \loss(\mu, \bar{\pi})}^2}} \notag \\
		                                      & =
		\abs*{
			\frac{h(\mu)}{\norm{\grad \loss(\mu,\bar{\pi})}^2 (1 + O(C \kappa^{-3/2}))^2} -
			\frac{\loss(\mu, \bar{\pi})}{\norm{\grad \loss(\mu, \bar{\pi})}^2} \left[
				1 - \frac{1}{(1 + O(C \kappa^{-3/2}))^2}
				\right]
		}                                                                   \notag \\
		                                      & \lesssim
		\frac{\loss(\mu, \bar{\pi})}{\norm{\grad \loss(\mu, \bar{\pi})}^2} \abs*{
			\frac{(1 + O(C \kappa^{-3/2}))^2 - 1}{(1 + O(C \kappa^{-3/2}))^2}
		}                                                                   \notag \\
		                                      & \leq
		\frac{C}{\kappa^{3/2}} \cdot \zeta_{\bar{\pi}},
		\label{eq:polyak-step-diff-ii}
	\end{align}
	again using~\eqref{eq:polyak-fval-diff} and~\eqref{eq:polyak-grad-diff}, as well
	as adjusting $\kappa$ and $C$ and relabeling if necessary.
	Finally, we plug~\cref{eq:polyak-step-diff-i,eq:polyak-step-diff-ii} into~\eqref{eq:polyak-step-difference} to obtain
	the first inequality in~\eqref{eq:relative-polyak-comparison-distance}:
	\begin{align*}
		\norm{\mu^{+} - \bar{\mu}^{+}} & \leq \frac{2C}{\kappa^{3/2}} \cdot \frac{\loss(\mu, \bar{\pi})}{\norm{\grad \loss(\mu, \bar{\pi})}}.
	\end{align*}

	We now prove the second inequality in~\eqref{eq:relative-polyak-comparison-distance}. From the tangent-normal expansion, we deduce
	\begin{align*}
		\loss(\mu, \bar{\pi}) & = \loss_{T}(y) + \loss_{N}(\mu) \lesssim \disp(y)^2 + \dist^2_{\mathcal{M}}(\mu) \lesssim \disp(y)^2 + \tau^2 \disp(y)^3,
	\end{align*}
	where the first inequality follows from Item~(\ref{item:reference-ravine-tangent-growth}) of~\cref{lem:relative-gradient-annulus}
	and the fact that $\loss_{N}$ is $C^{\infty}$ and minimized on $\mathcal{M}$, and the second inequality follows from $\mu \in \mathcal{T}_{\tau}$. After shrinking \(V\) if necessary, the term \(\tau^2\disp(y)^3\) is absorbed into
	\(\disp(y)^2\). On the other hand, Item~(\ref{item:reference-gradient-pl-like}) of
	\cref{lem:relative-gradient-annulus} gives
	\[
		\norm{\grad \loss(\mu, \bar{\pi})} \gtrsim \disp(y)^{3/2}.
	\]
	using Item~(\ref{item:reference-gradient-pl-like}) of~\cref{lem:relative-gradient-annulus}.
	Combining the two inequalities yields
	\[
		\frac{\loss(\mu, \bar{\pi})}{\norm{\grad \loss(\mu, \bar{\pi})}} \lesssim \disp(y)^{1/2} \lesssim \norm{P_{\mathcal{M}}(\mu) - \mu^{\star}},
	\]
	where the last inequality follows from Item~(\ref{item:reference-ravine-tangent-growth}) of~\cref{lem:relative-gradient-annulus}.
	This proves~\eqref{eq:relative-polyak-comparison-distance}, after potentially enlarging the constant $C > 0$.
\end{proof}

Corollary~\ref{proposition:polyak-step-contraction-reference-loss} is applicable to
iterates near the manifold $\mathcal{M}$; however, the Polyak step may push $\mu^{+}$ (respectively, $\bar{\mu}^{+}$) far from the manifold and outside
the annulus $\mathcal{A}_{\tau, \kappa}$. Our next result shows that, outside the tube $\mathcal{T}_{\tau}$, short gradient steps on $\loss(\mu, \pi)$
reduce the distance $\dist_{\mathcal{M}}(\mu)$. To do so, we leveraging the perturbative analysis in~\cref{proposition:loss-perturbation}.
First, we record a sufficient condition for re-entering the annulus $\mathcal{A}_{\tau, \kappa}$.

\begin{lemma}[Gradient descent approaches the ravine]
	\label{lem:short-gd-ravine-recovery}
	Fix $\mu^{(0)} := \mu$ and define
	\[
		\mu^{(j+1)}=\mu^{(j)}- \eta \grad \loss(\mu^{(j)},\pi).
	\]
	Then, for any fixed \(\rho>0\), there exists a neighborhood \(U\) of \(\mu^\star\),
	a stepsize threshold \(\eta_{\mathsf{ub}}>0\), and constants \(c,C>0\)
	such that the following holds: if $\eta < \eta_{\mathsf{ub}}$, \(\mu^{(j)}\in U \setD \mathcal C_{\kappa}^{\mathrm{split}}\) and $\mu^{(j)}\notin\mathcal{T}_{\tau}$ for all $j=0,...,k-1$,
	we have that:
	\[
		\dist_{\mathcal{M}}(\mu^{(j+1)}) \leq \left[
			1 - c \eta + \frac{C \eta}{\tau} \left(\rho + \frac{1}{\kappa^{3/2}}\right)
			\right] \cdot \dist_{\mathcal{M}}(\mu^{(j)}).
	\]
	Consequently, if the following inequalities hold:
	\begin{align}
		\label{eq:short-gd-annulus-entry-size}
		\frac{C}{\tau}\left(\rho+\kappa^{-3/2}\right)
		\leq
		\frac{c}{2}
		\quad \text{and} \quad
		(1-\tfrac c2\eta)^k
		\dist_{\mathcal M}(\mu^{(0)})
		\leq
		\tau\disp(y^{(k)})^{3/2},
	\end{align}
	then it follows that $\mu^{(k)} \in \mathcal{A}_{\tau, \kappa}$.
\end{lemma}
\begin{proof}
	We first analyze one step of gradient descent for the reference objective.
	Since $\mathcal{M}$ is a local \(C^\infty\) Morse ravine $\loss(\mu, \bar{\pi})$, there is a neighborhood
	$U$ of $\mu^{\star}$ on which
	\begin{equation}
		\mu = y + \xi, \;\;
		\text{where} \;\; y = P_{\mathcal{M}}\mu, \; \xi \perp \mathcal{T}_{y}\mathcal{M} \;\; \text{and} \;\;
		\norm{\xi} \leq \dist_{\mathcal{M}}(\mu).
		\label{eq:tangent-normal-expansion}
	\end{equation}
	Consequently, Corollary~\ref{cor:ravine-short-step-input}, with the coefficient \(\rho>0\), shows that there exists \(c>0\) such that
	\begin{align}
		\label{eq:reference-ravine-normal-splitting}
		\norm{\xi-\eta\grad_{\mu} \loss(\mu,\bar{\pi})} \leq
		(1 - c \eta) \dist_{\mathcal{M}}(\mu) + \rho \eta \cdot \dist(y, \mathcal{S}^{\star})^{3},
	\end{align}
	for all $\mu \in U$ and $\eta \in (0, \eta_{\mathsf{ub}})$; the exponent $p-1 = 3$ since
	$\loss(\mu, \bar{\pi})$ has $4^{\text{th}}$-order growth near $\mu_{\star}$.

	We now analyze a step of gradient descent on the mismatched objective. We obtain
	\begin{align}
		\dist_{\mathcal{M}}(\mu^{+})
		 & = \inf_{u \in \mathcal{M}} \norm{\mu^{+} - u}                                                                           \notag      \\
		 & \leq \norm{\mu^{+} - P_{\mathcal{M}} \mu}                                                                                \notag     \\
		 & = \norm{\mu - P_{\mathcal{M}}\mu - \eta \grad_{\mu} \loss(\mu, \pi)}                                                     \notag     \\
		 & \leq \norm{\xi - \eta \grad_{\mu} \loss(\mu, \bar{\pi})} + \eta\const^{(1)} \Delta_{\pi}                                     \notag \\
		 & \leq (1 - c \eta) \dist_{\mathcal{M}}(\mu) + \rho \eta \cdot \dist(y, \mathcal{S}^{\star})^{3} + \eta \const^{(1)} \Delta_{\pi},
		\label{eq:ravine-normal-splitting-master-bound}
	\end{align}
	where the penultimate inequality follows from~\eqref{eq:tangent-normal-expansion} and Item (ii) of~\cref{proposition:loss-perturbation}
	and the last inequality follows from~\eqref{eq:reference-ravine-normal-splitting}. It remains to show that the last two terms on the RHS are controlled by \(\disp(y)^{3/2}\).

	\medskip

	We focus on $\Delta_{\pi}$ first. Since $\mu \notin \mathcal{C}_{\kappa}$, we have
	\begin{equation}
		\Delta_{\pi}^{2/3} \leq \frac{1}{\kappa} \disp(y)
		\quad\Longrightarrow\quad
		\eta \const^{(1)} \Delta_{\pi}
		\leq
		\eta \const^{(1)}\kappa^{-3/2}\disp(y)^{3/2}.
		\label{eq:delta-pi-residual-dominated}
	\end{equation}
	At the same time, Item~(\ref{item:reference-ravine-tangent-growth}) of~\cref{lem:relative-gradient-annulus} shows that
	\[
		\dist(y,\mathcal S^\star)^3
		\leq
		C\disp(y)^{3/2}.
	\]
	Plugging~\eqref{eq:delta-pi-residual-dominated} and the preceding inequality into~\eqref{eq:ravine-normal-splitting-master-bound}, and enlarging \(C\) if necessary, we obtain
	\begin{equation}
		\dist_{\mathcal{M}}(\mu^+)
		\leq
		(1 - c \eta)\dist_{\mathcal{M}}(\mu)
		+
		C\eta(\rho+\kappa^{-3/2})\disp(y)^{3/2}.
		\label{eq:ravine-normal-splitting-master-bound-simplified}
	\end{equation}
	Now write $y^{(j)} = P_{\mathcal{M}}(\mu^{(j)})$ and suppose that
	$\mu^{(0)}, \dots, \mu^{(k-1)} \notin \mathcal{T}_{\tau}$. This means
	\[
		\disp(y^{(j)})^{3} \leq \frac{\dist_{\mathcal{M}}^2(\mu^{(j)})}{\tau^2}.
	\]
	Plugging this back into~\eqref{eq:ravine-normal-splitting-master-bound-simplified}, we obtain the one-step improvement
	\[
		\dist_{\mathcal{M}}(\mu^{(j+1)}) \leq \left[
			1 - c \eta + \frac{C \eta}{\tau} \left(\rho + \frac{1}{\kappa^{3/2}}\right)
			\right] \cdot \dist_{\mathcal{M}}(\mu^{(j)}).
	\]
	By the first condition in~\eqref{eq:short-gd-annulus-entry-size}, the last term
	inside the contraction factor is less than $\tfrac{c}{2}$. Therefore,
	\[
		\dist_{\mathcal M}(\mu^{(j+1)})
		\le
		(1-\tfrac c2\eta)
		\dist_{\mathcal M}(\mu^{(j)}).
	\]
	Applying the preceding estimate successively for \(j=0,\dots,k-1\), we obtain
	\begin{align*}
		\dist_{\mathcal M}(\mu^{(k)})
		 & \leq (1-\tfrac c2\eta)^k\dist_{\mathcal M}(\mu^{(0)}).
	\end{align*}
	By the second condition in~\eqref{eq:short-gd-annulus-entry-size},
	\[
		\dist_{\mathcal M}(\mu^{(k)})
		\le
		\tau\disp(y^{(k)})^{3/2}.
	\]
	We conclude that \(\mu^{(k)}\in\mathcal T_\tau\), whence $\mu^{(k)} \in \mathcal{A}_{\tau, \kappa}$
	as $\mu^{(k)} \in \mathcal{C}_{\kappa}$ by assumption.
\end{proof}

\begin{theorem}[Contraction or convergence up to tolerance]
	\label{thm:tangential-annulus-contraction-or-core}
	Define the Polyak step
	\[
		\mu^+:= \mu-\frac{\loss(\mu,\pi)-\loss(\theta^\star)}{\norm{\grad_{\mu} \loss(\mu,\pi)}^2} \grad_{\mu} \loss(\mu, \pi).
	\]
	For any sufficiently small \( \tau > 0 \), there is a neighborhood \(W \subset V \) of
	\(\mu^\star\), a constant \(C > 0\) and a threshold \(\kappa_{0}\ge 1\) such that \emph{exactly one} of the following holds for every
	\(\kappa\ge \kappa_{0}\) and every \(\mu\in W \cap \mathcal A_{\tau,\kappa}\) (provided
	\(\Delta_{\pi}\) is sufficiently small):
	\begin{itemize}
		\item either we have $\mu^{+} \in \mathcal{C}_{\kappa}$; or,
		\item the following contraction holds:
		      \begin{equation}
			      \label{eq:tangential-annulus-contraction-or-core}
			      \norm{P_{\mathcal{M}}(\mu^{+}) - \mu^{\star}} \leq \left(q + \frac{C}{\kappa}\right) \cdot \norm{P_{\mathcal{M}}(\mu) - \mu^{\star}}.
		      \end{equation}
	\end{itemize}
	In particular, after enlarging \(\kappa_{0}\) if necessary, we have $q + \tfrac{C}{\kappa} < 1$.
\end{theorem}

\begin{proof}
	Fix \(\kappa\ge \kappa_{0}\), \(\mu\in W\cap \mathcal A_{\tau,\kappa}\), and
	\(
	y:=P_{\mathcal{M}}(\mu),
	\)
	and suppose that
	\(
	\mu \notin \mathcal{C}_{\kappa}.
	\)

	We first verify the antecedent condition in~\eqref{eq:small-normal-grad-implies-contraction}. Since
	\(\mu\in \mathcal A_{\tau,\kappa}\subset \mathcal T_\tau\), we have
	\[
		\dist_{\mathcal M}(\mu)\le \tau \disp(y)^{3/2}.
	\]
	Moreover,~\eqref{eq:reference-gradient-pl-like-b} combined with~\eqref{eq:tangent-grad-approx-dispersion-pow} shows
	\[
		\norm{\grad \loss_N(\mu)}
		\lesssim
		\dist_{\mathcal M}(\mu)
		\leq
		C\tau \disp(y)^{3/2} \lesssim
		\tau \norm{\grad \loss_{T}(y)}
		\leq \frac{1}{100} \norm{\grad \loss_{T}(y)},
	\]
	after shrinking \(\tau\) if necessary. We have the following chain of inequalities:
	\begin{align*}
		\norm{P_{\mathcal{M}}(\mu^+) - \mu^{\star}} & \leq
		\norm{P_{\mathcal{M}}(\bar{\mu}^+) - \mu^{\star}} +
		\norm{P_{\mathcal{M}}(\mu^{+}) - P_{\mathcal{M}}(\bar{\mu}^+)} \\
		                                            & \leq
		q \cdot \norm{P_{\mathcal{M}}(\mu) - \mu^{\star}} +
		C \norm{\mu^{+} - \bar{\mu}^{+}}                               \\
		                                            & \leq
		\left(q + \frac{C}{\kappa}\right) \norm{P_{\mathcal{M}}(\mu) - \mu^{\star}},
	\end{align*}
	where the second inequality follows from~\cref{proposition:polyak-step-contraction-reference-loss}
	and local Lipschitz continuity of $P_{\mathcal{M}}$ and the last inequality follows from
	\cref{lemma:polyak-step-comparison} and relabeling.
	Enlarging \(\kappa_{0}\) if necessary gives the desired contraction.
\end{proof}
Theorem~\ref{thm:tangential-annulus-contraction-or-core}
suggests that Polyak steps either contract the distance to $\mu^{\star}$
along the ravine or enter the set $\mathcal{C}_{\kappa}$, whereupon the algorithm makes no further progress.
However, a Polyak step may land outside the annulus $\mathcal{A}_{\tau, \kappa}$; our next result
shows that gradient descent steps either restore proximity to the manifold without undoing
the progress achieved by the Polyak step, or land inside $\mathcal{C}_{\kappa}$ themselves.

\begin{corollary}[Distance contraction after \(\gdpolyak\) epochs]
	\label{cor:distance-contraction-after-gdpolyak-epochs}
	Fix $\tau>0$ sufficiently small. After possibly shrinking the neighborhood
	$W\subset V$ from~\cref{thm:tangential-annulus-contraction-or-core}, there exist
	constants $q_{\mathrm{amb}}\in(0,1)$, $\eta_{\mathrm{ub}}>0$, $\kappa_0\ge1$, and
	$\delta_\pi>0$ such that the following implication holds:
	Let $\kappa\ge\kappa_0$, $0<\eta<\eta_{\mathrm{ub}}$, and
	$\mu\in W\cap\mathcal A_{\tau,\kappa}$, and assume $\Delta_\pi\le\delta_\pi$.
	Let
	\[
		\mu^{(0)}
		:=
		\mu-
		\frac{\loss(\mu,\pi)-\loss(\theta^\star)}
		{\norm{\grad_\mu\loss(\mu,\pi)}^2}
		\grad_\mu\loss(\mu,\pi),
	\]
	followed by gradient steps $\mu^{(j+1)}
		=
		\mu^{(j)}-\eta\grad_\mu\loss(\mu^{(j)},\pi)$, for $j \geq 0$.
	Moreover, define
	\[
		K:=
		\min\set{
			j\ge0:
			\mu^{(j)}\in\mathcal T_\tau\cup\mathcal C_\kappa}.
	\]
	If $\mu^{(j)}\in W$ for all $j=0,\dots,K$, then one of the following holds:
	\begin{itemize}
		\item $\mu^{(K)}\in\mathcal C_\kappa$; or
		\item $\norm{\mu^{(K)} - \mu^{\star}}
			      \le
			      q_{\mathrm{amb}}
			      \norm{\mu - \mu^{\star}}.$
	\end{itemize}

\end{corollary}
\begin{proof}
	Let $y^{(j)}:=P_{\mathcal M}(\mu^{(j)})$, $y := P_{\mathcal{M}}(\mu)$, and $r:=\norm{y-\mu^\star}$; moreover, assume
	$\mu^{(0)} \notin \mathcal{C}_{\kappa}$ (since otherwise the first alternative holds trivially).
	By~\cref{thm:tangential-annulus-contraction-or-core}, we have
	\begin{equation}
		\label{eq:epoch-polyak-projected-contraction}
		\norm{P_{\mathcal M}(\mu^{(0)})-\mu^\star}
		\le
		\left(q+\frac{C}{\kappa}\right) \norm{y - \mu^{\star}}.
	\end{equation}
	After possibly enlarging $C$, we have the inequality
	\begin{equation}
		\label{eq:epoch-polyak-offshoot-bound}
		\dist_{\mathcal M}(\mu^{(0)})
		\le
		\dist_{\mathcal M}(\bar\mu^{(0)})
		+
		\norm{\mu^{(0)}-\bar\mu^{(0)}}
		\le
		\left(C + \frac{C}{\kappa}\right) \norm{y - \mu^{\star}}
		\leq C \norm{y - \mu^{\star}},
	\end{equation}
	where $\bar{\mu}^{(0)}$ denotes the corresponding reference Polyak step, the second inequality follows from~\eqref{eq:small-normal-grad-implies-contraction} applied to the reference
	Polyak step and~\eqref{eq:relative-polyak-comparison-distance}, and the last inequality uses $\kappa\ge 1$.

	If $\mu^{(K)}\in\mathcal C_\kappa$, then the first alternative holds; henceforth, let $\mu^{(K)}\notin\mathcal C_\kappa$.
	By definition of $K$, this implies that $\mu^{(K)} \in \mathcal{T}_{\tau} \setminus \mathcal{C}_{\kappa} = \mathcal{A}_{\tau, \kappa}$.
	It remains to show that gradient descent retains the progress made by the Polyak step towards the solution.

	\begin{claim}
		\label{claim:short-gd-projected-drift}
		For every \(\varepsilon>0\), after shrinking \(W\) and enlarging \(\kappa_0\) if necessary,
		\begin{equation}
			\label{eq:short-gd-projected-drift-claim}
			\norm{y^{(K)}-\mu^\star}
			\le
			\norm{y^{(0)}-\mu^\star}
			+
			\varepsilon \norm{y - \mu^{\star}}.
		\end{equation}
	\end{claim}
	\begin{proof}
		All occurrences of \(o_W(1)\) below are uniform over the iterates and tend to zero as
		the neighborhood \(W\) shrinks to \(\mu^\star\). Fix \(j<K\). Since \(P_{\mathcal M}\) is \(C^2\) locally, a Taylor expansion gives
		\begin{align*}
			y^{(j+1)} & =
			P_{\mathcal{M}}(\mu^{(j)} - \eta \grad_{\mu} \loss(\mu, \pi)) \\
			          & =
			P_{\mathcal{M}}(\mu^{(j)}) -
			\eta \grad P_{\mathcal{M}}(\mu) \grad_{\mu} \loss(\mu^{(j)}, \pi) +
			O(\eta^2 \norm{\grad_{\mu} \loss(\mu^{(j)}, \pi)}^2)
		\end{align*}
		We focus on the first-order term above. In particular, we have
		\begin{align*}
			 & \grad P_{\mathcal{M}}(\mu^{(j)}) \grad \loss(\mu^{(j)}, \pi) \\
			 & =
			\grad \loss_{T}(y^{(j)}) +
			\grad P_{\mathcal{M}}(\mu^{(j)}) \grad \loss(\mu^{(j)},\bar{\pi})
			- \grad \loss_{T}(y^{(j)})
			+ \grad P_{\mathcal{M}}(\mu^{(j)})(\grad \loss(\mu^{(j)}, \pi)
			- \grad \loss(\mu^{(j)}, \bar{\pi}))                            \\
			 & =
			\grad \loss_{T}(y^{(j)}) +
			\grad P_{\mathcal{M}}(\mu^{(j)}) \grad \loss(\mu^{(j)},\bar{\pi})
			- \grad \loss_{T}(y^{(j)})
			+ O(\Delta_{\pi})                                               \\
			 & =
			\grad \loss_{T}(y^{(j)}) +
			o_{W}(1) \cdot \dist_{\mathcal{M}}(\mu^{(j)}) + O(\Delta_{\pi}),
		\end{align*}
		where the penultimate equality follows from~\cref{proposition:loss-perturbation} and the last equality follows from
		Item (i) of~\cite[Theorem 4.2(i)]{davis2024gradient}.
		Since \(\mu^{(j)}\notin\mathcal C_\kappa\), $\Delta_\pi\le\kappa^{-3/2}\disp(y^{(j)})^{3/2}$, and since \(\mu^{(j)}\notin\mathcal T_\tau\), $\disp(y^{(j)})^{3/2}\le\tau^{-1}\dist_{\mathcal M}(\mu^{(j)})$.
		Therefore,
		\[
			\grad P_{\mathcal{M}}(\mu^{(j)}) \grad \loss(\mu^{(j)}, \pi)
			= \grad \loss_{T}(y^{(j)}) +
			\left( o_{W}(1) + \frac{C \kappa^{-3/2}}{\tau}\right) \dist_{\mathcal{M}}(\mu^{(j)}).
		\]
		At the same time, the second-order remainder term satisfies
		\begin{align*}
			\norm{\grad_{\mu} \loss(\mu^{(j)}, \pi)} & \leq
			\norm{\grad_{\mu} \loss(\mu^{(j)}, \pi) - \grad_{\mu} \loss(\mu^{(j)}, \bar{\pi})}
			+ \norm{\grad_{\mu} \loss(\mu^{(j)}, \bar{\pi}) - \grad_{\mu} \loss(\mu^{\star}, \bar{\pi})}              \\
			                                         & \lesssim
			\Delta_{\pi} + \norm{\grad_{\mu} \loss(\mu^{(j)}, \bar{\pi})}                                             \\
			                                         & \leq
			\left( \tfrac{\disp(y^{(j)})}{\kappa} \right)^{3/2} + \norm{\grad_{\mu} \loss(\mu^{(j)}, \bar{\pi})}      \\
			                                         & \leq
			\frac{\dist_{\mathcal{M}}(\mu^{(j)})}{\tau \kappa^{3/2}} + \norm{\grad_{\mu} \loss(\mu^{(j)}, \bar{\pi})} \\
			                                         & \lesssim
			\frac{\dist_{\mathcal{M}}(\mu^{(j)})}{\tau \kappa^{3/2}} +
			o_{W}(1) \norm{\grad \loss_{T}(\mu^{(j)})} + \dist_{\mathcal{M}}(\mu^{(j)})                               \\
			                                         & \lesssim
			\frac{\dist_{\mathcal{M}}(\mu^{(j)})}{\tau \kappa^{3/2}} +
			o_{W}(1) \norm{\grad \loss_{T}(y^{(j)})} + \dist_{\mathcal{M}}(\mu^{(j)})                                 \\
			                                         & \lesssim
			\frac{\dist_{\mathcal{M}}(\mu^{(j)})}{\tau \kappa^{3/2}} +
			o_{W}(1) \disp(y)^{3/2} + \dist_{\mathcal{M}}(\mu^{(j)})                                                  \\
			                                         & \lesssim \dist_{\mathcal{M}}(\mu^{(j)}),
		\end{align*}
		where the second inequality follows from~\cref{proposition:loss-perturbation}, the third inequality
		follows from $\mu^{(j)} \notin \mathcal{C}_{\kappa}$, the fourth inequality follows from $\mu^{(j)} \notin \mathcal{T}_{\tau}$, the
		fifth inequality follows from~\cite[Lemma 4.1 \& Theorem 4.2(3)]{davis2024gradient}, the sixth inequality
		follows from local smoothness of $\loss_{T}$, the penultimate inequality follows from~\cref{lem:relative-gradient-annulus},
		and the last inequality follows from the fact that $\mu^{(j)} \notin \mathcal{T}_{\tau}$. Putting everything together,
		\[
			\eta^2\norm{\grad_\mu\loss(\mu^{(j)},\pi)}^2
			\leq
			\eta^2 \dist_{\mathcal{M}}^2(\mu^{(j)})
			\leq
			\eta^2 \norm{\mu^{(j)} - \mu^{\star}} \dist_{\mathcal{M}}(\mu) \leq
			\eta^2 o_W(1)\dist_{\mathcal M}(\mu^{(j)}),
		\]
		Absorbing the remainder into the first order term, we have
		\[
			y^{(j+1)}
			=
			y^{(j)}
			-
			\eta\grad\loss_T(y^{(j)})
			+
			\eta \left(
			o_{W}(1) + \frac{C \kappa^{-3/2}}{\tau}
			\right) O(\dist_{\mathcal{M}}(\mu^{(j)})).
		\]
		Therefore, we obtain the progress bound
		\begin{align}
			\norm{y^{(j+1)}-\mu^\star}
			 & \le
			\norm{y^{(j)}-\eta\grad\loss_T(y^{(j)})-\mu^\star}
			+
			\eta\left(o_{W}(1) + \frac{C \kappa^{-3/2}}{\tau}\right) \dist_{\mathcal{M}}(\mu^{(j)}) \notag \\
			 & \le
			\norm{y^{(j)}-\mu^\star}
			+
			\eta
			\left(
			o_W(1)+\frac{C}{\tau}\kappa^{-3/2}
			\right)
			\dist_{\mathcal M}(\mu^{(j)}),
			\label{eq:short-gd-progress-bound}
		\end{align}
		where the second inequality follows from~\cite[Lemma 6.2]{davis2024gradient}.

		Next, since $\mu^{(j)}\notin\mathcal T_\tau$ and $\mu^{(j)}\notin\mathcal C_\kappa$, the proof of~\cref{lem:short-gd-ravine-recovery} gives
		\begin{align}
			\dist_{\mathcal M}(\mu^{(j+1)})
			 & \le
			(1-c\eta)\dist_{\mathcal M}(\mu^{(j)})
			+
			C\eta(\rho+\kappa^{-3/2})\disp(y^{(j)})^{3/2} \notag \\
			 & \le
			\left[
				1-c\eta+\frac{C\eta}{\tau}(\rho+\kappa^{-3/2})
				\right]
			\dist_{\mathcal M}(\mu^{(j)})                 \notag \\
			 & \le
			(1-\tfrac c2\eta)\dist_{\mathcal M}(\mu^{(j)}),
			\label{eq:manifold-distance-contraction}
		\end{align}
		where the second inequality follows from the inclusion \(\mu^{(j)}\notin\mathcal T_\tau\) and the last inequality follows from choosing \(\rho\) and \(\kappa_0\) sufficiently small and sufficiently
		large, respectively, so that they satisfy
		\[
			\frac{C}{\tau}(\rho+\kappa^{-3/2})\le \frac c2.
		\]
		Iterating the inequality in~\eqref{eq:manifold-distance-contraction} from $j = 0$ to an arbitrary $k < K$, we obtain
		\begin{equation}
			\dist_{\mathcal M}(\mu^{(k)})
			\le
			(1-\tfrac c2\eta)^k\dist_{\mathcal M}(\mu^{(0)}) \implies
			\eta\sum_{j=0}^{K-1}\dist_{\mathcal M}(\mu^{(j)})
			\le
			\frac{2}{c}\dist_{\mathcal M}(\mu^{(0)}),
			\label{eq:manifold-distance-geometric-series}
		\end{equation}
		where the last inequality follows from the geometric series sum formula.
		Telescoping~\eqref{eq:short-gd-progress-bound},
		\begin{align*}
			\norm{y^{(K)}-\mu^\star}
			 & \le
			\norm{y^{(0)}-\mu^\star}
			+
			\eta
			\left(
			o_W(1)+\frac{C}{\tau}\kappa^{-3/2}
			\right)
			\sum_{j=0}^{K-1}\dist_{\mathcal M}(\mu^{(j)}) \\
			 & \le
			\norm{y^{(0)}-\mu^\star}
			+
			\frac{2}{c}
			\left(
			o_W(1)+\tau^{-1}\kappa^{-3/2}
			\right)
			\dist_{\mathcal M}(\mu^{(0)})                 \\
			 & \le
			\norm{y^{(0)}-\mu^\star}
			+
			C
			\left(
			o_W(1)+\tau^{-1}\kappa^{-3/2}
			\right) \norm{y - \mu^{\star}},
		\end{align*}
		where the penultimate inequality follows from~\eqref{eq:manifold-distance-geometric-series} and the last inequality
		follows from~\eqref{eq:epoch-polyak-offshoot-bound} and relabeling $C$. Finally, we shrink \(W\) and enlarge \(\kappa_0\) so that
		\[
			C\left(o_W(1)+\tau^{-1}\kappa^{-3/2}\right)\le \varepsilon.
		\]
		This proves~\eqref{eq:short-gd-projected-drift-claim}.
	\end{proof}

	We now prove the contraction estimate. Combining~\eqref{eq:short-gd-projected-drift-claim} with
	\eqref{eq:epoch-polyak-projected-contraction}, we obtain
	\[
		\norm{P_{\mathcal M}(\mu^{(K)})-\mu^\star}
		\le
		\left(q+\frac{C}{\kappa}+\varepsilon\right)
		\norm{P_{\mathcal M}(\mu)-\mu^\star}.
	\]
	Choosing \(\varepsilon > 0\) sufficiently small and \(\kappa_0\) sufficiently large so that
	\[
		q+\frac{C}{\kappa}+\varepsilon
		=:
		q'
		<1
		\qquad
		\text{for all }\kappa\ge\kappa_0,
	\]
	we obtain a contraction towards $\mu^{\star}$ along the manifold for $\mu^{(K)}$:
	\begin{equation}
		\norm{P_{\mathcal M}(\mu^{(K)})-\mu^\star}
		\le
		q'
		\norm{P_{\mathcal M}(\mu)-\mu^\star}.
		\label{eq:almost-ambient-contraction}
	\end{equation}
	Finally, since \(\mu^{(K)}\in\mathcal T_\tau\) by definition of the hitting time $K$, we have
	\begin{equation}
		\dist_{\mathcal M}(\mu^{(K)})
		\le
		\tau \disp^{3/2}(y^{(K)}) \lesssim \tau \norm{y^{(K)} - \mu^{\star}}^{3},
		\label{eq:last-iterate-distance-from-M}
	\end{equation}
	where the second inequality follows from Item~\eqref{item:reference-ravine-tangent-growth} of
	\cref{lem:relative-gradient-annulus}.
	Consequently,
	\begin{align*}
		\norm{\mu^{(K)} - \mu^{\star}} & \leq
		\norm{y^{(K)} - \mu^{\star}} + \dist_{\mathcal{M}}(\mu^{(K)})            \\
		                               & \leq
		(1 + C \tau \norm{y^{(K)} - \mu^{\star}}^2) \norm{y^{(K)} - \mu^{\star}} \\
		                               & \leq
		(1 + \varepsilon') \cdot q' \cdot \norm{y - \mu^{\star}} \\
        &\leq
		(1 + \varepsilon') \cdot q' \cdot \big[ 1 + C_{W} \norm{\mu - \mu^{\star}} \big] \cdot \norm{\mu - \mu^{\star}},
	\end{align*}
	where the first inequality follows from the triangle inequality, the second
	inequality follows from~\eqref{eq:last-iterate-distance-from-M}, the third inequality follows from~\eqref{eq:almost-ambient-contraction},
    and the last inequality follows from~\cref{lemma:distance-conversion}.
    We label $q_{\mathrm{amb}} = (1 + \varepsilon') q' \cdot \big[1 + C_{W} \norm{\mu - \mu^{\star}}\big] < 1$ by suitably shrinking $W$ if necessary.
\end{proof}

The reader may notice that~\cref{cor:distance-contraction-after-gdpolyak-epochs} contains the
inclusion $\mu^{(j)} \in W$ as an explicit condition. Our final result shows that we can choose nested neighborhoods
$W_{\text{in}} \subset W_{\text{out}} \subset W$ such that every iterate remains in $W_{\text{out}}$.
\begin{lemma}[Gradient EM iterates remain close to $\mu^{\star}$]
    There exists a constant $\lambda_{\mathsf{ub}} \in (0, 1)$ such that, for all $\lambda < \lambda_{\mathsf{ub}}$, the following holds in the setting of~\cref{cor:distance-contraction-after-gdpolyak-epochs}:
    if $\mu \in (\lambda W) \cap \mathcal{A}_{\tau, \kappa}$, then all iterates $\mu^{(j)} \in W$ for $j \in \set{1, \dots, K}$. In
    particular, $\mu^{(K)} \in \lambda W$ itself.
\end{lemma}
\begin{proof}
Let $W_{\text{out}} \equiv W$ and $\tau_{W}$ denote the following stopping time:
\[
    \tau_{W} := \inf\set{k \in \mathbb{N}
    \mid \mu^{(k)} \notin W_{\text{out}}
    },
\]
and suppose that $\tau_{W} < K < \infty$, where $K$ is the stopping time defined in~\cref{cor:distance-contraction-after-gdpolyak-epochs}.
By definition, $\mu^{(0)}, \dots, \mu^{(\tau_W - 1)} \in W_{\mathrm{out}}$. For the iterate $\mu^{(\tau_W)}$ itself, we have
\begin{align*}
    \norm{\mu^{(\tau_W)} - \mu^{\star}} &\leq
    \norm{P_{\mathcal{M}}(\mu^{(\tau_W)}) - \mu^{\star}}
    + \dist_{\mathcal{M}}(\mu^{(\tau_W)}) \\
    &\leq
    \norm{y^{(\tau_W)} - \mu^{\star}} +
    \left(1 - \frac{c \eta}{2}\right)^{(\tau_W)}
    \dist_{\mathcal{M}}(\mu^{(0)}) \\
    &\leq
    \norm{y^{(0)} - \mu^{\star}} + \dist_{\mathcal{M}}(\mu^{(0)})
    + \varepsilon \norm{y - \mu^{\star}} \\
    &\leq
    (1 + \varepsilon) \norm{y - \mu^{\star}}
    + \dist_{\mathcal{M}}(\mu^{(0)}) \\
    &\leq
    (1 + \varepsilon + C) \norm{y - \mu^{\star}} \\
    &\leq
    (1 + \varepsilon + C) \left[1
    + C_{W} \norm{\mu - \mu^{\star}} \right]\norm{\mu - \mu^{\star}} 
\end{align*}
where the second inequality follows from~\cref{lem:short-gd-ravine-recovery},
the third inequality holds for any $\varepsilon > 0$ by suitably shrinking $W$ (following the proof of~\cref{claim:short-gd-projected-drift}),
the fourth inequality follows from~\cref{thm:tangential-annulus-contraction-or-core}, the penultimate inequality follows from~\eqref{eq:epoch-polyak-offshoot-bound}, and the last inequality
follows from~\cref{lemma:distance-conversion}. By suitably shrinking $W$ and defining $W_{\text{in}} = (1 + \varepsilon + C)^{-1} W_{\text{out}} \subset W$, we arrive at a contradiction; therefore, $\mu^{(j)} \in W_{\text{out}}$ for all
$j \leq K$.
\end{proof}

\end{document}